%% file: main.tex
\documentclass[twoside,11pt]{article}

\usepackage{blindtext}

\usepackage{jmlr2e}

\usepackage{amsmath, mathtools,subcaption,diagbox,thmtools}
\usepackage[english]{babel}
\usepackage[linesnumbered,ruled,vlined]{algorithm2e}
\SetKwInput{KwInput}{Input} 
\SetKwInput{KwOutput}{Output}   
\usepackage[dvipsnames]{xcolor}
\usepackage{cleveref}

\crefname{proposition}{proposition}{propositions}
\crefname{assumption}{assumption}{assumptions}
\crefname{lemma}{lemma}{lemmas}

\newcommand{\M}{\mathcal{M}}
\newcommand{\R}{\mathbb{R}}

\newcommand{\bs}[1]{\boldsymbol{#1}}

\newcommand*{\horzbar}{\rule[.5ex]{2.5ex}{0.5pt}}

\DeclareMathOperator*{\argmax}{arg\,max}
\DeclareMathOperator*{\argmin}{arg\,min}

\usepackage{lastpage}
\jmlrheading{23}{2025}{1-\pageref{LastPage}}{6/25}{}{21-0000}{Haoyu Chen, Anna Little, and Akil Narayan}

\ShortHeadings{Largest Angle Path Distance For MMC}{Chen, Little, and Narayan}
\firstpageno{1}

\begin{document}

\title{Robust Multi-Manifold Clustering via Simplex Paths}

\author{\name Haoyu Chen \email Haoyu.Chen@utah.edu \\
       \addr Department of Mathematics\\
       University of Utah\\
       Salt Lake City, UT 84112, USA
       \AND
       \name Anna Little \email little@math.utah.edu \\
       \addr Department of Mathematics\\
       University of Utah\\
       Salt Lake City, UT 84112, USA
       \AND
       \name Akil Narayan \email akil@sci.utah.edu \\
       \addr Department of Mathematics\\ Scientific Computing and Imaging (SCI) Institute \\ University of Utah \\ 
       Salt Lake City, UT 84112, USA
}

\editor{My editor}

\maketitle

\begin{abstract}%
This article introduces a novel, geometric approach for multi-manifold clustering (MMC), i.e. for clustering a collection of potentially intersecting, $d$-dimensional manifolds into the individual manifold components. We first compute a locality graph on $d$-simplices, using the dihedral angle in between adjacent simplices as the graph weights, and then compute infinity path distances in this simplex graph. This procedure gives a metric on simplices which we refer to as the \textit{largest angle path distance} (LAPD). We analyze the properties of LAPD under random sampling, and prove that with an appropriate denoising procedure, this metric separates the manifold components with high probability. We validate the proposed methodology with extensive numerical experiments on both synthetic and real-world data sets. These experiments demonstrate that the method is robust to noise, curvature, and small intersection angle, and generally out-performs other MMC algorithms. In addition, we provide a highly scalable implementation of the proposed algorithm, which leverages approximation schemes for infinity path distance to achieve quasi-linear computational complexity. 
\end{abstract}

\begin{keywords}
  Path-based metric, manifold learning, graph theory, clustering
\end{keywords}

\section{Introduction}

Analyzing high-dimensional data poses significant challenges, especially when the number of dimensions exceeds the sample size. In such settings, traditional statistical methods often fail, and standard algorithms become computationally impractical. However in many real-world applications, data tends to concentrate near low-dimensional structures, which can be exploited to mitigate the curse of dimensionality. In the simplest case, data may cluster along a single low-dimensional subspace or manifold. A more flexible and realistic model assumes that data lies on multiple low-dimensional structures. By identifying and treating each structure individually, one can obtain a sparse and efficient representation of the data, and the first step in this process is to assign data points to their respective structures. This leads to the central focus of this article: the multi-manifold clustering (MMC) problem, in which data points are drawn from a collection of manifolds, and the objective is to recover the manifold labels up to permutation.

A closely related problem to MMC is (linear) subspace clustering, which simplifies the general model by assuming that the underlying manifolds are linear subspaces. This problem has been widely studied, with notable methods including \cite{elhamifar_sparse_2013, liu_robust_2012, vidal_principal_2016, bradley_k-plane_2000, vidal_subspace_2011, ho_clustering_2003, tipping_mixtures_1999, costeira_multibody_1998, wright_robust_2009, zhang_survey_2015}, among others. Many of the most successful approaches rely on the self-expressive property, in which each data point is represented as a linear combination of other data points. Ideally, points from the same subspace have large coefficients in this representation, while points from different subspaces have small coefficients. The resulting coefficient matrix is then transformed into an affinity matrix, and standard clustering algorithms such as $K$-means or spectral clustering are then applied to obtain the final segmentation. Subspace clustering methods differ in their computational efficiency and in how they enforce self-representation. Broadly, they fall into three main categories:
(i) sparse or low-rank representation methods, such as SSC \cite{elhamifar_sparse_2013} and LRR \cite{liu_robust_2012}, which use convex optimization to enforce self-representation;
(ii) function or polynomial fitting methods, which cast the problem as learning a parametric model in which algebraic constraints can be used to segment the data \citep{Li2022, vidal_principal_2016};
(iii) deep learning-based methods, which employ autoencoders, self-expressive layers, and adversarial training to learn feature representations that reveal subspace structures \citep{Ji2017DeepSC, pengDSC2020, pan2018}.

However, the assumption that data lies on a collection of linear subspaces is often too restrictive for real-world datasets, which frequently exhibit nonlinear structures that are better modeled as multiple manifolds. Such a multi-manifold modeling assumption has been proven useful 
for motion segmentation and image classification \citep{souvenir_manifold_2005,ye_multimanifold_2019,hastie_metrics_1997, basri_lambertian_2003, tomasi_shape_1992}, for dynamic pattern identification and optical flow \citep{gong_robust_2012,wang_multimanifold_2015}, and for classification of text on webpages and in news stories \citep{zong_multiview_2017, babaeian_multiple_2015}. 
Popular manifold clustering methods \citep{arias-castro_spectral_2017, souvenir_manifold_2005, elhamifar_sparse_2013, babaeian_multiple_2015, wang_multimanifold_2015, Tenenbaum2000AGG, Belkin2003} thus extend the existing algorithms for subspace clustering to handle nonlinear data structures. The recent survey, \cite{abdolali_beyond_2021}, taxonomizes manifold clustering methods into three categories: (i) locality preserving, (ii) kernel based, and (iii) neural network based. Locality preserving methods \citep{arias-castro_spectral_2017, elhamifar_sparse_2013, babaeian_multiple_2015, Tenenbaum2000AGG, Belkin2003,trillos_large_2021} focus on maintaining the local geometric structure of the manifolds when mapping to a lower-dimensional space, ensuring that the new representation reflects the true underlying structure. Kernel based methods leverage kernel functions to map data into a higher-dimensional space where linear subspace methods can be applied \citep{Patel2014}.  
Neural network-based approaches to manifold clustering \citep{wu2022generalized, wu2022deep, McConville2021, Li2022NeuralMC} leverage deep learning models to learn meaningful data embeddings that facilitate clustering. A central technique in this category is autoencoder-based clustering, where an autoencoder compresses the data into a lower-dimensional latent space, and clustering is then performed on this learned representation.
Our approach is a locality preserving method which leverages local angle information to define an appropriate distance/affinity.

Naturally, the major challenge in the MMC problem is that the underlying manifolds may intersect (otherwise density-based clustering approaches would be sufficient). Previous locality preserving methods have utilized local tangent plane approximation as well as curvature and angle constraints to mitigate this challenge \citep{arias-castro_spectral_2017, babaeian_multiple_2015, trillos_large_2021, wang_multimanifold_2015}. Our approach avoids explicitly estimating such geometric structures by letting the algorithm itself explore the local information, requiring only knowledge of the intrinsic data dimension $d$. In particular, we compute a graph of local $d$-dimensional simplices, weighted according to the dihedral angle in between simplices, and then compute an infinity path distance inside the resulting weighted graph. We note that the computation of infinity path distances is well studied in the literature, where it is encountered under many names, including the maximum capacity path problem, the bottleneck edge query problem, the widest path problem, and the \textit{longest leg path distance} \citep{little2020path, hu1961maximum, camerini1978min, gabow1988algorithms}. Note \cite{little2020path} analyzed paths of edges weighted by Euclidean distance, instead of simplex paths weighted by angle, but we use the same approximation scheme as \cite{little2020path} to estimate infinity path distances in a scalable manner. Empirical results indicate that clustering based on this angle based metric can recover the manifold components with high accuracy. In addition, our method also has the following advantages over the state-of-the-art approaches (i) it can reliably learn the number of manifolds, which is often a required input by other MMC algorithms; (ii) it can be implemented in computational complexity that is quasi-linear with respect to the sample size, while most other methods are at least quadratic; and (iii) it is more robust to noise and curvature than competing MMC methods. An overview of the proposed method (without analysis) can be found in the conference paper \cite{haoyu_lapd_2023}. This article analyzes the proposed metric under random sampling, derives theoretical guarantees for solving MMC, and includes much more extensive numerical experiments (with analysis of real data, not done in \cite{haoyu_lapd_2023}).

The remainder of the paper is organized as follows. Section \ref{sec:problem setup} introduces the problem set-up and assumptions and provides an overview of our proposed approach. Section \ref{sec:methodology} gives a more formal presentation of the methodology including the construction of the LAPD metric from a simplex graph and a novel denoising procedure which enhances manifold separability. 
Section \ref{sec:main results} gives theoretical guarantees that the method can solve MMC with high probability, quantifying the impact of noise and intersection angle (quantifying impact of curvature is left to future work). Section \ref{sec:Experiments} discusses the computational complexity of the algorithm and presents extensive experiments on both synthetic and real data investigating the accuracy, robustness, and scalability of the proposed method. Finally, Section \ref{sec:Conclusion} summarizes the article and suggests directions for future research; additional details and technical proofs are provided in the appendices. 

\section{Problem Setup and Method Overview}
\label{sec:problem setup}
 
Let $\M \coloneqq \{\M_j\}_{j \in [m]}$ denote a collection of $m$ compact, Riemannian manifolds embedded in $\R^D$, where each $\M_j$ has intrinsic dimension $d \ll D$. We denote by $T_\tau(\M_j) \coloneqq \{ x \in \mathbb{R}^{D}: \exists z \in \M_j \text{ such that } \|z-x\| \leq \tau \}$ a tube of radius $\tau$ around manifold $\M_j$, and $T_{\tau}(\M) \coloneqq \cup_{j \in [m]} T_\tau (\M_j)$ the union of these tubes. We think of $T_{\tau}(\M)$ as a collection of noisy manifolds (with noise level $\tau$), and we assume access to a dataset $X = \{x_i\}_{i \in [n]}$ consisting of i.i.d. samples from $T_{\tau}(\M)$. Let $z_i = \argmin_{y\in\M} \|y-x_i \|$ denote the closest point on $\M$ to data point $x_i$, i.e. $z_i$ is the noiseless counterpart of $x_i$. We let $\{\ell_i\}_{i\in [n]}$ denote the underlying manifold labels, i.e. $\ell_i=j$ if the clean sample $z_i \in \M_j$, and our goal is to recover the labels $\ell_i$ up to permutation from the noisy data $X$. We are primarily interested in the case when manifolds intersect, and we let $\Theta_{ij} \in (0, \frac{\pi}{2}]$ denote the intersection angle between two intersecting manifolds $\M_i$ and $\M_j$. When the context is clear, we shall omit the manifold index subscript and simply denote the intersection angle between two manifolds as $\Theta$. To simplify the analysis, we make the following assumptions.

\begin{assumption}[Intersection dimension]
\label{assump:intersection_dim}
    If $\M_i \cap \M_j \ne \emptyset$, then $\M_i \cap \M_j$ has dimension $d-1$.
\end{assumption}

\begin{assumption}[Uniform sampling]
\label{assump:uniform_samp}
    The samples $\{x_i\}_{i \in [n]}$ are iid draws from the uniform probability measure $P$ on $T_\tau(\M)$, i.e. 
    $dP = \frac{dV}{V}$ for $dV$ the natural volume measure on $T_\tau(\M)$ and $V=\text{vol}\left(T_\tau(\M)\right)$. 
\end{assumption}

We remark that Assumptions \ref{assump:intersection_dim} and \ref{assump:uniform_samp} are for convenience, but not fundamentally required by our method. For example, all our theoretical results still hold when sampling from a nonuniform probability measure on $T_\tau(\M)$; as long as the corresponding density $\rho$ satisfies $0<\rho_{\min} \leq \rho \leq \rho_{\max}$, only constants will be affected. Furthermore, Assumption \ref{assump:intersection_dim} can be thought of as a ``worst case" scenario, as we expect success in this regime will guarantee success for intersections with dimension smaller than $d-1$.

Next we summarize our approach to solving this problem, postponing many technical details to Section \ref{sec:methodology}.
We start by identifying a locality graph $\mathcal{G}_X = (X, E_X)$ on the nodes $X$, where edges are defined through an annular nearest neighbor computation. The inner radius of the annulus is chosen to lie above the noise level, and the outer radius of the annulus is chosen to maintain a feasible computational complexity. From this node-level graph, we construct a set of $d$-simplices $S$, where a simplex $\Delta \in S$ if, among the $d+1$ vertices of $\Delta$, there is one vertex that is $\mathcal{G}_X$-connected to the remaining $d$ vertices. We remove from $S$ simplices of bad ``quality'', i.e., those that are very elongated. 

From the set of $d$-simplices, we say that $\Delta_i, \Delta_j \in S$ are \textit{adjacent} if they share a common face, i.e., if the number of shared vertices is $d$, $|\Delta_i \cap \Delta_j| = d$. When two such simplices are adjacent, we can compute a dihedral angle between them, $\theta_{ij} \coloneqq \theta(\Delta_i, \Delta_j) \in (0,\pi]$. The set of all such intersections and angles from adjacent pairs in $S$ identifies a \textit{simplex-level} graph $\mathcal{G}_S = (S, E_S, W_S)$, where $E_S$ contains pairs of simplices that are adjacent, and the corresponding weights $W_S$ are defined by their intersection angle $\theta$, i.e. $W_S(\Delta_i, \Delta_j) = \pi-\theta_{ij}$. Intuitively, $W_S$ encodes the angle dissimilarity between adjacent simplices, since when the simplices form a flat connection, $\theta_{ij}\approx \pi$ and $W_S(\Delta_i, \Delta_j)\approx 0$. From the simplex graph $\mathcal{G}_S = (S, E_S, W_S)$ created with angle-based weights, we derive an angle-based metric between all simplices in $S$ by computing the infinity shortest path distance in $\mathcal{G}_S$; we refer to this metric as the \textit{largest angle path distance (LAPD)}. After incorporating an appropriate de-noising procedure, we perform distance-based clustering with LAPD to recover a partition of $S$ into $m$ clusters (generally the parameter $m$ can also be learned during this step). Finally, the partition of the simplex set $S$ can be extended to a partition of the node set $X$ by a majority vote procedure (i.e. to cluster data point $x_i$, we assign the most common label in the set of simplices containing $x_i$). 

\begin{figure}[tb]
\centering 
\begin{subfigure}{0.4\textwidth}
\includegraphics[width=\linewidth]{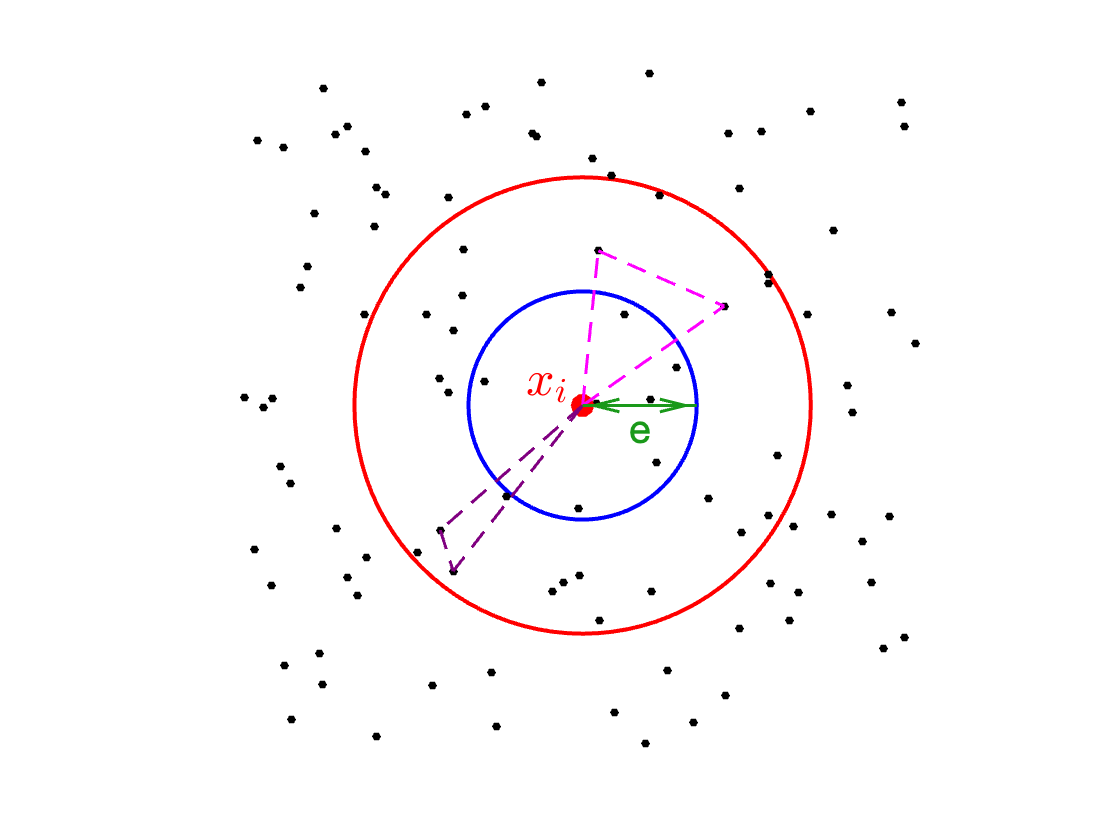}
\caption{Simplex construction for $d=2$.} 
\label{subfig:simplex construction}
\end{subfigure}
\begin{subfigure}{0.4\textwidth}
\includegraphics[width=\linewidth]{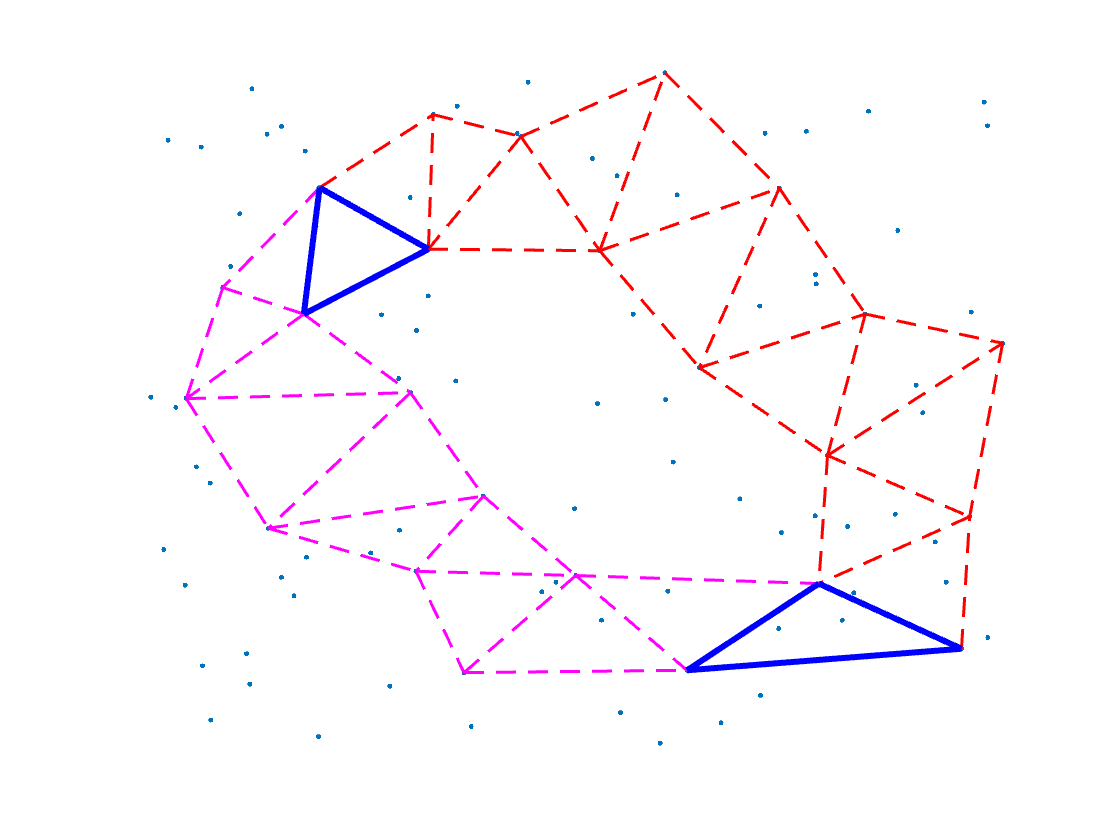}
\caption{Two sample paths for $d=2$.} 
\label{subfig:sample paths}
\end{subfigure} \\ 
\begin{subfigure}{0.4\textwidth}
\includegraphics[width=\linewidth]{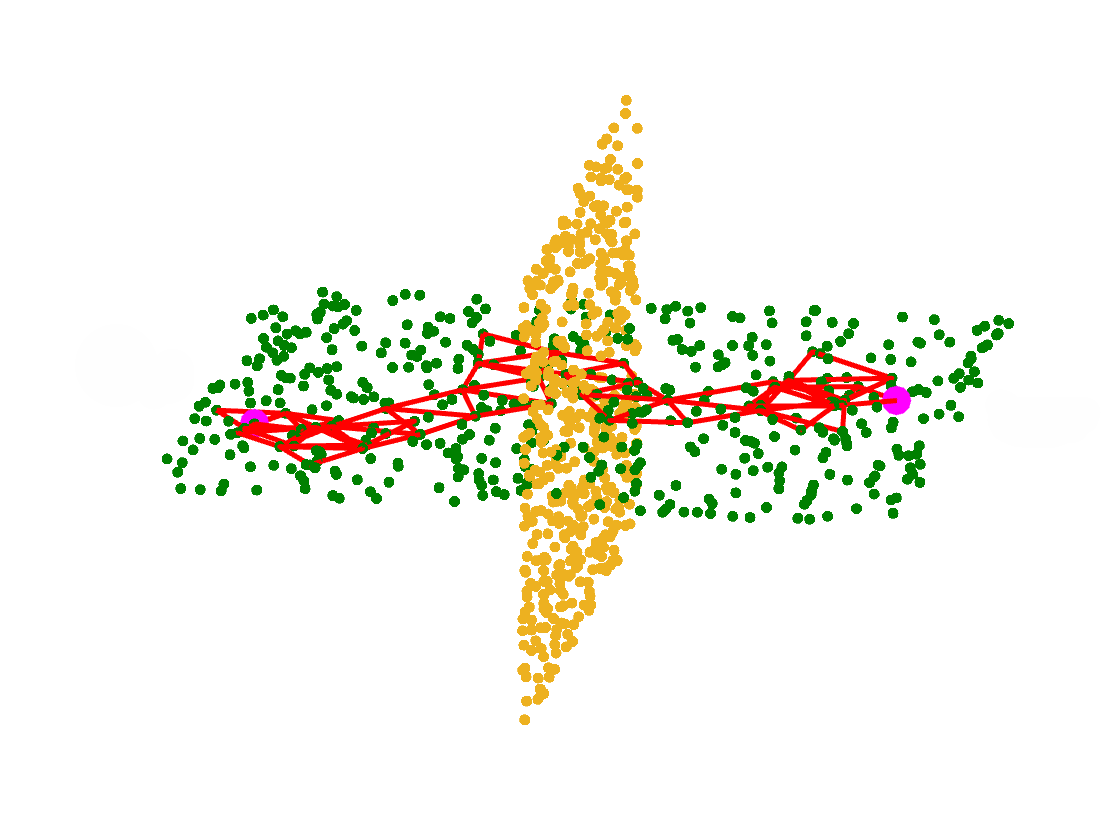}
\caption{Path within the same manifold.} 
\label{subfig:within path}
\end{subfigure}
\begin{subfigure}{0.4\textwidth}
\includegraphics[width=\linewidth]{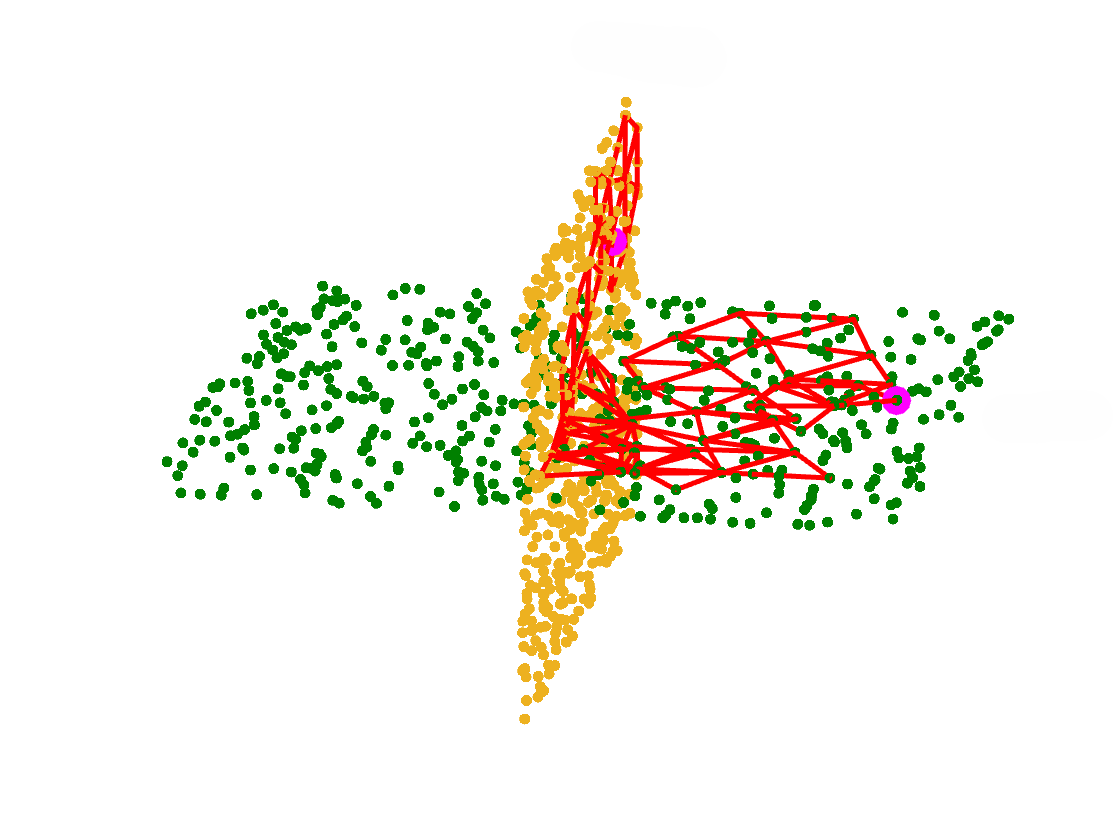}
\caption{Path connecting different manifolds.} 
\label{subfig:between path}
\end{subfigure}
\caption{Construction of simplices and simplex paths connected by adjacent simplices over a pair of linear, 2-dimensional intersecting manifolds.}
\label{fig:method_illustration}
\end{figure} 

Figure \ref{fig:method_illustration} illustrates the key ideas of our approach. Figure \ref{subfig:simplex construction} illustrates the construction of the simplex set $S$: for each data point $x_i$, we define local simplices by the various subsets formed by combining $x_i$ with $d$ of its annular neighbors (when $d=2$, one obtains triangles as shown). We then consider paths of connected simplices as shown in Figure \ref{subfig:sample paths}: two simplices are generally connected by a large number of different paths, but the best path according to the LAPD metric is the flattest path, i.e. the path which avoids making any sharp turns, since according to LAPD, the length of a path is the single worst angle encountered along the path. Figures \ref{subfig:within path} and \ref{subfig:between path} illustrate why the LAPD metric is natural for solving MMC: Figure \ref{subfig:within path} illustrates that any two simplices from the same manifold can always be connected with a flat path, while Figure \ref{subfig:between path} shows that every path of simplices between manifolds must always make a sharp turn somewhere along the path. Our theoretical results establish that under reasonable assumptions, the LAPD between simplices on the same manifold is always small, while the LAPD between simplices on different manifolds is always large, so that the metric can successfully differentiate the manifold components. Informally, the assumptions needed to guarantee success are (1) the noise level $\tau$ is not too large, (2) the curvature of the manifolds is not too large, and (3) the intersection angles of the manifolds are not too small; these conditions guarantee that simplices can be constructed at the right scale for solving MMC, i.e. one needs simplices to be large relative to the noise but small relative to the manifold curvature.

\section{Methodology}
\label{sec:methodology}

We now describe the details of our proposed algorithm. Section \ref{ssec:method-simplices} describes the construction of the original locality graph and the derived set of local simplices $S$. Section \ref{ssec:method-lapd} describes the creation of the weighted graph $\mathcal{G}_S$ on $S$ with angle-based weights and the computation of the LAPD shortest path distances in $\mathcal{G}_S$. Section \ref{ssec:method-clustering} discusses a procedure for de-noising the LAPD metric and distance-based clustering with LAPD.

\subsection{Construction of Simplices}\label{ssec:method-simplices}

We first fix a minimal edge length $e$, which will encode the scale at which the metric is constructed. For each data point $x_i$, we remove any neighbors within distance $e$ and then connect $x_i$ to its $B$ nearest neighbors among the remaining points. We let $\mathcal{G}_X(e,B)=(X,E_X)$ denote the corresponding unweighted annular locality graph, which for brevity we denote simply as $\mathcal{G}_X$. We then form an initial set of candidate simplices $\mathcal{S}$ by taking subsets of $d+1$ points where at least one point is $\mathcal{G}_X$ connected to all remaining points, i.e.:
\begin{align}
\label{equ:cand_simplices}
    \mathcal{S} &= \{\Delta : \Delta \subseteq X, |\Delta|=d+1, \exists x\in \Delta \text{ such that } E_X(x,y)>0\, \forall\, y\in \Delta\setminus x \} \, .
\end{align}
As our method is more robust when we use regularly shaped simplices, we define the following measures of simplex distortion:
\begin{align*}
    \text{Distortion}_{1}(\Delta=\{x_1,\ldots,x_{d+1}\}) &= \frac{ \min_{i\ne j} \| x_i-x_j\|}{\max_{i\ne j}\|x_i-x_j\|} \, , \\
    \text{Distortion}_{2}(\Delta=\{x_1,\ldots,x_{d+1}\}) &= \frac{\text{vol}_d(\Delta)}{V_0 \min_{i\ne j} \| x_i-x_j\|^d} \, ,
\end{align*}
where $\text{vol}_d(\Delta)$ is the $d$-volume of simplex $\Delta$ and $V_0 \coloneqq \frac{d+1}{d! \sqrt{2^d}}$ is the $d$-volume of a regular simplex with unity edge length. Note Distortion$_1$ measures irregularity of edge lengths while Distortion$_2$ measures irregularity of volume, i.e. the deviation of the $d$-volume of any $\Delta$ from that of a regular simplex.
We then let $0<q<1$, $0<r_0<1$ be parameters controlling Distortion$_1$,  Distortion$_2$, and define the set of \textit{valid simplices} $S$ by
\begin{align*}
    S &= \{ \Delta \in \mathcal{S} : \text{Distortion}_1(\Delta) \geq q , \text{Distortion}_2(\Delta) \geq r_0 \} \, .
\end{align*}
By construction, the set of valid simplices $S$ satisfies the following assumption:
\begin{assumption}{(Valid simplices: size and distortion)} 
\label{assump:size-quality}
  Let $S$ be a set of simplices constructed from samples of $T_{\tau}(\M)$. We assume $S$ satisfies size/distortion constraint encoded by $(e,q,r_0)$, i.e. 
  there exist constants $e > 0$, $q \in (0,1)$, and $r_0 \in (0,1)$ such that every simplex $\Delta=\{x_1,\ldots,x_{d+1}\} \in S$ satisfies
\begin{align}
    e &\leq \| x_i - x_j \| \leq \frac{e}{q} \;\; \hspace{5pt} \text{for all } i,j\in[d+1],\, i<j \, , \label{eq:simplex quality constraint}  \\
    r_0 V_0 e^d &\leq \mathrm{vol}_d(\Delta) \leq V_0 \left(\frac{e}{q}\right)^d \, . \label{eq:vol-constraint}
\end{align}
\end{assumption}
Thus four key parameters determine the set of valid simplices: $e$, $q$, $r_0$, and $B$; $e$ constrains the size, while $(q,r_0)$ constrains the shape of allowable simplices. Reasonable and fixed values of $q$ in \eqref{eq:simplex quality constraint}, say $q \approx 1/2$, are not sufficient to guarantee \eqref{eq:vol-constraint} uniformly in $d$, so we introduce $r_0$ explicitly to enforce this constraint. Since the cardinality of $S$ is $O(nB^d)$, $B$ constrains the scalability of the proposed algorithm. Figure \ref{subfig:simplex construction} illustrates the construction of $S$: $e$ is the radius of the inner circle, $B$ is the number of points in the annulus; the elongated purple simplex is a candidate simplex, but clearly does not survive filtering if for example $q > \frac{1}{2}$. Theorems in Section \ref{sec:main results} provide some heuristics for tuning these parameters. Furthermore, our theoretical results will require that the noise level is bounded by the allowable edge length:  
{\begin{assumption}{(Small noise)} 
\label{assump:small noise}
The noise level $\tau$ is small compared to the size of valid simplices, i.e. $\tau \leq \frac{e}{q}$. 
\end{assumption}

\subsection{Weighted Graph $\mathcal{G}_{S}$ and LAPD}\label{ssec:method-lapd}

Having established a set of valid simplices $S$, we can proceed to define the weighted graph $\mathcal{G}_S=(S,E_S,W_S)$ by connecting adjacent simplices in $S$ with angle-based weights. Two simplices $\Delta_i$ and $\Delta_j$ are adjacent if they share a common face, i.e. $|\Delta_i \cap \Delta_j|=d$. For every pair of adjacent simplices $\Delta_i, \Delta_j$, we can compute the dihedral angle $\theta_{ij} \coloneqq \theta(\Delta_i, \Delta_j) \in (0,\pi]$ between them as follows. We let $x_i, x_j$ denote the apex points in $\Delta_i, \Delta_j$ (i.e. the points which are not shared), and let $v_i$ be the unique vector connecting $\Delta_i \cap \Delta_j$ and $x_i$ which is normal to $\Delta_i \cap \Delta_j$, and similarly for $v_j$. Then  $\theta_{ij}$ is simply the angle in between $v_i, v_j$. More specifically, we fix any basepoint $c\in \Delta_i \cap \Delta_j$ (for example, the centroid) and define:  
\begin{align*}
    v_i =& (x_i - c) - P_{\Delta_i \cap \Delta_j}(x_i-c) \\
    v_j =& (x_j - c) - P_{\Delta_i \cap \Delta_j}(x_j-c) \\
    \theta_{ij} &= \arccos\left( \frac{\langle v_i, v_j \rangle}{\|v_i\| \cdot \|v_j\|} \right) \, ,
\end{align*}
where $P_{\Delta_i \cap \Delta_j}$ denotes projection onto the span of $\{x_s - c : x_s \in \Delta_i \cap \Delta_j \}$, which is $d-1$ dimensional by Assumption \ref{assump:intersection_dim}. That is, $P_{\Delta_i \cap \Delta_j} = Q Q^\intercal$ with $Q$ being the orthonormal basis of the shared face $\Delta_i \cap \Delta_j$.

When $\theta_{ij}\geq \frac{\pi}{2}$, we connect $\Delta_i$ and $\Delta_j$ in the graph $\mathcal{G}_S$ and define the edge weight between them by 
\begin{equation}
\label{eq:simplex graph weight}
    W_S(\Delta_i,\Delta_j) = \pi -\theta_{ij} \, .
\end{equation}
Note $W_S$ encodes an angle-based dissimilarity measure on $S$.\footnote{For numerical implementation, in our code we define $W_S(\Delta_i,\Delta_j) = \max\{\pi -\theta_{ij},\delta\}$ for a small numerical tolerance parameter $\delta=10^{-8}$; this ensures $W_S(\Delta_i,\Delta_j)>0$ if and only if $\Delta_i, \Delta_j$ are connected in $\mathcal{G}_S$.}

When $\Delta_i, \Delta_j$ form a flat connection,  $\theta_{ij}\approx \pi$, and $W_S$ is small; if however the simplices connect at a sharp angle, $W_S$ will be large. We note that some of our experiments utilize a two-sided version of the above weight construction defined in \eqref{eq:two way weight}. Our core Definition \ref{def:LAPD} of LAPD defines a metric on $S$ by computing the infinity-shortest path distance in $\mathcal{G}_S$, i.e. it uses the local angle dissimilarities encoded in $W_S$ to define a global metric on all of $S$.   

\begin{definition}[LAPD]
\label{def:LAPD}
Let $\mathcal{P}=\mathcal{P}(\Delta_i, \Delta_j)$ be the set of all simplex paths connecting $\Delta_i, \Delta_j$ in $\mathcal{G}_S$, i.e. any $\{\Gamma_1, \ldots, \Gamma_L\} \in \mathcal{P}$ satisfies $\Gamma_1 = \Delta_i, \Gamma_L = \Delta_j$, and $E_S(\Gamma_t,\Gamma_{t+1})$ exists for all $1\leq t\leq L-1$. Then the largest angle path-distance (LAPD) between $\Delta_i, \Delta_j$ is:
{\small
\begin{align}
\label{equ:LAPD}
    \text{LAPD}(\Delta_i, \Delta_j) &:= \min_{\{\Gamma_1, \ldots, \Gamma_L\} \in \mathcal{P}} \, \max_{1\leq t\leq L-1} W_S(\Gamma_t,\Gamma_{t+1}) \, .
\end{align}
    }
\end{definition}

We remark that LAPD looks for the single \textit{shortest} path connecting $\Delta_i, \Delta_j$ in $\mathcal{G}_S$ (thus the minimum over all paths $\mathcal{P}$), where the length of the path is determined by the single largest angle encountered along the path (thus the maximum over the path segments). Thus LAPD is small if there exists a flat path connecting two simplices (this should be the case if $\Delta_i, \Delta_j$ originate on the same manifold; see Figure \ref{subfig:within path}), while the LAPD is large if every path connecting $\Delta_i, \Delta_j$ takes a sharp turn along the way (this should be the case if $\Delta_i, \Delta_j$ originate on different manifolds; see Figure \ref{subfig:between path}).

To quantify this observation, we define the maximal within manifold LAPD in Definition \ref{def:Within LAPD} and the minimal between manifolds LAPD in Definition \ref{def:Between LAPD}. We let $S(\M_j)$ denote the elements $\Delta=\{x_1, \ldots, x_{d+1}\}$ of $S$ for which $z_i\in \M_j$ for all $1\leq i\leq d+1$, i.e. the underlying clean samples $z_i$ all lie on the same manifold $\M_j$; we refer to such simplices as $\textit{pure}$ simplices. Those simplices that are not pure, $S \setminus \cup_{j \in [m]} S(\M_j)$, are referred to as \textit{mixed} simplices.

\begin{definition}[Within LAPD]
\label{def:Within LAPD}
Let $S$ be the collection of valid simplices constructed from samples of $T_\tau(\M)$. Then the maximal within manifold LAPD is:
{\small
\begin{align}
\label{equ:Within LAPD}
    \text{wLAPD} &:= \max_{1\leq j \leq m} \, \max_{\Delta_1,\Delta_2 \in S(\M_j)} \text{LAPD}(\Delta_1, \Delta_2).
\end{align}
}
\end{definition}

\begin{definition}[Between LAPD]
\label{def:Between LAPD}
Let $S$ be the collection of valid simplices constructed from samples of $T_\tau(\M)$. Then the minimal between manifold LAPD is:
{\small
\begin{align}
\label{equ:between LAPD}
    \text{bLAPD} &:= \min_{i \ne j}\, \min_{\Delta_1 \in S(\M_i),\, \Delta_2 \in S(\M_j)} \text{LAPD}(\Delta_1, \Delta_2).
\end{align}
}
\end{definition}

Solving MMC using the LAPD metric becomes tenable precisely when there is a significant gap between the maximal within and minimal between LAPD, i.e. when wLAPD $\ll$ bLAPD. In Section \ref{sec:main results} we will derive high probability bounds controlling wLAPD and bLAPD (see Theorems \ref{thm:within LAPD upper bound}, \ref{thm:between LAPD lower bound}, \ref{thm:between LAPD after denoising}), which lead to conditions guaranteeing a good gap between wLAPD and bLAPD (see Theorem \ref{thm:LAPD gap}, which is the main theoretical contribution of this article). 

Figure \ref{fig:sLAPD} illustrates the nature of the LAPD metric for two linear, 2-dimensional noiseless manifolds intersecting at an angle of $\frac{\pi}{2}$. The color corresponds to the LAPD from the base simplex (shown in pink on the horizontal manifold in (a) and (c)) to all of the other simplices. As expected, the distance to all other simplices on the same manifold is small (essentially zero), while the distance to pure simplices on the vertical manifold is large in comparison (approximately 0.6) in Figure \ref{subfig:sLAPD1}. The picture is much less clear for mixed simplices formed from points close to the intersection, which have a highly variable LAPD to the base simplex. In Section \ref{ssec:method-clustering} we will introduce a denoising procedure which significantly increases the LAPD gap by removing most of these mixed simplices. The effectiveness of the denoising procedure can be verified in Figure \ref{subfig:sLAPD3}. 

\begin{figure}[tb]
\centering
\begin{subfigure}{0.32\textwidth}
\includegraphics[width=\linewidth]{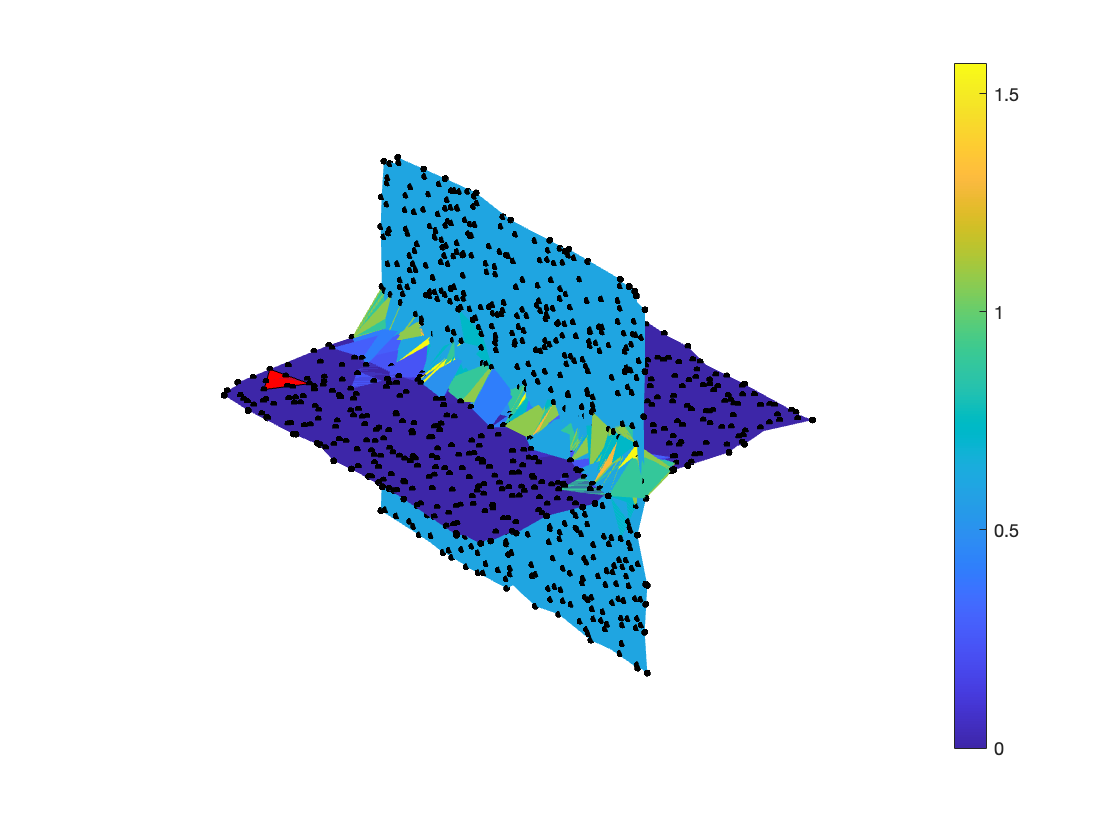}
\caption{} \label{subfig:sLAPD1}
\end{subfigure}\hspace*{\fill}
\begin{subfigure}{0.32\textwidth}
\includegraphics[width=\linewidth]{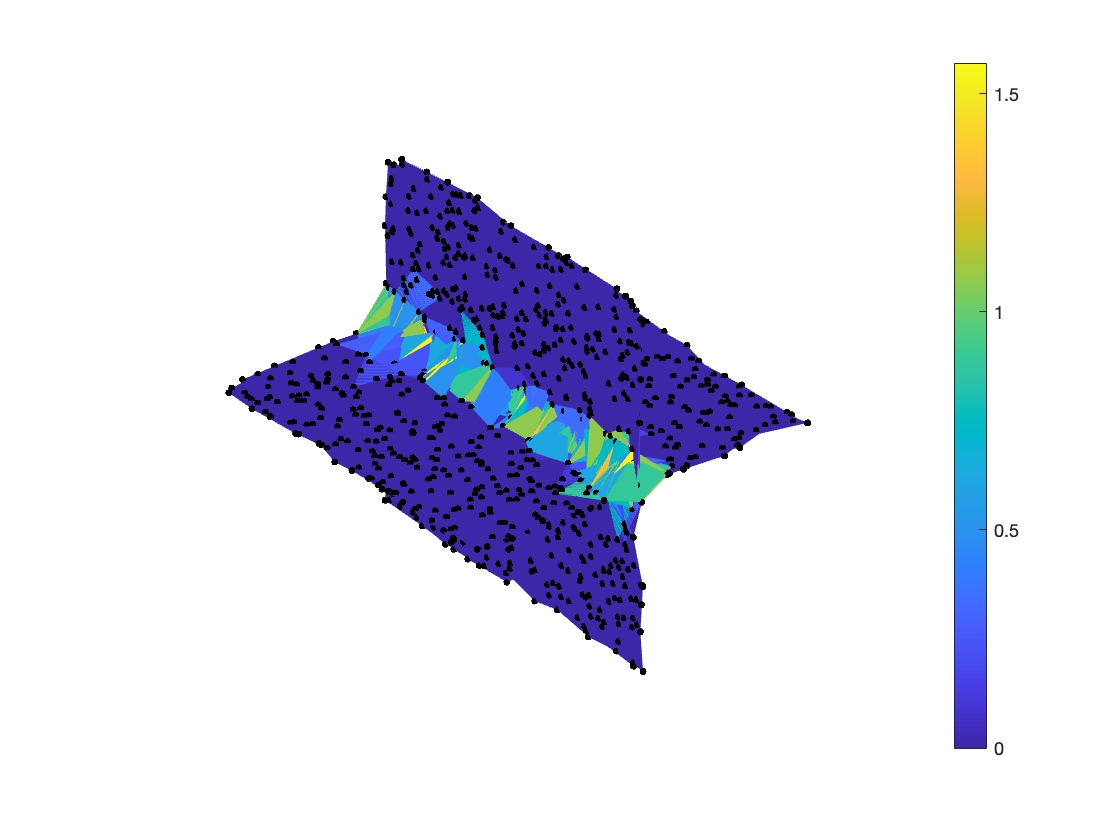}
\caption{} \label{subfig:sLAPD2}
\end{subfigure}
\begin{subfigure}{0.32\textwidth}
\includegraphics[width=\linewidth]{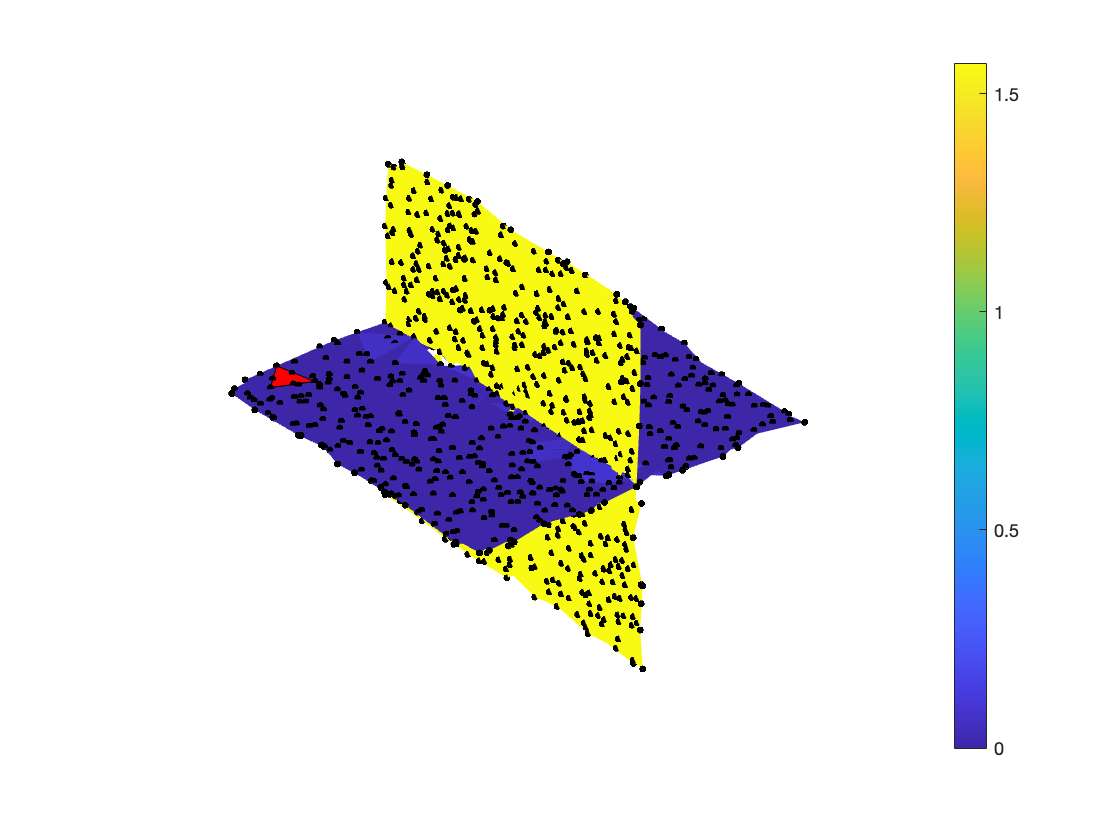}
\caption{} \label{subfig:sLAPD3}
\end{subfigure}
\caption{(a) LAPD (indicated by colors) to the target simplex (pink) before de-noising; (b) LAPD of each simplex to its $70^{\text{th}}$ nearest neighbor; (c) LAPD (indicated by colors) to the target simplex (pink) after de-noising. Black dots represent the data points. } \label{fig:sLAPD}
\end{figure}

\subsection{Denoising and Clustering}\label{ssec:method-clustering}

As illustrated in Figure \ref{subfig:sLAPD2}, mixed simplices around the intersection decrease the bLAPD and are difficult to cluster, as they can be equidistant from both manifolds. Fortunately, it is generally straightforward to identify these simplices, as they tend to have significantly larger nearest neighbor LAPD distances than the pure simplices. We thus propose the following denoising procedure: we fix a number of nearest neighbors $\kappa$ and a threshold $\eta$, and we discard from $S$ all simplices whose LAPD to their $\kappa^{\text{th}}$ nearest neighbor exceeds $\eta$. This procedure is described mathematically in the following definition. 

\begin{definition}[Denoised Simplices]
\label{def: Denoised Simplices}
Given $\kappa$ and $\eta$, we compute the LAPD between a simplex $\Delta \in S$ to its $\kappa$-th nearest neighbor as 
\begin{align}\label{eq:kappa}
    \kappa\text{NN}(\Delta) &= \min_{N\subseteq S, |N|=\kappa} \max_{\square \in N} \text{LAPD}(\Delta, \square) \, 
\end{align}
and define the set of denoised simplices as 
\begin{align}\label{eq:eta}
    S_{\text{dns}} &= \{ \Delta \in S : \kappa\text{NN}(\Delta) \leq \eta \} \, .
\end{align}
\end{definition}

After this denoising procedure, we repeat the LAPD construction outlined in Section \ref{ssec:method-lapd}, but simply replace $S$ with $S_{\text{dns}}$, i.e. LAPD is computed in $\mathcal{G}_{S_{\text{dns}}}$. Figures \ref{subfig:sLAPD2} and \ref{subfig:sLAPD3} illustrate the proposed denoising procedure. Figure \ref{subfig:sLAPD2} shows the LAPD of each simplex to its $\kappa=70^{\text{th}}$ nearest neighbor; this $\kappa\text{NN}$ distance is small for the pure simplices, but large for the mixed simplices around the intersection, and thus can be used to identify and remove these problematic simplices. Figure \ref{subfig:sLAPD3} shows that after denoising, the recomputed LAPD from the base simplex to the vertical manifold is now approximately 1.5, on the same order as the intersection angle $\frac{\pi}{2}$ between the manifolds. We will analyze the impact of denoising mathematically in Theorem \ref{thm:between LAPD after denoising}.

The last step of our algorithm is to use the LAPD metric to partition the denoised simplex set $S_\text{dns}$, and then infer a partition on the node set $X$ to solve MMC. We let $D_\text{LAPD}$ denote the $|S_\text{dns}| \times |S_\text{dns}|$ LAPD distance matrix. Given the LAPD distances, there are many available algorithms for clustering including hierarchical clustering with various linkage functions, density-based clustering, and spectral clustering. However we wish to utilize a method which can avoid explicit construction of the $D_\text{LAPD}$ matrix, which is quite large ($O(|S_\text{dns}|^2)$) and not sparse.

We thus utilize the approach developed in \cite{little2020path}, which offers a fast, multi-scale procedure to approximate single linkage hierarchical clustering of infinity path distances with a controlled error. Briefly, given a sequence of thresholds $t_1 < t_2 < \dots < t_k$, we build approximating graphs $\mathcal{G}(t_1), \dots, \mathcal{G}(t_k)$ such that each $\mathcal{G}(t_i)$ for $i \in [k]$ is constructed by discarding edge weights of magnitude greater than $t_i$ in $\mathcal{G}_{S_{\text{dns}}}$. LAPD can then be approximated simply by computing the connected components in this sequence of graphs, since if $\Delta_1, \Delta_2$ are in different connected components at scale $t_i$ but the same connected component at scale $t_{i+1}$, then the true LAPD distance between them is between $t_i$ and $t_{i+1}$. The connected components of the multi-scale graphs thus provide a collection of $k$ slices of the full LAPD single linkage dendrogram; see \cite{little2020path} for more details, including an efficient implementation of spectral clustering on infinity path distances (for simplicity we simply utilize hierarchical clustering in this article). Table \ref{table:single linkage clustering approximation} illustrates this multi-scale procedure on a toy example of 16 simplices; 8 scales (appearing in the first row) are used to approximate the metric, and the table shows the indices of the connected components at each scale. The simplices start in distinct connected components but eventually merge into a single connected component at the last scale, provided the graph is fullly connected. 

\begin{table}[tbh]
 \begin{center}
\begin{tabular}{|l||*{8}{c|} }
    \hline
\backslashbox[25mm]{simplices}{scales} 
            & \multicolumn{1}{c|}{1e-8}
            & \multicolumn{1}{c|}{3e-7}
            & \multicolumn{1}{c|}{1.4e-5}
            & \multicolumn{1}{c|}{6.3e-4}
            & \multicolumn{1}{c|}{0.028}
            & \multicolumn{1}{c|}{0.3418}
            & \multicolumn{1}{c|}{0.8866}
            & \multicolumn{1}{c|}{1.5707}\\
    \hline\hline
$\Delta_1$ & 1 & 1 & 1 & 1 & 1 & 1 & 1 & 1\\ \hline
$\Delta_2$ & 2 & 1 & 1 & 1 & 1 & 1 & 1 & 1\\ \hline
$\Delta_3$ & 3 & 1 & 1 & 1 & 1 & 1 & 1 & 1\\ \hline
$\Delta_4$ & 4 & 2 & 2 & 2 & 1 & 1 & 1 & 1\\  \hline 
$\Delta_5$ & 5 & 2 & 2 & 2 & 1 & 1 & 1 & 1\\ \hline
$\Delta_6$ & 6 & 3 & 3 & 2 & 1 & 1 & 1 & 1\\ \hline
$\Delta_7$ & 7 & 3 & 3 & 2 & 1 & 1 & 1 & 1\\ \hline
$\Delta_8$ & 8 & 3 & 3 & 2 & 1 & 1 & 1 & 1\\ \hline
$\Delta_9$ & 9 & 4 & 4 & 3 & 2 & 2 & 2 & 1\\ \hline
$\Delta_{10}$ & 10 & 4 & 4 & 3 & 2 & 2 & 2 & 1\\ \hline
$\Delta_{11}$ & 11 & 5 & 4 & 3 & 2 & 2 & 2 & 1\\ \hline
$\Delta_{12}$ & 12 & 5 & 4 & 3 & 2 & 2 & 2 & 1\\ \hline
$\Delta_{13}$ & 13 & 5 & 4 & 3 & 2 & 2 & 2 & 1\\ \hline
$\Delta_{14}$ & 14 & 6 & 5 & 3 & 2 & 2 & 2 & 1\\ \hline
$\Delta_{15}$ & 15 & 6 & 5 & 3 & 2 & 2 & 2 & 1\\ \hline
$\Delta_{16}$ & 16 & 7 & 5 & 3 & 2 & 2 & 2 & 1\\ \hline
\end{tabular}
\end{center}
 \caption{The merging procedure in \cite{little2020path} to approximate hierarchical clustering. A toy example of 16 simplices and 8 scales are used here to demonstrate the idea. Each cell contains the index of the connected component at that scale. Simplices start in distinct components but eventually merge into a single connected component at the last scale. } 
\label{table:single linkage clustering approximation}
\end{table}

An important question when performing distance-based clustering is how to choose the number of clusters $m$ (indeed, we see different numbers of clusters at different scales). When this parameter is unknown, we approximate the number of clusters using a persistence approach, i.e. letting $m(t_1), \ldots, m(t_k)$ denote the number of nontrivial connected components of $\mathcal{G}(t_1), \dots, \mathcal{G}(t_k)$, we define
\begin{align}
\label{equ:estimate_num_clusters}
    c_j &= \sum_{i=1}^k \mathbf{1}\{m(t_i)=j\}\;,\quad \hat{m} = \argmax_j c_j \, .
\end{align}
For example, we would infer $\hat{m}=2$ in the toy example shown in Table \ref{table:single linkage clustering approximation}. When $m$ is known, we simply return the hierarchical clustering corresponding to that $m$. This heuristic of identifying the number of clusters with the longest branches is well known \citep{Belyadi2021,GERE2023} and aligns with the theory of the gap statistic \citep{Tibshirani2002} and the variance ratio criterion \citep{Calinski1974}. One advantage of our proposed method is that it can usually learn the correct number of manifolds $m$ in the mixture, while the methods we compare with require this parameter as an input (see the experiments in Section \ref{sec:Experiments}).

Finally, we extend the simplex partition to a node partition by a majority vote procedure. More specifically, letting $\ell(\Delta)\in [m]$ denote the simplex partition and $S_{\text{dns}}(x_i)$ the subset of simplices containing $x_i$, we define the label assigned to a node $x_i$ as 
\begin{align}\label{equ:majority_vote}
    \ell(x_i) &= \argmax_{j \in [m]} \sum_{\Delta \in S_{\text{dns}}(x_i)} \mathbf{1}\{\ell(\Delta)=j\} \, .
\end{align}

Alternatively, one could define a node-level distance by incorporating an additional minimum over all pairs of simplices containing the nodes $x_i, x_j$; namely, 
\begin{equation}
    \label{eq:nLAPD}
     \text{nLAPD}(x_i,x_j) := \min_{\Delta_1 \in S_{\text{dns}}(x_i),\, \Delta_2 \in S_{\text{dns}}(x_j)} \text{LAPD}(\Delta_1, \Delta_2)\, , 
\end{equation}
and then cluster the nodes using nLAPD. Experiments indicate that clustering the nodes with nLAPD is quite similar to clustering the simplices with LAPD and then inferring a partition on the nodes. However, the additional minimum takes an $O(n^2)$ search and is thus not computationally feasible; we thus utilize the simplex level LAPD for clustering.


\section{Main Results}
\label{sec:main results}

In this section we analyze the proposed LAPD metric, bound wLAPD and bLAPD with high probability (Theorems \ref{thm:within LAPD upper bound}, \ref{thm:between LAPD lower bound}, and \ref{thm:between LAPD after denoising}), and derive conditions guaranteeing a good LAPD gap (wLAPD $\ll$ bLAPD), so that the metric can successfully solve MMC with high probability (Theorem \ref{thm:LAPD gap}). To simplify the analysis, we consider the following two components, linear model. 
\begin{assumption}{(Linear manifolds)}
\label{assump:linear}
 $\M=\M_1 \cup \M_2$ for two \textit{linear} manifolds $\M_1, \M_2$ intersecting at angle $\Theta$.
\end{assumption}
We discuss extensions to the general case at the end of this section.

We start by deriving an upper bound on wLAPD in terms of the noise level $\tau$ and simplex edge length $e$; Theorem \ref{thm:within LAPD upper bound} is in fact not a probabilistic statement, but derived from a deterministic ``worst case" bound on how much bending can occur for pure simplices, i.e. simplices constructed from a single manifold. The key idea is that if the noise level $\tau$ is small relative to the simplex size $e$, then all pure simplices will be nearly parallel to the underlying manifold, and wLAPD will be small. 

\begin{restatable}[wLAPD upper bound]{theorem}{first}
\label{thm:within LAPD upper bound}
Let Assumptions \ref{assump:intersection_dim}, \ref{assump:uniform_samp}, \ref{assump:size-quality}, \ref{assump:small noise} and \ref{assump:linear} hold. Assume also the construction outlined in Section \ref{sec:methodology} with either weight construction \eqref{eq:simplex graph weight} or \eqref{eq:two way weight}, and that $\mathcal{G}_S$ is connected. Then: 
    \begin{equation}\label{eq:within LAPD upper bound}
        \text{wLAPD} \lesssim \frac{\tau}{e}. 
    \end{equation}
\end{restatable}
\begin{proof}
    See Appendix \ref{pf:within LAPD upper bound}. 
\end{proof}

\begin{remark}
    The upper bound in \eqref{eq:within LAPD upper bound} is independent of the ambient dimension $D$; in particular,  $\text{wLAPD} \leq C_{d,q}\frac{\tau}{e}$ for a constant $C_{d,q}$ depending on the intrinsic dimension $d$ and simplex quality $q$. 
\end{remark}

\begin{remark}
    Our goal is for $\text{wLAPD} \lesssim \Theta$, i.e. for the within manifold distance to be small relative to the manifold intersection angle, as this will be necessary to achieve a good LAPD gap; \eqref{eq:within LAPD upper bound} guarantees this will be the case as long as $e \gtrsim \frac{\tau}{\Theta}$, i.e. we choose a scale $e$ for the simplex construction which is large relative to the noise level $\tau$ and the inverse of the intersection angle. 
\end{remark}

Next we establish an upper bound on bLAPD, the between manifold distance. Theorem \ref{thm:between LAPD lower bound} is derived from probabilistic arguments that when sampling randomly it is unlikely to obtain a long chain of mixed simplices near the intersection which all connect at a flat angle, i.e. the formation of a flat bridge in between manifolds is unlikely. We introduce the following non-restrictive assumption in order to give a precise bound on this distance.

\begin{assumption}
\label{assump:volume_growth}
There exists a constant $R_{\max}$ such that for all $R\in[\tau,R_{\max}]$ 
\begin{align}
\label{equ:elong}
    \text{vol}_{D}(\mathcal{T}_R) &\in [C_1, C_2] R \tau^{D-d}\, ,
\end{align}   
where $\mathcal{T}_R=T_R(\M_1 \cap \M_2)\cap T_\tau(\M)$ denotes the subset of the data support within distance $R$ of $\M_1 \cap \M_2$. 
\end{assumption}

The constants $C_1, C_2$ encode the elongation of $\M$ in the direction of the intersection; in particular, $C_1,C_2$ will scale as $\text{vol}_{d-1}(\M_1 \cap \M_2)$. Describing $\text{vol}_{D}(\mathcal{T}_R)$ with an interval is convenient, since if the boundaries of $\M_1,\M_2$ curve, $\text{vol}_D(\mathcal{T}_R)/R$ will have small variations as $R$ changes.

\begin{restatable}[bLAPD lower bound]{theorem}{second}
\label{thm:between LAPD lower bound}
    Let Assumptions \ref{assump:intersection_dim}, \ref{assump:uniform_samp}, \ref{assump:size-quality}, \ref{assump:small noise}, \ref{assump:linear}, and \ref{assump:volume_growth} hold and assume the construction outlined in Section \ref{sec:methodology} with either weight construction \eqref{eq:simplex graph weight} or \eqref{eq:two way weight}. If $e$ satisfies $n e^d \geq \frac{2(d+1)}{C_5\Theta\log 2} \log n$ and $e \leq 1$, then for $n$ large enough
    \begin{align*}
        \text{bLAPD} &\geq (2C_5 n e^d)^{-1}
    \end{align*}
    with probability at least $1-3C_4 n^{-{(d+1)}}$, where $C_4 = \left(\frac{2C_2}{q\sin\Theta}\right)^{d+1}$, $C_5 = \frac{2C_2 C_{\Theta,q,d}}{C_1} $, $C_1,C_2$ are the constants from \eqref{equ:elong}, and $C_{\Theta,q,d,r_0}$ is the constant from Proposition \ref{prop:single_link_prob}. 
\end{restatable}
\begin{proof}
See Appendix \ref{pf:between LAPD lower bound}. A key technical step is the single link volume estimate given in Proposition \ref{prop:single_link_prob}, which is proved in Appendix \ref{app:single-link}. 
\end{proof}

Theorem \ref{thm:between LAPD lower bound} guarantees that when there is no noise, $\text{wLAPD}=0 \ll (\log n)^{-1} \lesssim \text{bLAPD}$, and the metric can differentiate the manifolds; this is obtained by choosing $e$ as small as possible subject to the lower bound $ne^d \gtrsim \log n$. However when observations are corrupted by noise and the wLAPD can become as large as $O(\frac{\tau}{e})$ as quantified in Theorem \ref{thm:within LAPD upper bound}, the metric would only be able to differentiate the manifolds for very small values of $\tau$, namely $\tau \lesssim (\log n)^{-\frac{(d-1)}{d}}n^{-\frac{1}{d}}$. Fortunately, if we incorporate an appropriate denoising procedure, we can increase the bLAPD significantly without compromising the wLAPD. In particular, by incorporating the denoising procedure from Definition \ref{def: Denoised Simplices}, one can obtain a bLAPD scaling like the intersection angle $\Theta$; despite the much stronger conclusion, the proof of Theorem \ref{thm:between LAPD after denoising} is a minor modification of the proof of Theorem \ref{thm:between LAPD lower bound}. 

\begin{restatable}[bLAPD lower bound with denoising]{theorem}{third}
\label{thm:between LAPD after denoising}
Let Assumptions \ref{assump:intersection_dim}, \ref{assump:uniform_samp}, \ref{assump:size-quality},  \ref{assump:small noise}, \ref{assump:linear}, and \ref{assump:volume_growth} hold and assume the construction outlined in Section \ref{sec:methodology} with either weight construction \eqref{eq:simplex graph weight} or \eqref{eq:two way weight}. Assume also that $e$ satisfies $n e^d \geq \frac{8(d+1)}{C_5\Theta\log 2} \log n $, $e \leq 1$ and that we remove all simplices whose LAPD to their $\kappa=\frac{\Theta}{4}C_5 ne^d$-th nearest neighbor exceeds $\eta = (2C_5ne^d)^{-1}$. Then for $n$ large enough, if LAPD is recomputed after denoising,
\begin{align*}
   \text{bLAPD}_{\text{dns}} \geq \frac{\Theta}{4} \, 
\end{align*}
with probability at least $1-3C_4 n^{-{(d+1)}}$, where constants $C_4$, $C_5$ are as in Theorem \ref{thm:between LAPD lower bound}.
\end{restatable}

\begin{proof}
    See Appendix \ref{pf:between LAPD after denoising}. 
\end{proof}

\begin{remark}
\label{re:denoising parameters}
    Again, the best possible result is obtained by choosing $e\sim \left(\frac{\log n}{n}\right)^{\frac{1}{d}}$ so that $\kappa \sim \log n$ and $\eta \sim (\log n)^{-1}$; in this setting the bLAPD is proportional to $\Theta$ with high probability.
\end{remark}

Finally, we present the main result of the article, which combines Theorems \ref{thm:within LAPD upper bound} and \ref{thm:between LAPD after denoising} to obtain a condition guaranteeing that LAPD can solve MMC. We remark that the denoising procedure will not impact the wLAPD upper bound computed in Theorem \ref{thm:within LAPD upper bound} as long as the graph stays connected. 

\begin{theorem}[LAPD gap]
    \label{thm:LAPD gap}
    Let Assumptions \ref{assump:intersection_dim}, \ref{assump:uniform_samp}, \ref{assump:size-quality},  \ref{assump:small noise}, \ref{assump:linear}, and \ref{assump:volume_growth} hold and assume the construction outlined in Section \ref{sec:methodology} with either weight construction \eqref{eq:simplex graph weight} or \eqref{eq:two way weight}. 
    In particular, assume $S_{\text{dns}}$ is obtained by removing all simplices whose $\kappa^{\text{th}}$ nearest neighbor LAPD exceeds $\eta$ for $\kappa \sim ne^d$, $\eta \sim (ne^d)^{-1}$, and that $\mathcal{G}_{S_{\text{dns}}}$ is connected.
    For $n$ large enough, if $e$ satisfies 
    \begin{equation}
    \label{eq:LAPD gap}
        \underbrace{\frac{\tau}{\Theta} }_{\text{wLAPD}} \vee  \underbrace{\left( \frac{\log n}{n} \right)^{\frac{1}{d}}}_{\text{bLAPD}} \lesssim e \leq 1,
    \end{equation}
    then wLAPD $\lesssim$  bLAPD with probability $1-O(n^{-(d+1)})$. 
\end{theorem}
\begin{proof}
By Theorem \ref{thm:within LAPD upper bound}, $\text{wLAPD}\lesssim \frac{\tau}{e}$. Applying Theorem \ref{thm:between LAPD after denoising}, $\text{bLAPD}\gtrsim \Theta$ as long as $ n e^d \geq \frac{8(d+1)}{C_5\Theta\log 2} \log n $ holds, i.e. for $e \gtrsim \left( \frac{\log n}{n} \right)^{\frac{1}{d}}$ which is guaranteed by \eqref{eq:LAPD gap}. We thus obtain $\text{wLAPD} \lesssim \frac{\tau}{e} \lesssim \Theta \lesssim \text{bLAPD}$.
\end{proof}

A few observations are in order: Theorem \ref{thm:LAPD gap} illustrates that for LAPD to successfully differentiate the manifolds, simplices must be contructed at a lengthscale which is above the noise, and a small intersection angle $\Theta$ leads to a smaller range of parameters where the approach succeeds. 
The upper bound $e\leq 1$ is quite non-restrictive (assumed for convenience in the proof of Theorem \ref{thm:between LAPD lower bound}), as throughout the article we consider $e$ as a local scaling. Nonetheless  
when manifolds are curved, the curvature will require a much more restrictive upper bound on the simplex size $e$ than the one appearing in Theorem \ref{thm:LAPD gap}. When the curvature is large, the wLAPD can become large where the manifold bends rapidly, and simplices must be kept small to mitigate this. We do not compute this curvature-based upper bound explicitly, as a precise characterization of the volumes arising in our chaining arguments is already cumbersome even in the linear case (see Proposition \ref{prop:single_link_prob}). From an algorithmic perspective, the strategy is already clear: take $e$ as small as possible while still satisfying \eqref{eq:LAPD gap} will allow the algorithm to tolerate as much curvature as possible. 

\section{Numerical Implementation \& Experiments}
\label{sec:Experiments}

This section provides implementation details and experimental results. In particular, Subsection \ref{sec:implementation} presents pseudo-code for the LAPD algorithm, analyzes the computational complexity, and discusses a variant of the weight definition in \eqref{eq:simplex graph weight} which enhances performance in some settings. We then conduct numerous experiments investigating the performance of LAPD and compare against two benchmark subspace clustering methods SSC-ADMM \citep{elhamifar_sparse_2011} and EKSS \citep{Lipor_2020} and three benchmark manifold clustering methods LocPCA \citep{arias-castro_spectral_2017}, PBC \citep{babaeian_multiple_2015}, and DCV \citep{wu2022deep}. Subsection \ref{sec:subspace clustering} reports results for subspace clustering on synthetic data, Subsection \ref{sec:manifold_clustering} reports results for nonlinear multimanifold clustering on synthetic data, and Subsection \ref{sec:real world} reports performance on real data. 

\subsection{Numerical Implementation}
\label{sec:implementation}

The pseudo-code of the LAPD algorithm is given in Algorithm \ref{alg:pseudo}. Our implementation\footnote{\url{https://github.com/HYfromLA/LAPD}} provides default values to the parameters ($B=25$, $e=\sqrt{2}\tau$, $q=\frac{1}{1.25+0.15(d-2)}$, $\kappa=10\log{n}$, $k=100$) to make the code easy to use, and estimates the intrinsic dimension $d$ and noise level $\tau$ using \cite{LITTLE2017504} if not provided by the user. Note that the choice of $B,k$ is to keep computation and memory use feasible, the choice of $e$ is to ensure the simplex size to be above the noise level, and the choice of $q$ and $\kappa$ is empirical. The default value of $\eta$ is set to be the elbow point when all the $\kappa$NN LAPDs of the simplices are sorted in a non-decreasing order. An illustration is given in Figure \ref{fig:default eta}. Careful readers may find that the parameter set of Algorithm 1 does not include $r_0$, the one that constraints the volume of the simplices. The reason is quite simple: keeping track of the volume of each single simplex is computationally unfeasible and empirically, we observe that characterizing the quality of simplices with $e,q$ is sufficient to guarantee a good performance of the algorithm. 

\begin{figure}[tb]
     \centering
     \begin{subfigure}[b]{0.45\textwidth}
         \centering        \includegraphics[width=\textwidth]{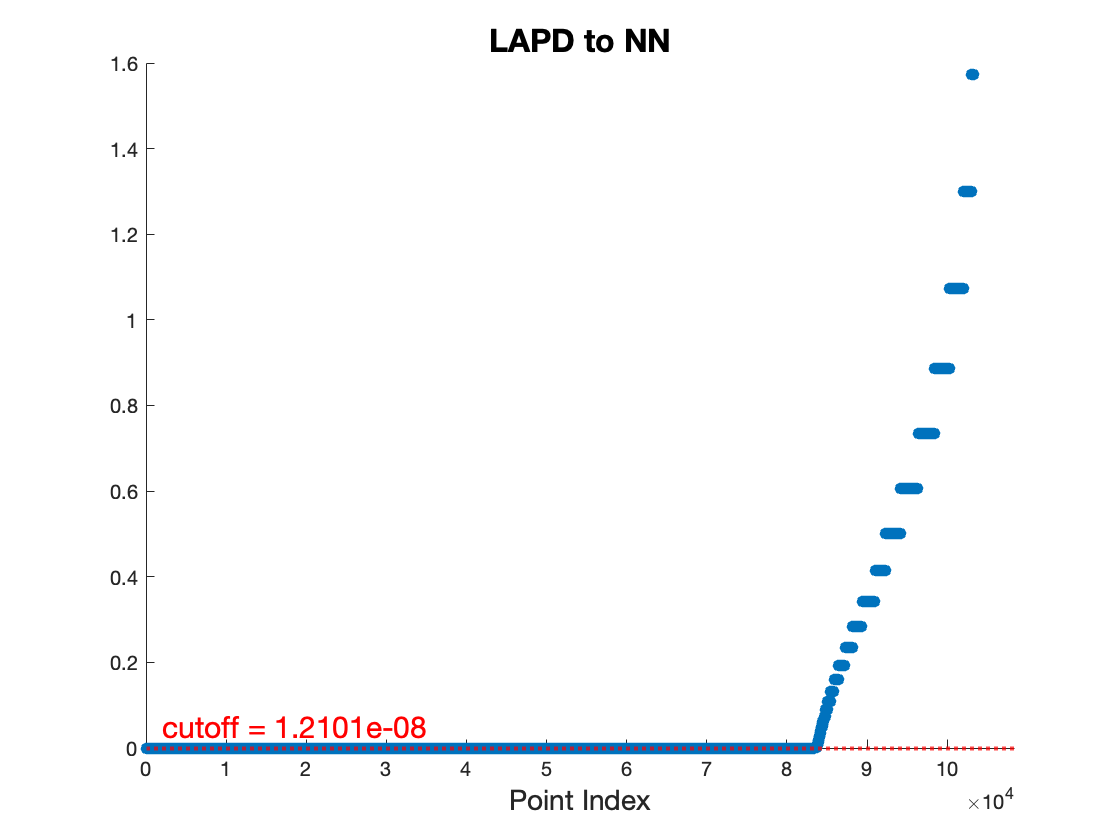}
         \caption{Default $\eta$ for a noiseless dataset.}
         \label{subfig: eta noiseless}
     \end{subfigure}
     \hfill
     \begin{subfigure}[b]{0.45\textwidth}
         \centering       \includegraphics[width=\textwidth]{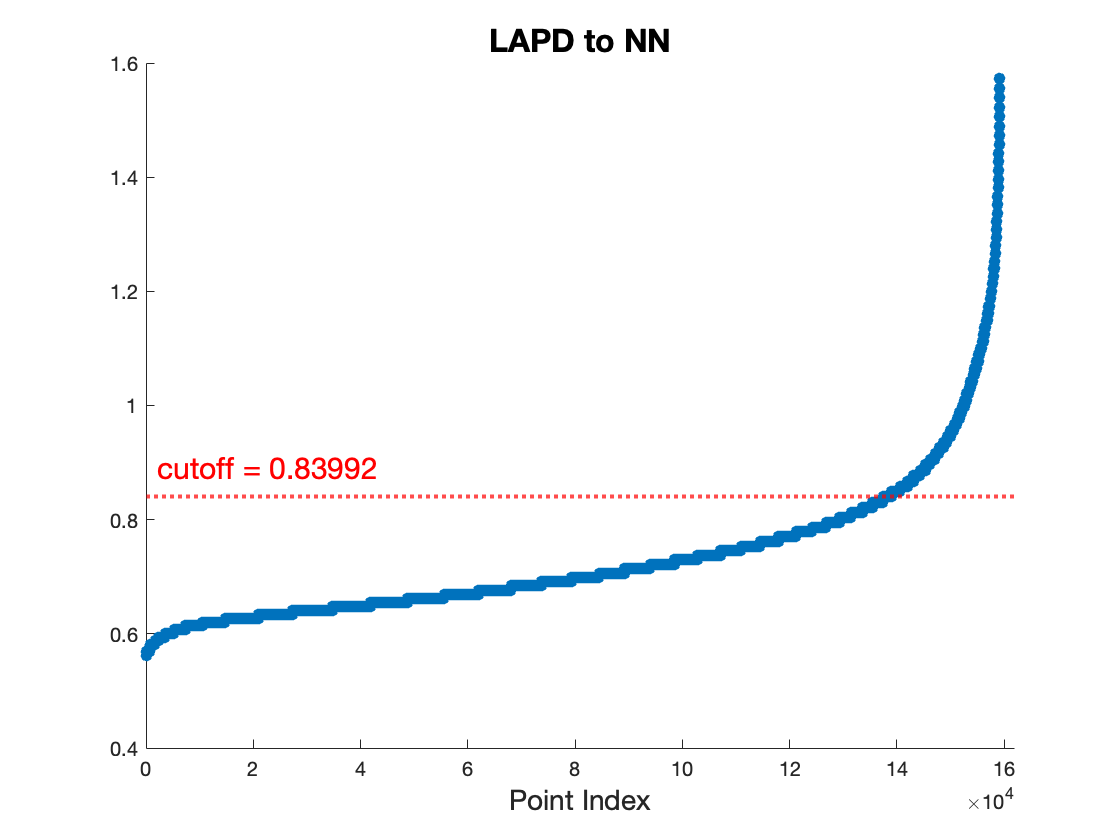}
         \caption{Default $\eta$ for same data with added noise.}
          \label{subfig:eta noisy}
     \end{subfigure}
     \caption{Default $\eta$ (cutoff) from the elbow point criteria; figure on the left is for noiseless data, and figure on the right is for the same data with added noise.}
     \label{fig:default eta}
\end{figure}

Based on the pseudo-code, the total complexity of LAPD as an algorithm for MMC can be calculated as follows: (i) searching and finding nearest neighbors on dataset $X$ takes $O(Dn\log(n))$ using KD-Tree; (ii) constructing the weighted simplex graph takes $nB^{d}\frac{d(d+1)}{2}D = O(Dnd^2B^d)$; (iii) using the multiscaled graph procedure from \cite{little2020path} to approximate a pruned single linkage dendrogram takes \small\[O\left(ndB^{d+1} \left(t \wedge \log \left(ndB^{d+1}\right)\right) + nB^dt\log \left(nB^d\right)\right)\]
(iv) denoising takes $(nB^d)\kappa = O(dB^d n\log(n))$. Putting all the pieces together gives the total complexity of the LAPD algorthm 
\small \[  O\left(\left(D+dB^d\right)n\log n +\left( Dd^2+dB\left(k \wedge \log \left( ndB^{d+1}\right)\right)\right)nB^d \right) = O\left(\left(D+B^d\right)n\log n + DnB^d \right),\]
assuming $d, k, B=O(1)$, i.e. the LAPD algorithm is linear in the ambient dimension $D$ and quasi linear in the dataset size $n$. We verify these computational dependencies with experiments (see Figure \ref{fig:Runtime}), while most of the benchmark algorithms are at least quadratic. On the other hand, the complexity of LAPD is exponential in the intrinsic dimension $d$, both memory-wise and computation-wise, which makes it hard to use for datasets of large intrinsic dimension (normally when $d > 5$). 

\begin{algorithm}[tb]
    \SetAlgoLined
    \KwInput{Data $X=\{x_1, \ldots, x_n\}$; Optional parameters $d, \tau, B,e,q,\kappa, \eta, k, m.$}
    \KwOutput{Clustering labels on $X$.}
            \vspace{.2cm}
            Estimate intrinsic dimension $d$ and noise level $\tau$ using the algorithm from \cite{LITTLE2017504} if not provided by user. \\
            [.1cm]
            Construct annular locality graph $\mathcal{G}_X(e,B)=(X,E_X)$ on $X$ and candidate simplices via \eqref{equ:cand_simplices}. \\[.1cm]
            Filter simplices with quality parameter $q$ via \eqref{eq:simplex quality constraint} to obtain simplex set $S$. \\ [.1cm]
            Construct an angle-based weighted graph on $S$: $\mathcal{G}_S\leftarrow(S,E_S,W_S)$ with $W_S(\Delta_i,\Delta_j)$ from \eqref{eq:simplex graph weight} or \eqref{eq:two way weight} dependent on data distribution. \\[.1cm]  
            Approximate LAPD in $\mathcal{G}_S$ from $k$ scales using the multi-scale procedure in \cite{little2020path}. \\ [.1cm]
            Denoise $S$ according to \eqref{eq:kappa} and \eqref{eq:eta} with thresholds $\kappa$, $\eta$ to obtain $S_\text{dns}$ and re-compute LAPD in $\mathcal{G}_{S_\text{dns}}$. \\ [.1cm]
            Use the multi-scale procedure \cite{little2020path} to approximate hierarchical clustering on $S_\text{dns}$. \\ [.1cm] 
            If unknown, estimate number of clusters $m$ via \eqref{equ:estimate_num_clusters} (persistence) and extract corresponding labels.  \\ [.1cm] 
            Extend the labels assigned to $S_\text{dns}$ to labels on $X$ via a majority vote (see \eqref{equ:majority_vote}).
\caption{LAPD for Multi-Manifolds Clustering.}
\label{alg:pseudo}
\end{algorithm}

We also introduce the following weight definition: 
\begin{equation}
    \label{eq:two way weight}
    W_S^2(\Delta_i, \Delta_j) = \min \{\pi-\theta(\Delta_i, \Delta_j), \theta(\Delta_i, \Delta_j) \} 
\end{equation}
as an alternative to \eqref{eq:simplex graph weight}. We use the superscript 2 to distinguish this second version of weight definition from \eqref{eq:simplex graph weight}. Indeed, \eqref{eq:two way weight} allows ``backflipping", i.e. if two simplices connect with a dihedral angle close to 0, this is also deemed a flat connection, while \eqref{eq:simplex graph weight} requires that the dihedral angle be close to $\pi$ for a flat connection. If we refer to the LAPD defined with \eqref{eq:simplex graph weight} as LAPD$^1$ and respectively the LAPD defined with \eqref{eq:two way weight} as LAPD$^2$, then it is straightforward to see that 
\begin{equation}
\label{eq:two LAPD relationship}
    \text{LAPD}^1(\Delta_i, \Delta_j) \geq \text{LAPD}^2(\Delta_i, \Delta_j) \, ,
\end{equation}
since all simplices are more connected (closer) under LAPD$^2$. Using LAPD$^2$ thus decreases the wLAPD (better within-manifold cohesiveness) at the cost of also decreasing the bLAPD (worse between manifold separability), and it is a priori unclear what is most advantageous for multi-manifold clustering, i.e. which metric will lead to the largest LAPD gap. Although our theoretical guarantees cover both constructions, we found LAPD$^1$ to be preferable for the synthetic data sets considered in Sections \ref{sec:subspace clustering} and \ref{sec:manifold_clustering} while LAPD$^2$ was preferable for the real data sets considered in Section \ref{sec:real world}. We conjecture that this is because in our synthetic data sets, the data density is high at the intersections, while for the real data sets the data density near class intersections tends to be lower (clusters are more separated or even disjoint); when the intersections are of high density, the risk of a bad path forming is greatly increased by allowing backflipping. Note if no superscript is specified, LAPD refers to the default LAPD$^1$.

\subsection{Subspace Clustering}
\label{sec:subspace clustering}

We first evaluate LAPD on a linear subspace clustering task using synthetic data. Each dataset consists of a pair of intersecting unit hypercubes: for \(d = 1\) or \(d = 2\), the data represent intersecting unit vectors or squares, respectively. We fix the dataset size at \(n = 6{,}000\), with \(3{,}000\) points drawn from each subspace, and set the ambient dimension to \(D = 100\). To introduce noise, we add a uniform random variable in each coordinate from the interval $[-\frac{\sqrt{3}\,\sigma}{\sqrt{D-d}}, \; \frac{\sqrt{3}\,\sigma}{\sqrt{D-d}}]$
resulting in a global noise variance of \(\sigma^2\) and a global noise level \( \tau = \sqrt{3}\sigma\).
The experiment proceeds as follows: for \(d \in \{1,2\}\), we test intersection angles \(\Theta \in \left\{\tfrac{\pi}{2}, \tfrac{\pi}{4}, \tfrac{\pi}{6}\right\}\). For each pair \((d, \Theta)\), we vary \(\sigma\) from \(0\) to \(0.1\) in increments of \(0.01\), and so equivalently the global noise level \(\tau\) from \(0\) to \(0.1\sqrt{3}\) in increments of \(0.01\sqrt{3}\).

For LAPD, we record the clustering accuracy both from using the default simplex size parameter $e$ (LAPD-D) and from using manually tuned $e$ (LAPD-T) in Figure \ref{fig:SSC results} (all other parameters were set to the default values). We compare against SSC-ADMM \citep{elhamifar_sparse_2011} and EKSS \citep{Lipor_2020}, methods designed for subspace clustering, as well as LocPCA \citep{arias-castro_spectral_2017}, which projects neighborhoods of data points onto a tangent subspace, and PBC \citep{babaeian_multiple_2015}, a path-based method; we omit comparison with the deep learning method DCV \citep{wu2022deep} in this section since it fails even without noise. Note all comparison algorithms require the ground truth number of clusters $m$. To make the comparison fair, we provide $m$ to LAPD as well. We shall explore LAPD's capability of automatically learning the number of clusters in Subsection \ref{sec:manifold_clustering}.

\begin{figure}[tb]
     \centering
     \begin{subfigure}[b]{0.32\textwidth}
         \centering        \includegraphics[width=\textwidth]{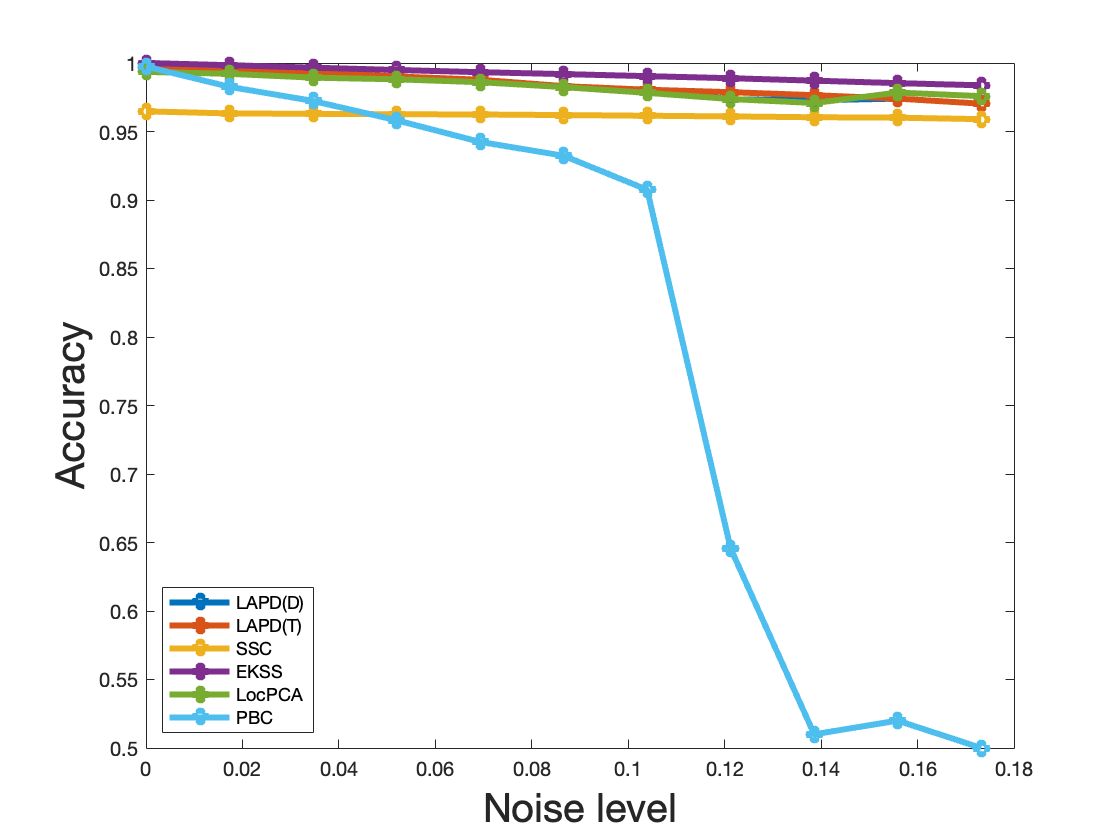}
         \caption{$d=1, \Theta=\frac{\pi}{2}$}
     \end{subfigure}
     \hfill
     \begin{subfigure}[b]{0.32\textwidth}
         \centering        \includegraphics[width=\textwidth]{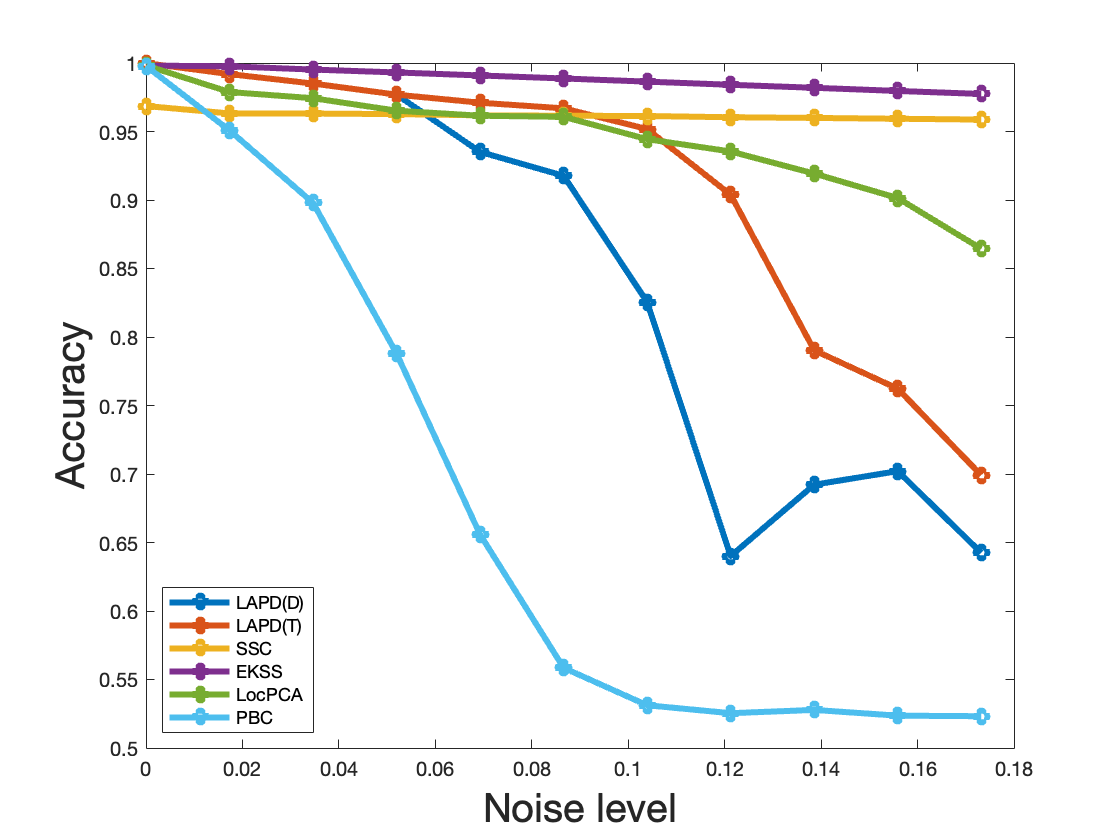}
         \caption{$d=1, \Theta=\frac{\pi}{4}$}
     \end{subfigure}
     \hfill
     \begin{subfigure}[b]{0.32\textwidth}
         \centering        \includegraphics[width=\textwidth]{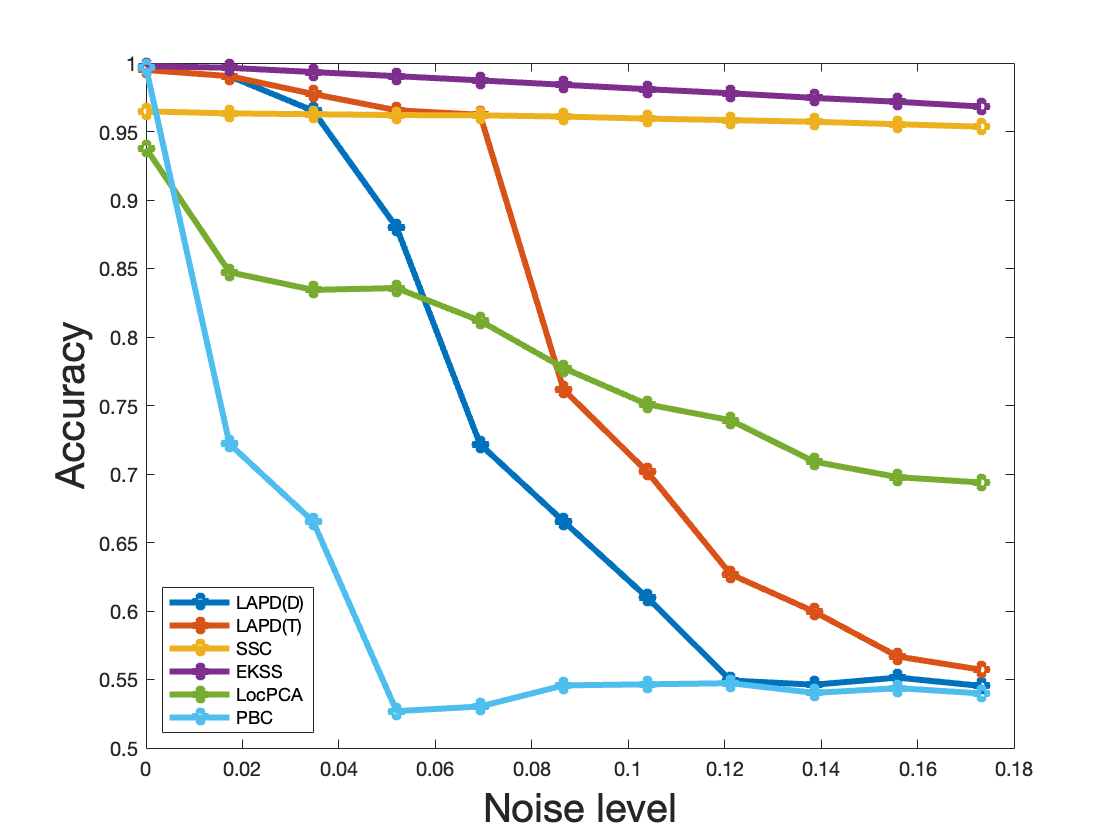}
         \caption{$d=1, \Theta=\frac{\pi}{6}$}
     \end{subfigure}
     \begin{subfigure}[b]{0.32\textwidth}
         \centering        \includegraphics[width=\textwidth]{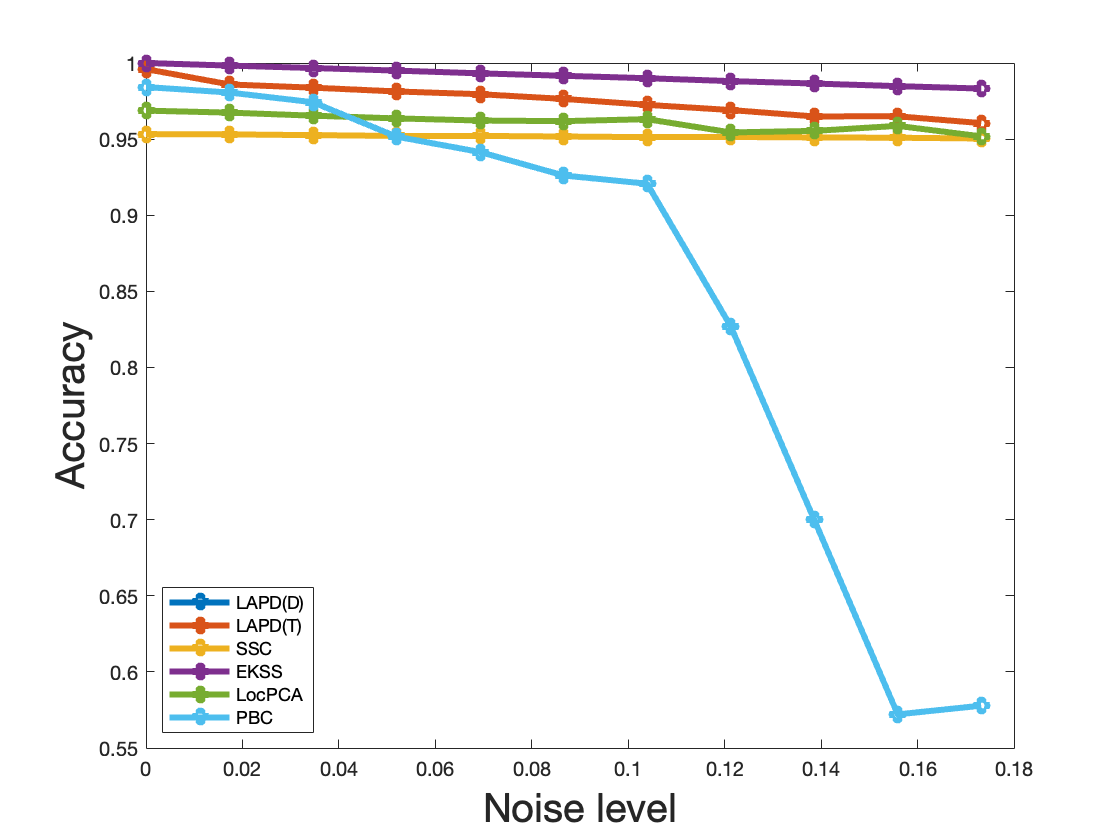}
         \caption{$d=2, \Theta=\frac{\pi}{2}$}
     \end{subfigure}
     \hfill
     \begin{subfigure}[b]{0.32\textwidth}
         \centering         \includegraphics[width=\textwidth]{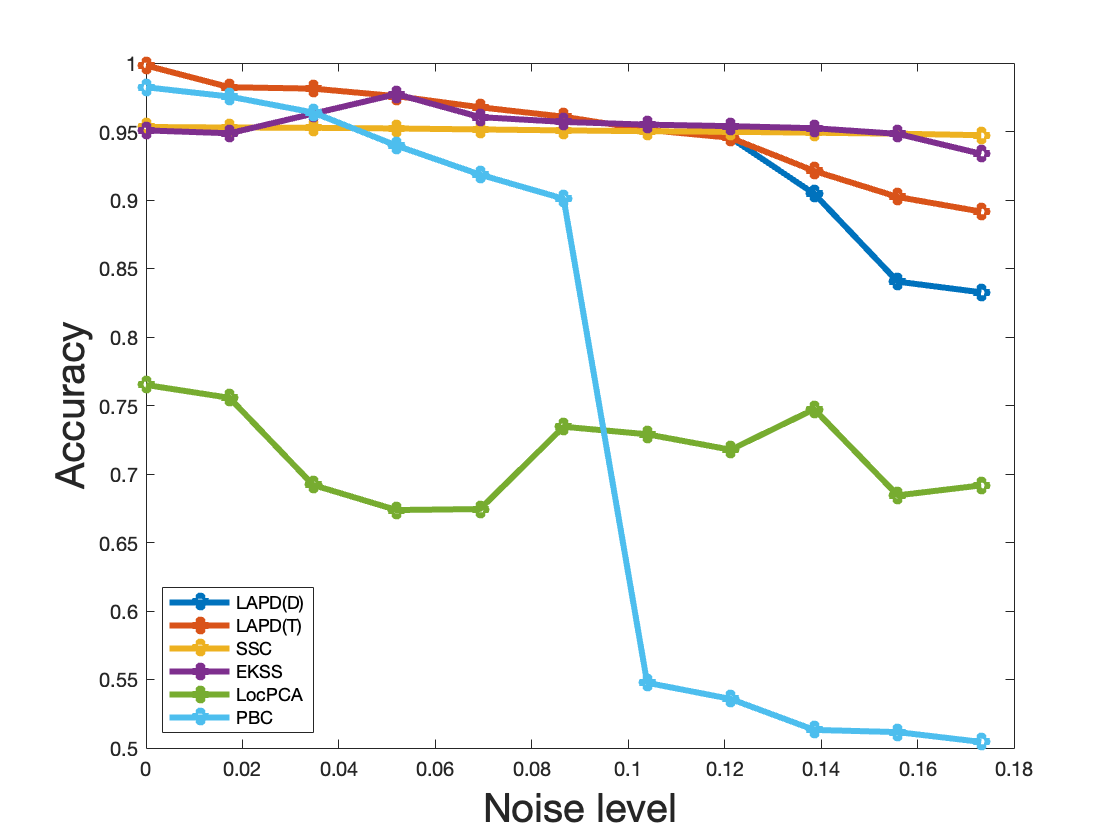}
         \caption{$d=2, \Theta=\frac{\pi}{4}$}
     \end{subfigure}
     \hfill
     \begin{subfigure}[b]{0.32\textwidth}
         \centering         \includegraphics[width=\textwidth]{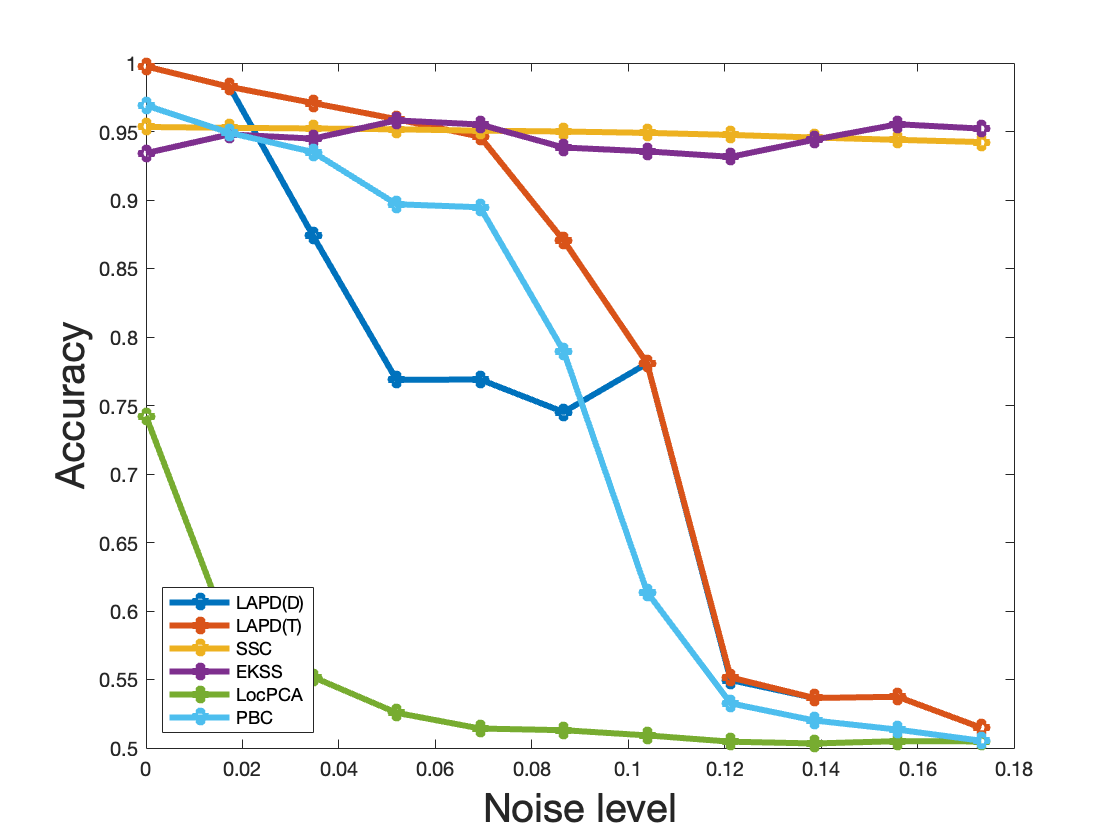}
         \caption{$d=2, \Theta=\frac{\pi}{6}$}
     \end{subfigure}
        \caption{Subspace clustering accuracy.}
        \label{fig:SSC results}
\end{figure}

Comparisons for the subspace clustering tasks are reported in Figure \ref{fig:SSC results}. As expected, SSC-ADMM and EKSS achieve the highest clustering accuracy in this context as they are specifically designed for subspace clustering. Among the manifold clustering methods, LAPD maintains the most reliable performance across various noise levels and intersection angles. In particular, LocPCA achieves higher clustering accuracy than LAPD for $d=1$ but performs poorly for $d=2$. LocPCA's deteriorating performance is observed for larger $d$'s as well, even with a large intersection angle of $\Theta = \frac{\pi}{2}$. PBC, another path based clustering method, is much less robust to noise and intersection angles. LAPD-D (all default parameters) performs quite well when the noise is moderate and intersection angle not too small; however manual tuning of the edge length parameter $e$ allows for LAPD-T to tolerate higher noise levels. For comparison, we remark that we manually tuned many parameters for all competing methods, and generally performance was sensitive to these parameters. Overall, these results support our theoretical result in \eqref{eq:LAPD gap} as we empirically observe that increasing $\tau$ and decreasing $\theta$ results in a more limited range of parameter values where the method can succeed, since the edge length $e$ must still remain a local construction.

We now examine how runtime (in seconds) varies with the sample size $n$ and the ambient dimension $D$.
These experiments use a dataset of two intersecting unit hypercubes with $d=2$. The intersection angle $\Theta$ is set to $\frac{\pi}{4}$ and $\sigma$ fixed at $0.05$ (noise level $\tau$ fixed at \(0.05\sqrt{3}\)). In theory, the computational complexities of the benchmark methods are as follows: \textsc{SSC-ADMM} is \(O(n^3)\), \textsc{EKSS} is \(O(n^2)\), and \textsc{PBC} is \(O(n \log n)\), assuming the number of source points is \(O(1)\). We did not find a formal complexity analysis for \textsc{LocPCA} in~\cite{arias-castro_spectral_2017}, but in our experiments, it proved highly scalable.

To investigate runtime dependence on \(n\), we fix \(D = 100\) and take a grid of \(n\) values ranging from 1000 to 35000. 
As shown in Figure~\ref{subfig:runtime n}, \textsc{LAPD} demonstrates near-linear scalability with respect to \(n\). Note we denote by \(\text{LAPD(NE)}\) our proposed method \emph{not estimating} the intrinsic dimension \(d\) and noise \(\tau\), while \(\text{LAPD(E)}\) includes those estimation steps. 
\textsc{SSC} and \textsc{EKSS} appear to grow at least quadratically, whereas \textsc{LocPCA} and \textsc{PBC} display performance comparable to \textsc{LAPD}. 

To examine how runtime depends on \(D\), we fix the dataset size \(n = 6{,}000\) and take a grid of \(D\) values ranging from 25 to 2500.
Figure~\ref{subfig:runtime d} plots the measured runtimes against \(D\). 
Overall, we observe a linear relationship between the LAPD runtime and \(D\). 
By contrast, \textsc{EKSS}, \textsc{LocPCA}, and \textsc{PBC} exhibit rapid growth as \(D\) increases, indicating they suffer from the ``curse of dimensionality.'' 
Among the compared methods, only \textsc{SSC} outperforms LAPD; however, its apparent independence from \(D\) arises because the \(O(n^3)\) term dominates the dependence on the ambient dimension, which is \(O(Dn^2)\). Taken together, these findings show that \textsc{LAPD} is highly scalable for large datasets in high-dimensional settings.

\begin{figure}[h]
    \centering
    \begin{subfigure}[b]{0.45\textwidth}
         \centering
        \includegraphics[width=\textwidth]{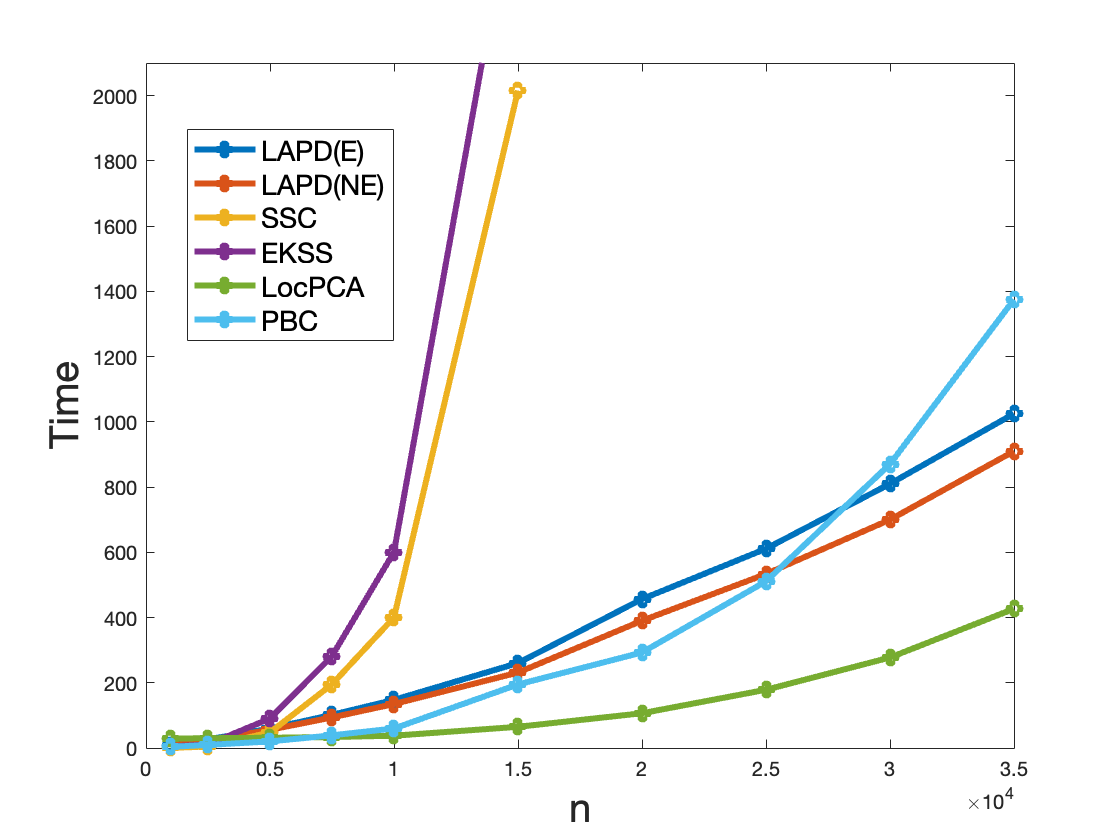}
         \caption{Run time w.r.t. $n$}
          \label{subfig:runtime n}
     \end{subfigure}
     \hfill
     \begin{subfigure}[b]{0.45\textwidth}
         \centering
         \includegraphics[width=\textwidth]{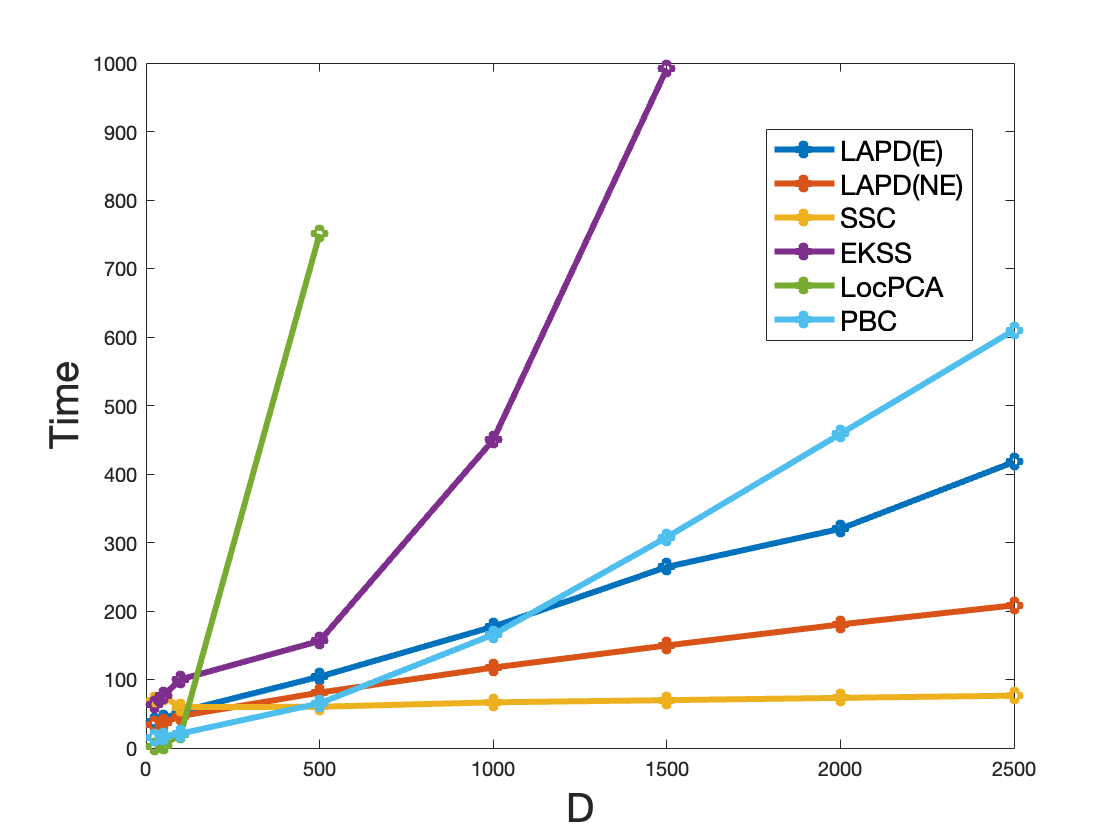}
         \caption{Run time w.r.t. $D$}
         \label{subfig:runtime d}
     \end{subfigure}
     \caption{Runtime (in seconds) of various clustering methods. (a) Runtime against the datasize $n$; (b) Runtime against the ambient dimension $D$. LAPD(NE) records the runtime of LAPD only, whereas LAPD(E) includes the time for estimating the intrinsic dimension $d$ and noise level $\tau$ using the algorithm from \cite{LITTLE2017504}.}
        \label{fig:Runtime}
\end{figure}

\subsection{Manifold Clustering}
\label{sec:manifold_clustering}

We next evaluate LAPD on manifold clustering using 11 synthetic datasets whose underlying manifolds may be nonlinear (see Figure~\ref{fig:nonlinear datasets}). 
The first four datasets have intrinsic dimension \(d = 1\), and the next four have \(d = 2\). 
We refer to the \(d = 1\) datasets as \emph{Dollar Sign} (DS), \emph{Olympic Rings} (OR), \emph{Three Curves} (TC), and \emph{Rose Circle} (RC), and the \(d = 2\) datasets as \emph{Two 2-Spheres} (2S), \emph{Two Triangles} (TT), \emph{Three Planes} (TP), and \emph{Swiss Roll} (SR).
To explore performance in even higher dimensions, we also run experiments on unit \(d\)-spheres (\(d \in \{3,4,5\}\)), referred to as 3S, 4S, and 5S for \emph{two 3-spheres}, \emph{two 4-spheres}, and \emph{two 5-spheres}, respectively.
We fix the ambient dimension at \(D = 100\), the dataset size at \(n = 6000\), and the overall noise level at \(\tau = 0.05\sqrt{3}\). Note that the noise is evenly distributed across the $D$ ambient dimensions and so only a small portion of the noise is visible in Figure \ref{fig:nonlinear datasets} 
We compare \textsc{LAPD} against the manifold clustering methods \textsc{LocPCA}, \textsc{PBC}, and \textsc{DCV} from the previous section, with clustering results shown in Table~\ref{table:manifold clustering result (1)}.  
For \textsc{LAPD}, we report clustering accuracy both when the ground truth number of clusters \(m\) is known (denoted as \(\textsc{LAPD}(m)\)) and when \(m\) is estimated by \textsc{LAPD} (denoted as \(\textsc{LAPD}(\hat{m})\)). 
The runtime remains the same in both cases.

Table~\ref{table:manifold clustering result (1)} shows that \textsc{LAPD} achieves the highest clustering accuracy for all datasets under default parameters and oracle \(m\). The only dataset where it performs poorly is TP, and in fact manual parameter tuning of $e$ can increase the accuracy of LAPD($m$) to .938 on the TP dataset. 
Competing methods may succeed on simpler or more regular manifolds---such as those with lower curvature (DS, OR, TT), larger intersection angles (DS), or smoothly varying shapes (2S, TT, 3S, 4S, 5S)---but they struggle with more challenging cases. 
For instance, TP presents vastly different intersection angles \(\Theta_1 = \frac{\pi}{2}\) and \(\Theta_2 = \frac{\pi}{5}\), SR has continuously changing curvature, and RC features multiple clusters lying close to each other. 
Only \textsc{LAPD} reliably handles such complexities. \textsc{LocPCA}, in particular, experiences the same difficulty here as in subspace clustering: its performance drops when \(d = 2\) compared to \(d = 1\). 
Although \textsc{LocPCA} handles higher-dimensional spheres (3S, 4S, 5S) reasonably well, its behavior on more complex shapes warrants further study. 
\textsc{PBC} achieves an accuracy level similar to \textsc{LocPCA} but often exhibits larger standard deviations, suggesting less consistent cluster assignments. 
While the runtime of these other methods may increase more slowly with the intrinsic dimension \(d\), extensive parameter tuning is required for each, offsetting any practical advantage over \textsc{LAPD} for higher dimensions.

\begin{figure}
     \centering
     \begin{subfigure}[b]{0.2\textwidth}
         \centering         \includegraphics[width=\textwidth]{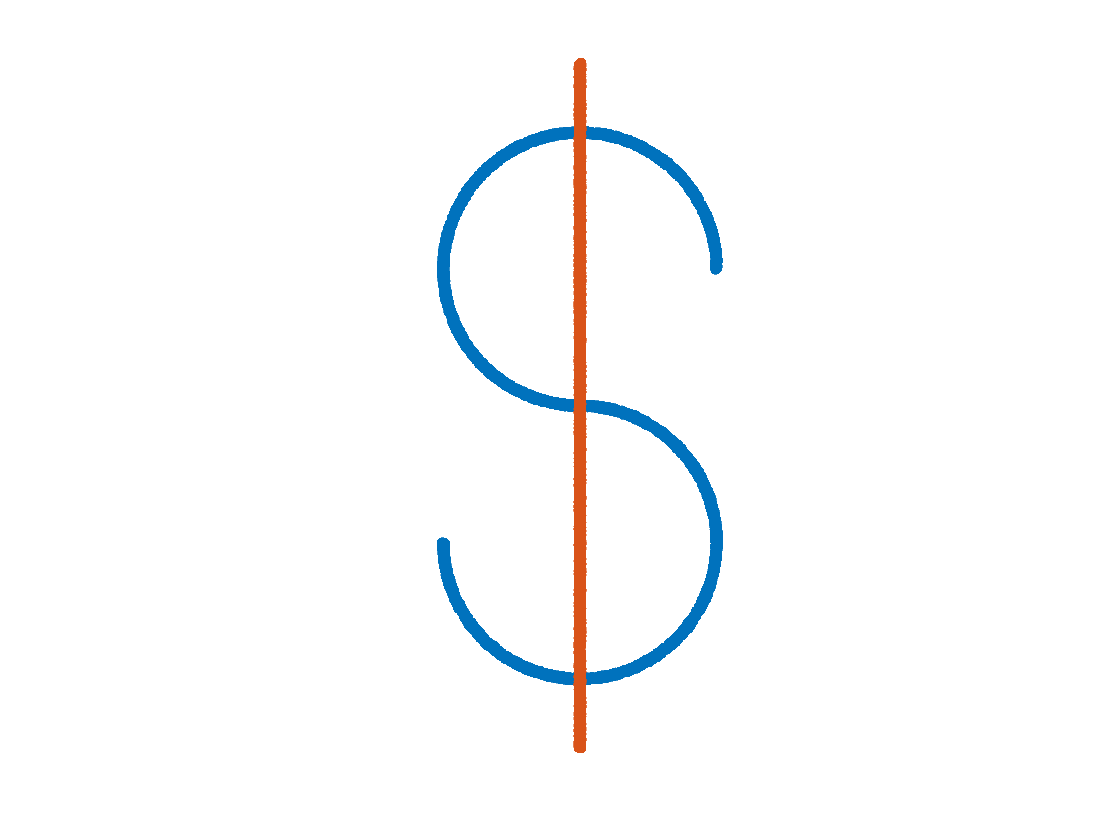}
         \caption{Dollar Sign}
         \label{fig:dollar sign}
     \end{subfigure}
     \hfill
     \begin{subfigure}[b]{0.2\textwidth}
         \centering
         \includegraphics[width=\textwidth]{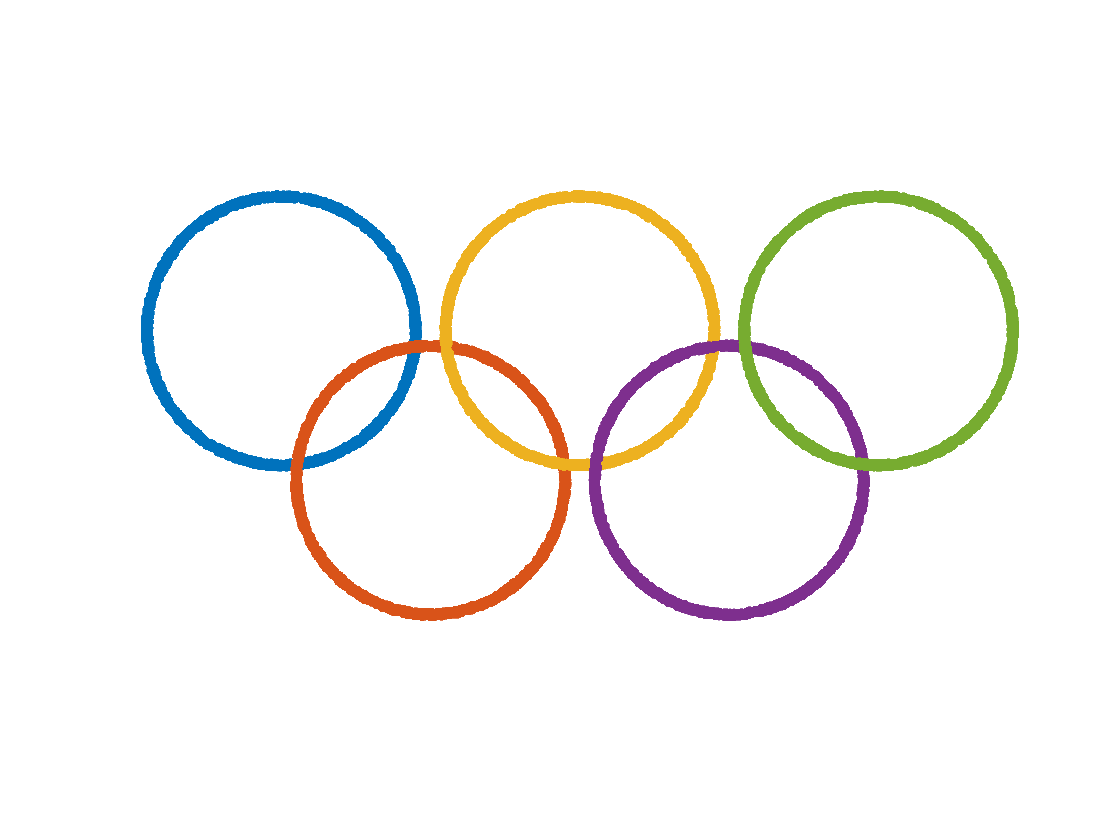}
         \caption{Olympic Rings}
         \label{fig:olympic rings}
     \end{subfigure}
     \hfill
      \begin{subfigure}[b]{0.2\textwidth}
         \centering
         \includegraphics[width=\textwidth]{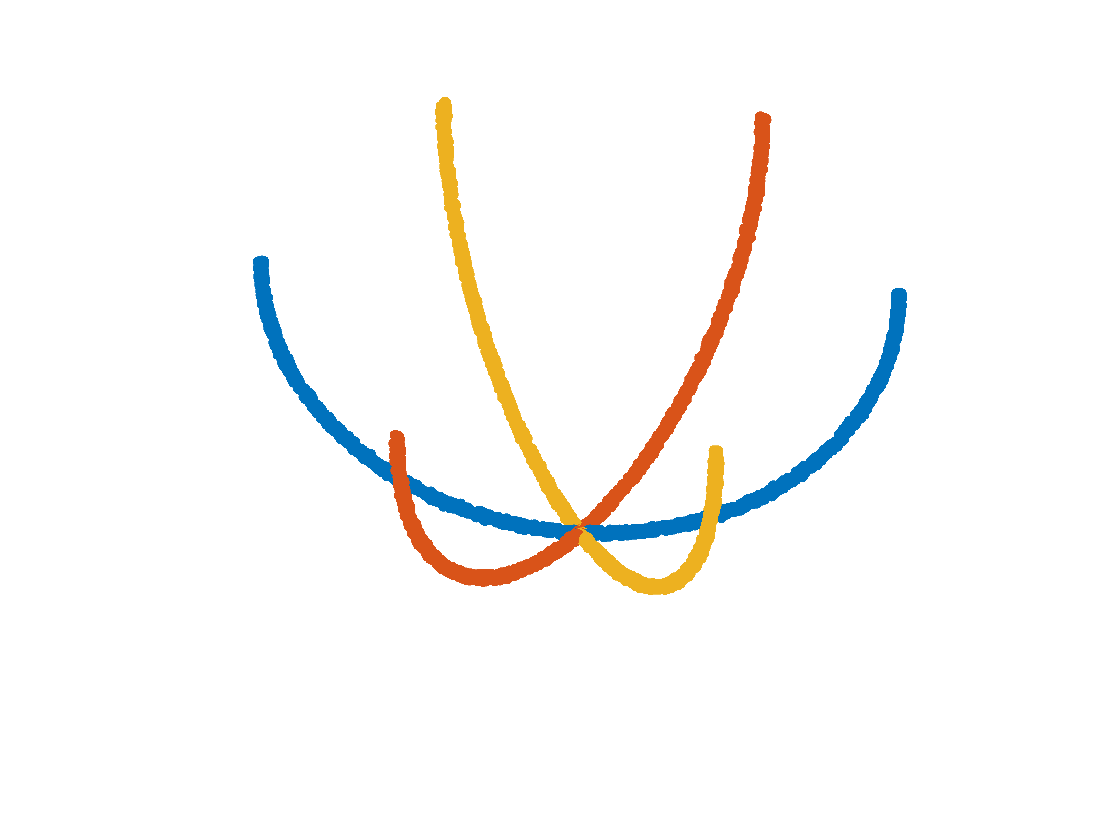}
         \caption{Three Curves}
         \label{fig:three curves}
     \end{subfigure}
     \hfill
     \begin{subfigure}[b]{0.2\textwidth}
         \centering
         \includegraphics[width=\textwidth]{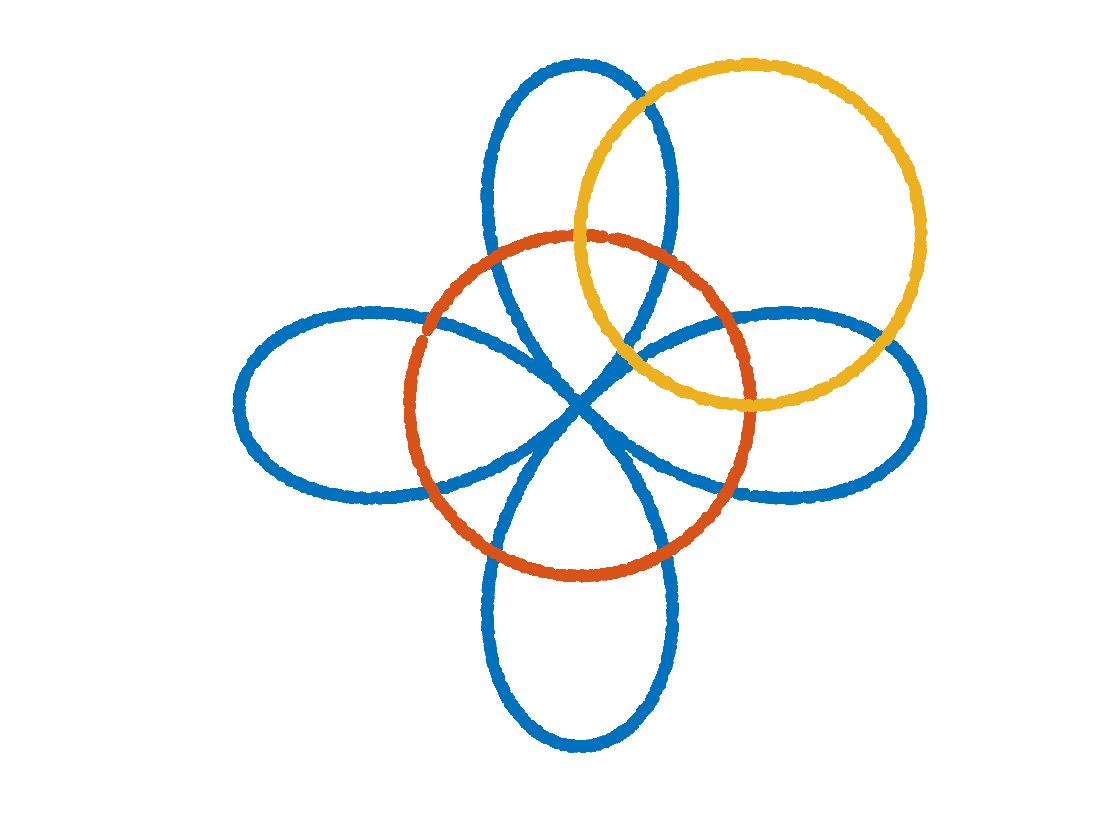}
         \caption{Rose Circles}
         \label{fig:Rose & Circle}
     \end{subfigure}
     \hfill
     \begin{subfigure}[b]{0.2\textwidth}
         \centering
         \includegraphics[width=\textwidth]{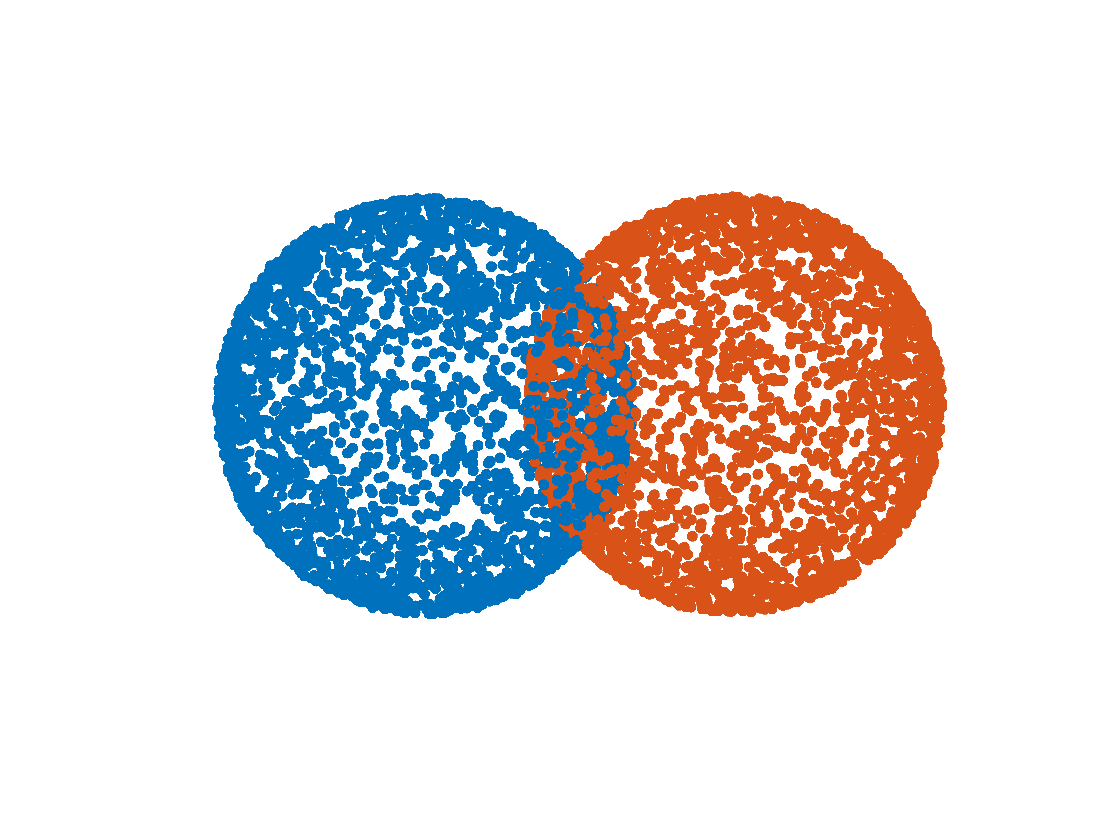}
         \caption{Two Spheres}
         \label{fig:two spheres}
     \end{subfigure}
     \hfill
     \begin{subfigure}[b]{0.2\textwidth}
         \centering
         \includegraphics[width=\textwidth]{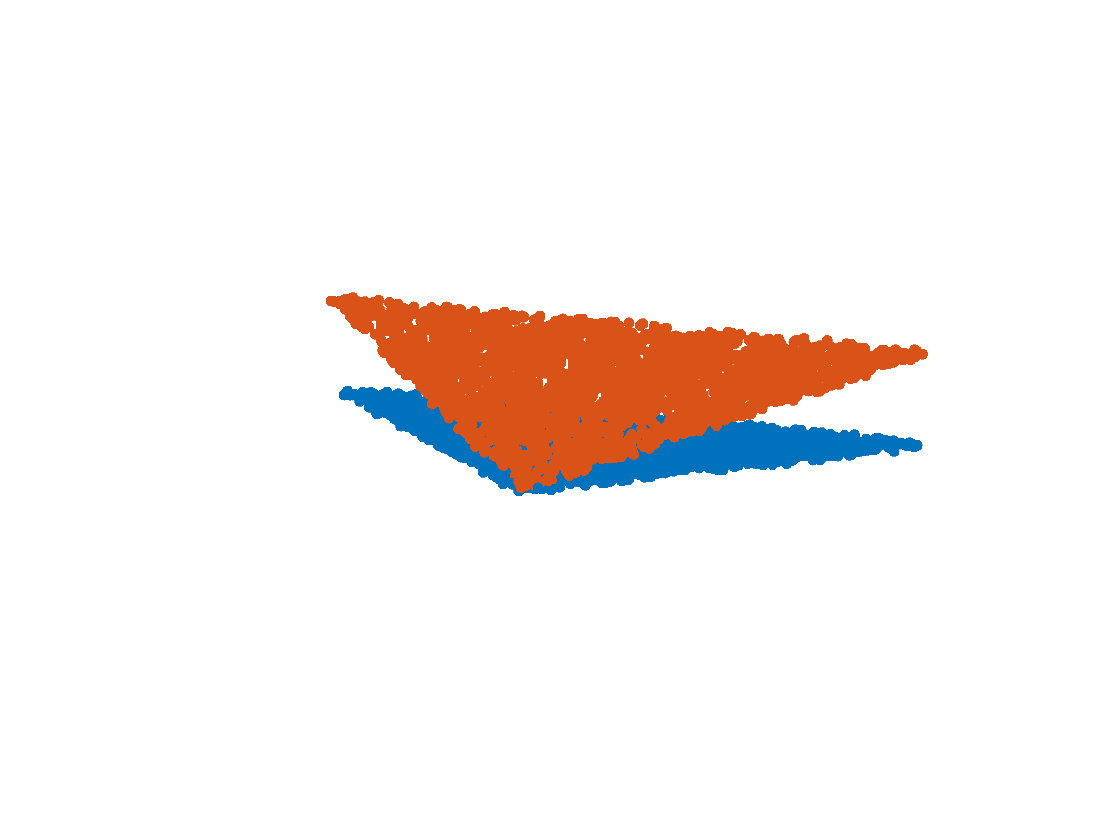}
         \caption{Two Triangles}
         \label{fig:two triangles}
     \end{subfigure}
     \hfill
     \begin{subfigure}[b]{0.2\textwidth}
         \centering
         \includegraphics[width=\textwidth]{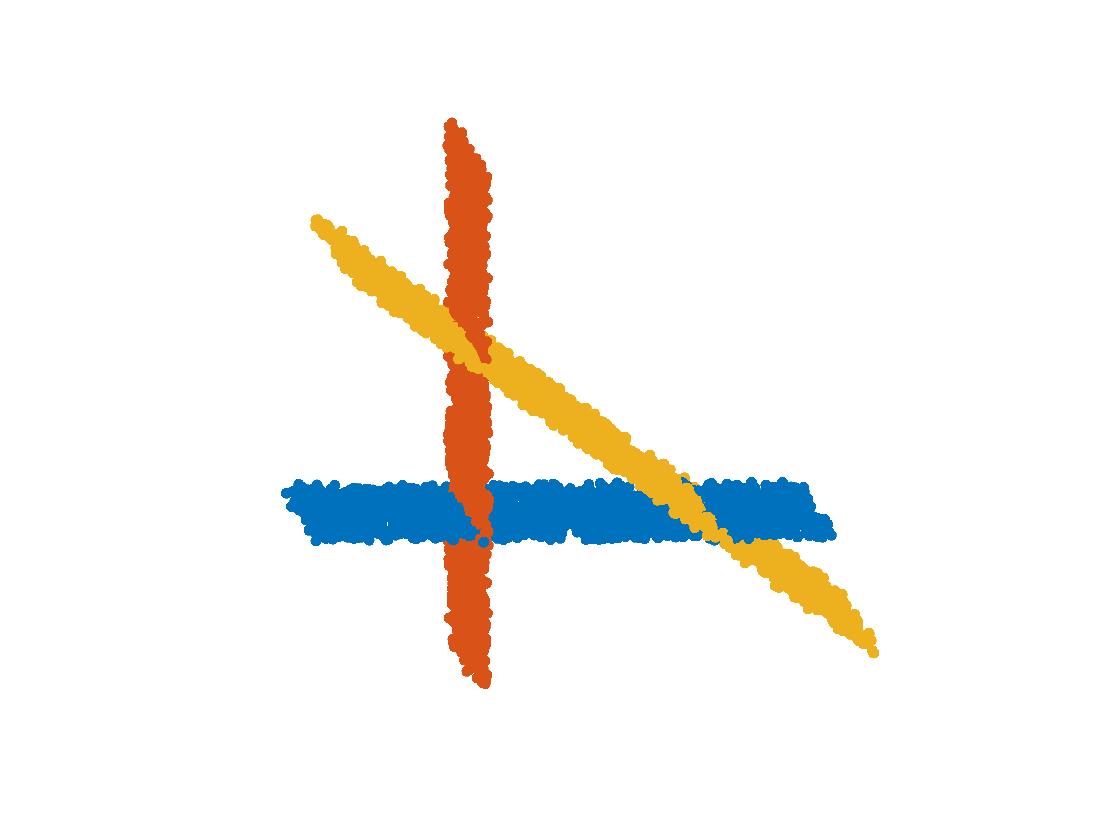}
         \caption{Three Planes}
         \label{fig:three planes}
     \end{subfigure}
     \hfill
     \begin{subfigure}[b]{0.2\textwidth}
         \centering
         \includegraphics[width=\textwidth]{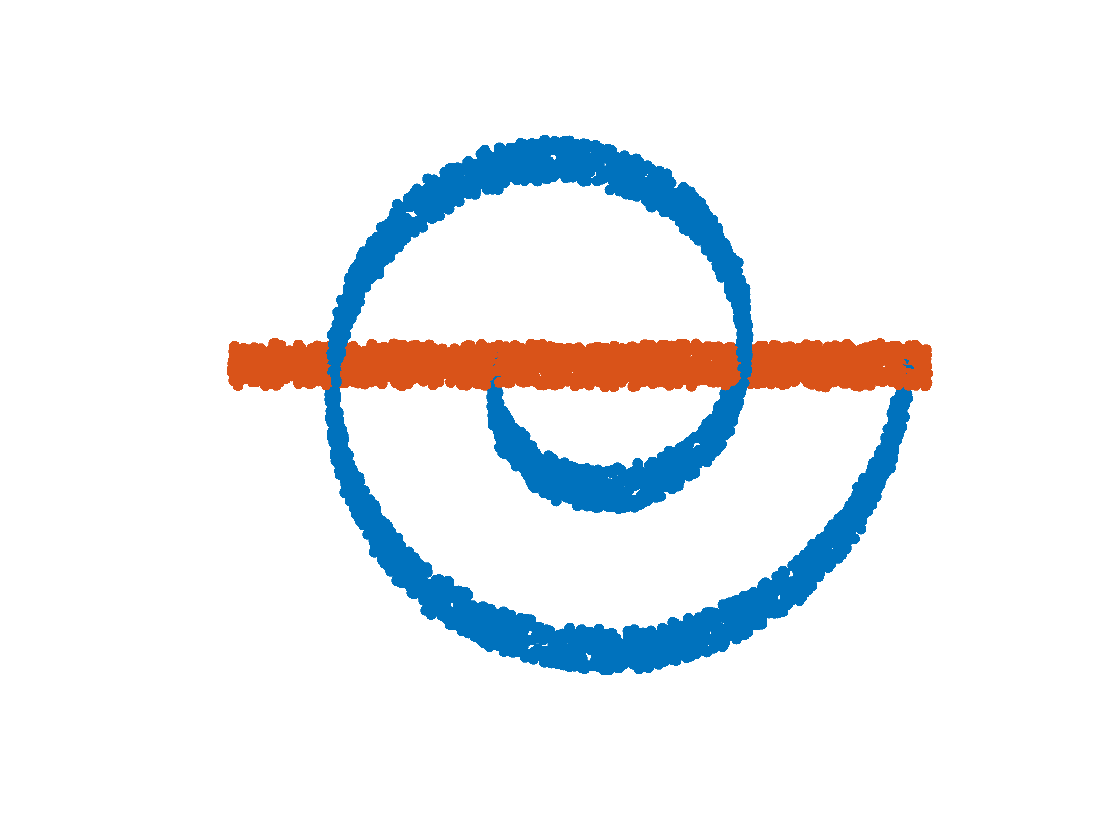}
         \caption{Swiss Roll}
         \label{fig:swiss roll}
     \end{subfigure}     
        \caption{Synthetic manifold learning datasets. Objects on the first row have intrinsic dimension $d=1$; objects on the second row have intrinsic dimension $d=2$. All the objects have ambient dimension $D=100$. For plotting purpose, the figures display the datasets' projection on the $d+1$ ($d+2$ for TC) dimensional space.}
        \label{fig:nonlinear datasets}
\end{figure}

\begin{table}[h]
 \begin{center}
\resizebox{\textwidth}{!}{
\begin{tabular}{|l||*{14}{c|} }
    \hline
\backslashbox[25mm]{Objects}{Methods} 
            & \multicolumn{3}{c|}{LAPD ($m$)}
            & \multicolumn{2}{c|}{LAPD ($\hat{m}$)}
            & \multicolumn{3}{c|}{LocPCA}
            & \multicolumn{3}{c|}{PBC}
            & \multicolumn{3}{c|}{DCV}\\
    \hline\hline
DS & .992 & .002 & 12.9 & .992 & .002 & .984 & .003 & 11.4 & .956 & .007 & 10.5 & .510 & .006 & 175.6 \\ \hline
OR & .982 & .002 & 11.4 & .982 & .002 & .881 & .043 & 54.6 & .838 & .121 & 7.0 & .597 & .007 & 213.7 \\ \hline
TC & .989 & .001 & 12.8 & .989 & .001 & .984 & .003 & 9.6  & .694 & .115 & 26.8 & .475 & .006 & 295.3\\ \hline
RC & .976 & .003 & 12.9 & .917 & .092 & .533 & .039 & 50.1 & .740 & .017 & 1.5 & .482 & .005 & 233.9 \\  \hline 
2S & .982 & .003 & 16.4 & .982 & .003 & .965 & .003 & 85.0 & .920 & .014 & 13.4 & .851 & .006 & 227.5 \\ \hline
TT & .988 & .003 & 35.4 & .988 & .003 & .547 & .050 & 38.0 & .980 & .009 & 6.9 & .516 & .004 & 234.6 \\ \hline
TP & .733 & .174 & 41.3 & .675 & .078 & .542 & .032 & 88.6 & .640 & .056 & 16.3 & .413 & .003 & 227.1 \\ \hline
SR & .975 & .002 & 26.7 & .975 & .002 & .700 & .064 & 39.4 & .803 & .114 & 15.8 & .608 & .005 & 234.8 \\ \hline
3S & .976 & .005 & 42.5 & .976 & .005  & .946 & .004 & 121.2 & .937 & .010 & 23.1 & .929 & .004 & 220.3\\ \hline
4S & .971 & .003 & 180.2 & .971 & .003 & .954 & .002 & 194.2 & .925 & .006 & 35.2 & .960 & .003 & 193.2\\ \hline
5S & .960 & .005 & 3352.9 & .960 & .005 & .963 & .002 & 204.5 & .941 & .006 & 34.0 & .973 & .002 & 247.7\\ \hline
\hline
\end{tabular}}
\end{center}
 \caption{Manifolds clustering results. The three sub-columns under each method record the average clustering accuracy, standard deviation, and runtime in seconds. LAPD($m$) denotes clustering with oracle cluster numbers $m$; LAPD($\hat{m}$) denotes when $m$ is estimated by LAPD. The two share about the same runtime. }
 \label{table:manifold clustering result (1)}
\end{table}

\subsection{Real-world Data}
\label{sec:real world}

\begin{table}[h]
  \begin{center}
    \resizebox{\textwidth}{!}{
    \begin{tabular}{|l||*{13}{c|} }
        \hline
        \backslashbox[15mm]{}{} 
                & \multicolumn{1}{c|}{LAPD$^1$}
                & \multicolumn{1}{c|}{LAPD$^2$}
                & \multicolumn{1}{c|}{time}
                & \multicolumn{1}{c|}{SSC}
                & \multicolumn{1}{c|}{time}
                & \multicolumn{1}{c|}{EKSS\footnote{We use the same data and parameters suggested in the original paper but could not reproduce the accuracy reported in the paper.}}
                & \multicolumn{1}{c|}{time}
                & \multicolumn{1}{c|}{LocPCA}
                & \multicolumn{1}{c|}{time}
                & \multicolumn{1}{c|}{PBC}
                & \multicolumn{1}{c|}{time}
                & \multicolumn{1}{c|}{DCV}
                & \multicolumn{1}{c|}{time}\\
        \hline\hline
    COIL [0:6,8:17,19] & .997 & .982 & 4.1 & .876 & 24.6 & .799 & 95.3 & .473 & 165.7 & .386 & 112.7 & .957 & 101.4  \\ \hline
    COIL [0:17,19] & .913 & .931 & 4.1 & .863 & 28.4 & .833 & 110.7 & .464 & 199.5 & .405 & 126.8 & .963 & 96.8\\ \hline
    COIL [0:19] & .857 & .904 & 4.8 & .836 & 31.2 & .806 & 132.8 & .360 & 180.9 & .437 & 134.2 & .900 & 104.6 \\  \hline\hline
    MNIST-T [0:7] & .952 & .961 & 84.9 & .737 & 403.8 & .869 & 4029.7 & .715 & 101.3 & .452 & 1658.0 & .967 & 400.6 \\ \hline
    MNIST-T [0:8] & .757 & .958 & 97.5 & .797 & 650.5 &.791 & 8752.5 & .557 & 387.6 & .414 & 3422.5 & .952 & 451.2\\ \hline
    MNIST-T [0:9] & .688 & .795 & 112.1 & .472 & 1081.3 & .685 & 13,675.7 & .489 & 604.0 & .358 & 5521.8 & .855 & 544.4 \\ \hline\hline
    MNIST-F [0:6] & .948 & .975 & 836.2 & NA & NA & NA & NA & .618 & 1320.4 & NA & NA & NA & NA \\ \hline 
    MNIST-F [0:8] & .751 & .957 & 1204.9 & NA & NA & NA & NA & .562 & 1803.6 & NA & NA & NA & NA \\ \hline
    MNIST-F [0:9] & .718 & .859 & 1338.4 & NA & NA & NA & NA & .538 & 2095.2 & NA & NA & NA & NA \\ \hline\hline
    USPS [0:6] & .944 & .984 & 40.7 & .665 & 356.5 & .823 & 347.8 & .653 & 1.41 & .440 & 64.6 & .980 & 214.5 \\ \hline
    USPS [0:6,8:9] & .724 & .942 & 77.1 & .679 & 1037.8 & .782 & 573.4 & .635 & 55.5 & .424 & 13.3 & .967 & 300.7 \\ \hline
    USPS [0:9] & .748 & .935 & 87.5 & .774 & 1400.4 & .778 & 1218.2 & .583 & 63.7 & .378 & 70.0 & .950 & 340.3 \\ \hline\hline
    \end{tabular}
    }
  \end{center}
  \caption{Clustering accuracy and runtime on real-world datasets. Four sections are each for COIL20, MNIST-test, MNIST-full, and USPS in order. NA indicates the method either encounters an out-of-memory issue or consumes a very long runtime on that particular dataset. For EKSS, we did not incorporate the pre-processing steps so our results are different from those reported in their original paper \citep{Lipor_2020}. }
  \label{table:real-world result}
\end{table}

Finally, we evaluate LAPD on four widely used real-world datasets: COIL20 \citep{nene1996columbia}, USPS \citep{uspsdataset}, MNIST-test \citep{deng2012mnist}, and MNIST-full \citep{deng2012mnist}. For COIL20, we set $d=1$ as the intrinsic dimension is clear from the data construction. For the remaining data sets we use $d=2$, which generally yielded good performance. These datasets are popular in subspace clustering \citep{pengDSC2020}, but their underlying manifolds generally have low curvature, and most classes are well-separated \citep{wu2022deep, wu2022generalized, McConville2021}. Thus, these benchmarks do not fully expose the advantages of our geometry-based approach, which is designed to handle more pronounced manifold intersections. 

Despite the overall separability of most classes, certain subsets do have strong similarities. For instance, COIL20 classes 2, 5, and 18 (cars) are nearly identical, and some MNIST digit classes overlap (e.g., 4, 7, and 9). To give a more nuanced view, for each real-world dataset we analyze 3 subsets: the first uses classes that are relatively distinct, the second introduces more difficult classes, and the third includes all classes. Detailed results appear in Table \ref{table:real-world result}.

A few observations stand out. First of all, LAPD\textsuperscript{2} typically outperforms LAPD\textsuperscript{1}; on these datasets, utilizing the two-sided weight \eqref{eq:two way weight} enhances within-class connections without substantially increasing inter-class connections, but we emphasize this was not the case for the synthetic data sets analyzed in Sections \ref{sec:subspace clustering} and \ref{sec:manifold_clustering}. Secondly, LocPCA and PBC generally yield lower accuracy than LAPD and can be sensitive to parameter tuning. Interestingly, DCV often matches LAPD’s high accuracy on these well-separated data sets, but the method fails on the majority of the synthetic data in Section \ref{sec:manifold_clustering} where manifolds intersect more substantially. Moreover, DCV’s deep-learning framework has many parameters, which can be challenging to optimize in an unsupervised setting. 
Finally, we note EKSS requires a large ensemble size (up to 1{,}000), additional preprocessing (e.g., whitening for COIL20), and carefully chosen parameters (e.g., intrinsic dimensions), and these factors can make its workflow more complex.
Overall, LAPD provides consistently strong performance on these benchmarks with relatively few parameters, reinforcing its utility as a practical method for manifold clustering.

\section{Conclusion}
\label{sec:Conclusion}
We propose and investigate a new algorithm for MMC, using angles between $d$-simplices as an ingredient for computing with the LAPD metric. Highlights of the algorithm include: (i) it works extremely well for clustering intersecting manifolds; (ii) it works with nonlinear manifolds and is robust to noise; (iii) it can generally give an estimate of the number of manifolds; and (iv) it is scalable with quasi-linear computational complexity. Our results demonstrate that LAPD is competitive with or superior to alternative MMC methods on both synthetic and real-world datasets. 

In future work we hope to extend the method to collections of manifolds with varying intrinsic dimension; such an extension would be non-trivial but potentially possible through an iterative procedure. 
In addition, in this article we have explored a completely angle-based path distance (LAPD) instead of a density-based path distance (LLPD). In practice, it could be helpful to define a metric combining both of these components, since as observed in the real data experiments class intersections typically consist of low-density regions.


\acks{AL and HC were partially supported by NSF DMS 2309570 and NSF DMS CAREER 2441153. AN and HC were partially supported by NSF DMS-1848508, AFOSR FA9550-20-1-0338, and AFOSR FA9550-23-1-0749. All authors were partially supported by NSF RTG DMS 2136198.}


\newpage

\appendix

\section{Proof of Theorem \ref{thm:within LAPD upper bound}}
\label{pf:within LAPD upper bound}

\first*

Recall our setting that the noise polluted data $X=\{x_i\}_{i \in [n]}$ is sampled from a tubular region $T_\tau(\M) = \cup_{j \in [m]} T_{\tau}(\M_j)$ around the set of underlying manifolds. Note that when $\tau=0$, the result is trivial since wLAPD $=0$; we thus assume $\tau>0$. We first bound the wLAPD for a fixed manifold component $\M_1$:
\begin{align*}
    \text{wLAPD}(\M_1) &= \max_{\Delta_+,\Delta_- \in S(\M_1)} \text{LAPD}(\Delta_+, \Delta_-) \, ,
\end{align*}
with $S(\M_1)$ 
denoting the set of \textit{valid, pure} simplices of $\M_1$ constructed from noisy samples from $T_\tau(\M_1)$. Refer to Section \ref{sec:methodology} for the definition of valid and pure simplices. For the rest of the section, since we are focusing on the fixed manifold $\M_1$, we are only dealing with the pure simplices of $\M_1$ and we will not specify this fact as long as the context is clear. We will also write the dihedral angle $\theta_{ij}$ for $\theta(\Delta_i, \Delta_j)$ or $\theta_{\pm}$ for $\theta(\Delta_-,\Delta_+)$. 

We derive a (nonrandom) upper bound by replacing the max over the discrete set $S\left(\M_1\right)$ with a max over $V(\M_1)$, the set of all $(e,q)$-simplices (equation \ref{eq:simplex quality constraint}) which can be constructed from points in a $D$-dimensional tube of radius $\tau$ about $\M_1$. Specifically
\small{\begin{align*}
    V(\M_1) &=\{ \Delta=\{x_1,\ldots,x_{d+1}\} : \|x_i - P_{\M_1}x_i\| \leq \tau \, , \, e \leq \| x_i - x_j \| \leq \frac{e}{q}\, \forall\, i,j\in[d+1],\, i<j\},   
\end{align*}}where $P_{\M_1}$ is projection onto $\M_1$; note $S(\M_1)\subseteq V(\M_1)$ because $V(\M_1)$ contains \textit{all} valid simplices, while $S(\M_1)$ contains just those arising from our discrete samples. 
We prove that Theorem \ref{thm:within LAPD upper bound} holds for both weight constructions considered in our paper, i.e. for $W^1_S(\Delta_i,\Delta_j)=\pi - \theta_{ij}$ when $\theta_{ij} \geq \frac{\pi}{2}$ (else not connected; see \eqref{eq:simplex graph weight}) and for $W^2_S(\Delta_i,\Delta_j)=\min\{\pi -\theta_{ij}, \theta_{ij}\} $ (see \eqref{eq:two way weight}). 
Note that since $\mathcal{G}_S$ is connected, all distances are finite, and: 
\begin{align*}
    \text{wLAPD}(\M_1, W_S^1) &\leq \max_{\Delta_i,\Delta_j \in V(\M_1),\, \theta_{ij} \geq \frac{\pi}{2} } \, \pi - \theta_{ij}  := (A) \\
    \text{wLAPD}(\M_1, W_S^2) &\leq \max_{\Delta_i,\Delta_j \in V(\M_1)} \min\{\pi -\theta_{ij}, \theta_{ij}\} := (B)
\end{align*}
Since $(A) \leq (B)$, $\text{wLAPD}(\M_1) \leq (B)$ under both weight constructions, and it suffices to bound (B). Letting $V(\M_1,D)$ denote the dependence on the ambient dimension of the simplex set construction, we first establish that (B) is independent of the ambient dimension $D$. Specifically, since $\Delta_i,\Delta_j$ must always live in a $(d+1)$-dimensional space by construction, we have:
\small{\begin{align*}
    \max_{\Delta_i,\Delta_j \in V(\M_1,D)} \min\{\pi -\theta_{ij}, \theta_{ij}\} 
    &= \max_{\substack{\text{subspaces } S,\\ \text{dim}(S)=d+1}} \max_{\Delta_i,\Delta_j \in V(\M_1,D)\cap S} \min\{\pi -\theta_{ij}, \theta_{ij}\} \\
    &= \max_{\Delta_i,\Delta_j \in V(\M_1,d+1)} \min\{\pi -\theta_{ij}, \theta_{ij}\} \, .
\end{align*}}
This is verified with numerical results from optimization in Figure \ref{fig:wLAPD bound vs D}. Thus moving forward without loss of generality we can consider $D=d+1$. 

The proof consists of the following steps:
\begin{enumerate}
    \item Assuming $q=1$, i.e. all simplices being regular, the optimal configuration of a pair of simplices which generates $(B)$ occurs when the shared face lies on one side of the boundary of $T_{\tau}(\M_1)$ and the two apex vertices lie on the opposite side of the boundary; see Subsubsection \ref{ss: largest configuration}. 
    \item LAPD upper bound corresponding to this configuration is $(B)=\sqrt{\frac{32d}{d+1}}\frac{\tau}{e} \sim O(\frac{\tau}{e})$; see Subsubsection \ref{ss: largest wLAPD}. 
    \item When $q\ne 1$, the optimal configuration still occurs when this configuration happens, and thus $(B)$ will depend on $q$ in a continuous fashion; we provide numerical evidence supporting this claim in Subsubsection \ref{ss: nonregular simplices}. Analytically computing the exact dependence on $q$ is difficult and not particularly insightful, but clearly we still have $(B) =O\left( \frac{\tau}{e} \right)$. 
\end{enumerate}
Finally, since an identical argument shows $\text{wLAPD}(\M_2)=O\left( \frac{\tau}{e} \right)$, we can conclude $\text{wLAPD}=O\left( \frac{\tau}{e} \right)$, which proves Theorem \ref{thm:within LAPD upper bound}.

\subsubsection{Optimal simplex configuration}
\label{ss: largest configuration}

We denote an arbitrary pair of adjacent $d$-simplices in a $(d+1)$-dimensional ambient space as $\Delta_-=[x_1, \dots, x_{d}, x_-]$ and $\Delta_+=[x_1, \dots, x_{d}, x_+]$ with $x_i, x_\pm \in \mathbb{R}^{d+1}$ for $i \in [d]$, i.e. the first $d$ vertices $[x_1, \dots, x_d]$ denote the shared face and $x_\pm$ denote the apex vertices from the pair of adjacent simplices $\Delta_\pm$. We choose a coordinate system so that $\M_1 \subseteq \text{span}\{e_1, \ldots, e_d\}$, i.e. the manifold lives in the first $d$ dimensions and the noise in dimension $d+1$, so that $|x_{i}^{(d+1)}| \leq \tau$ and similarly for $x_{\pm}$. The following proposition establishes the optimal configuration generating $(B)$ occurs when the shared face lies on the boundary of the tubular noise region and the two non-shared apex vertices lie on the other side of the boundary.

\begin{proposition}
\label{prop:largest within LAPD orientation}
    Assume $q=1$ and let $\Delta_+, \Delta_-\in V(\M_1,d+1)$ be two adjacent simplices maximizing $B(\Delta_+,\Delta_- )=\max_{\Delta_+, \Delta_-\in V(\M_1, d+1)}\min\{\pi-\theta_{\pm},\theta_{\pm}\}$. Then either
    \begin{equation*}
        x_i^{(d+1)} = -\tau,\, 1 \leq i \leq d \;\; \text{and}\;\; x_\pm^{(d+1)} =\tau
    \end{equation*}
    or 
    \begin{equation*}
        x_i^{(d+1)} = \tau,\, 1 \leq i \leq d \;\; \text{and} \;\; x_\pm^{(d+1)} = -\tau. 
    \end{equation*}
\end{proposition}

\begin{proof}
We first present a proof for $d=1,2$. For larger $d$'s, a rigorous proof becomes challenging and we will present an empirical result from optimization to support our conclusion.

Consider the optimization problem: 
\begin{align}
\label{eq:LAPD within optimization (2)}
     \max_{\Delta_+,\Delta_-\in V(\M_1,d+1)} \min\{\pi-\theta_{\pm},\theta_{\pm}\}. 
\end{align}
For $d=1$, we make the shared node $x_1$ of the pair of adjacent vectors $\overrightarrow{x_1 x_-}$ and $\overrightarrow{x_1 x_+}$ the origin of $e_1$, i.e. $x_1 = (0, x_1^{(2)})^\intercal$. Using the polar coordinate system gives a simplex expression for $x_-, x_+$ that satisfies the edge length constraint, $x_- = (e\cos{\theta_-}, x_1^{(2)}+e\sin{\theta_-})^\intercal$ and $x_+ = (e\cos{\theta_+}, x_1^{(2)}+e\sin{\theta_+})^\intercal$ where $\theta_-$ and $\theta_+$ are the angle made by the vector $\overrightarrow{v}_- \coloneqq \overrightarrow{x_1 x_-} = (e\cos{\theta_-}, e\sin{\theta_-})^\intercal$ and $\overrightarrow{v}_+ \coloneqq \overrightarrow{x_1 x_+} = (e\cos{\theta_+}, e\sin{\theta_+})^\intercal$ from $e_1$. W.l.o.g, we let $\theta_- \in [-\frac{\pi}{2}, \frac{\pi}{2}]$ and $|\theta_+ - \theta_-| = \theta$ for $\theta \in [0,\pi]$. The case for $\theta_- \in [\frac{\pi}{2}, \frac{3\pi}{2}]$ can be proved with the same argument and thus will lead to the same conclusion. It follows that the optimization problem \eqref{eq:LAPD within optimization (2)} for $d=1$, along with the appropriate constraints, can be reformulated as: 
\begin{align*}
\max_{\theta_-, \theta_+, x_1^{(2)}}  \min &\{\pi- \arccos\frac{v_- \cdot v_+}{\|v_-\| \|v_+\|}, \arccos\frac{v_- \cdot v_+}{\|v_-\| \|v_+\|}\} \\ 
\text{s.t.} \hspace{5pt} |x_1^{(2)}| \leq \tau &, \hspace{5pt} |x_1^{(2)} +e\sin{\theta_-}| \leq \tau, \hspace{5pt} |x_1^{(2)}+e\sin{\theta_+}| \leq \tau   
\end{align*}
whose objective function can be furthered simplified to 
\[
\max_{\theta \in [0, \pi]} \min \{\pi-\theta, \theta\}.
\]
Here two cases need be considered: (i) $\theta \in [0, \frac{\pi}{2}]$ and (ii) $\theta \in [ \frac{\pi}{2}, \pi]$. For (i) the problem simplifies to
\begin{align*}
        \max_{\theta \in [0, \frac{\pi}{2}]} & \quad \theta \\
        \text{s.t.} \hspace{5pt} |x_1^{(2)}| \leq \tau, \hspace{5pt} \theta_-, \theta_+ & \in \arcsin{[\frac{-\tau-x_1^{(2)}}{e}, \frac{\tau-x_1^{(2)}}{e}]}
    \end{align*}
Because $\arcsin$ is a strictly increasing function on its domain, the maximum is achieved when $\theta_-,\theta_+$ take the two ends of their constraint interval. It follows that 
\begin{align*}
        \max \theta & = \arcsin\left(\frac{\tau-x_1^{(2)}}{e}\right) - \arcsin\left(\frac{-\tau-x_1^{(2)}}{e}\right) \\ 
    & = \arcsin\left(\frac{\tau-x_1^{(2)}}{e}\right) + \arcsin\left(\frac{\tau+x_1^{(2)}}{e}\right) \\ 
    & = \arcsin(a-b) + \arcsin(a+b) \eqqcolon f(b) 
    \end{align*}
where $a = \frac{\tau}{e}$ is a constant and $b=\frac{x_1^{(2)}}{e}$ is a variable in $[-a, a]$.  Note that $f(b)$ is an even function and the maximum is achieved at the boundary points $b=\pm a$. Equivalently, we must have $x_1^{(2)}=\pm \tau$. We list all the possible configurations that achieve the maximum 
\begin{itemize}
    \item $x_1^{(2)}=\tau, x_-^{(2)}=\tau, x_+^{(2)}=-\tau$. 
    \item $x_1^{(2)}=\tau, x_-^{(2)}=-\tau, x_+^{(2)}=\tau$.
     \item $x_1^{(2)}=-\tau, x_-^{(2)}=\tau, x_+^{(2)}=-\tau$. 
    \item $x_1^{(2)}=-\tau, x_-^{(2)}=-\tau, x_+^{(2)}=\tau$.
\end{itemize}
In words, the optimal configuration for the $\theta \in [0, \frac{\pi}{2}]$ case occurs when the shared face $x_1$ lies on one side of the boundary of $T_{\tau}(\M_1)$, along with one of the two apex vertices $\{x_-, x_+\}$ lying on the same side of the boundary; The other apex vertex lies on the opposite side of the boundary. See $\theta_1$ in Figure \ref{fig:1-simplex in D=2 space}. One can verify that $\theta_{\max}$, in this case, equals $\arcsin\frac{2\tau}{e}$ and so the optimal value of case (i), for $\theta \in [0, \frac{\pi}{2}]$, is $\arcsin\frac{2\tau}{e}$. 

\begin{figure}[h]
\centering
\begin{subfigure}{0.4\textwidth}
\includegraphics[width=\linewidth]{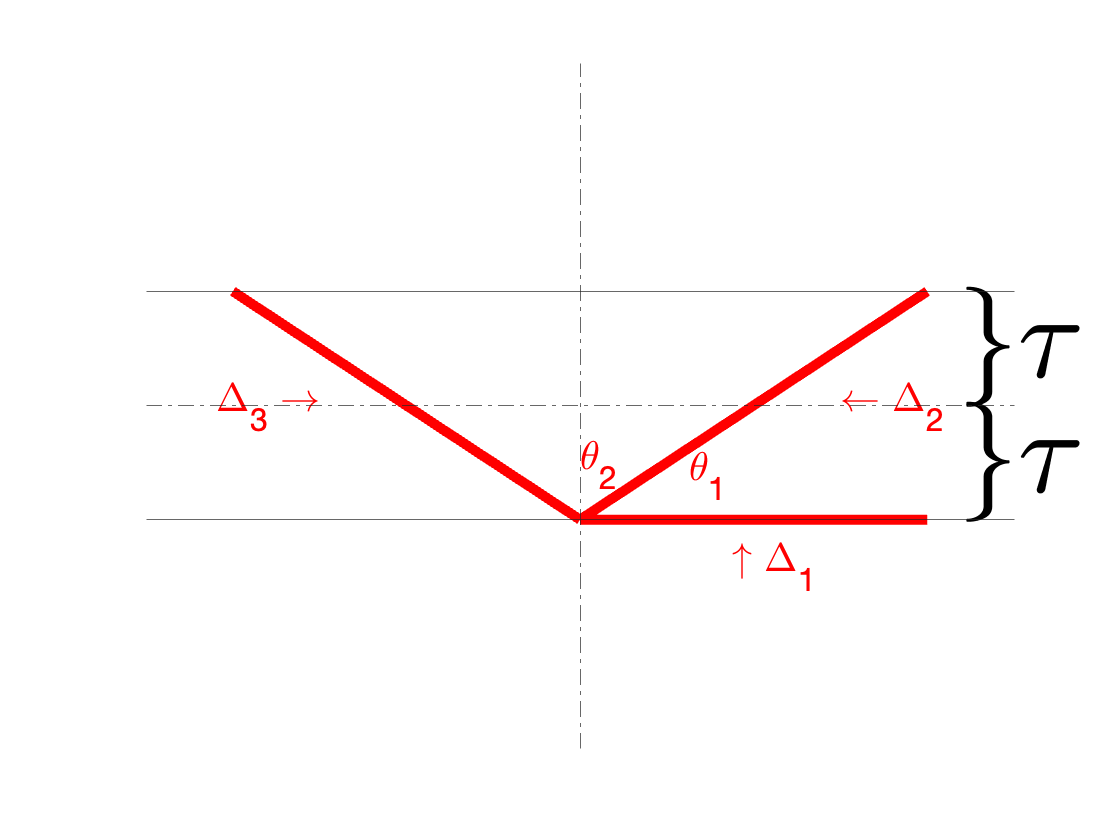}
\caption{$d=1$, $D=2$.} 
\label{fig:1-simplex in D=2 space}
\end{subfigure}\hspace*{\fill}
\begin{subfigure}{0.4\textwidth}
\includegraphics[width=\linewidth]{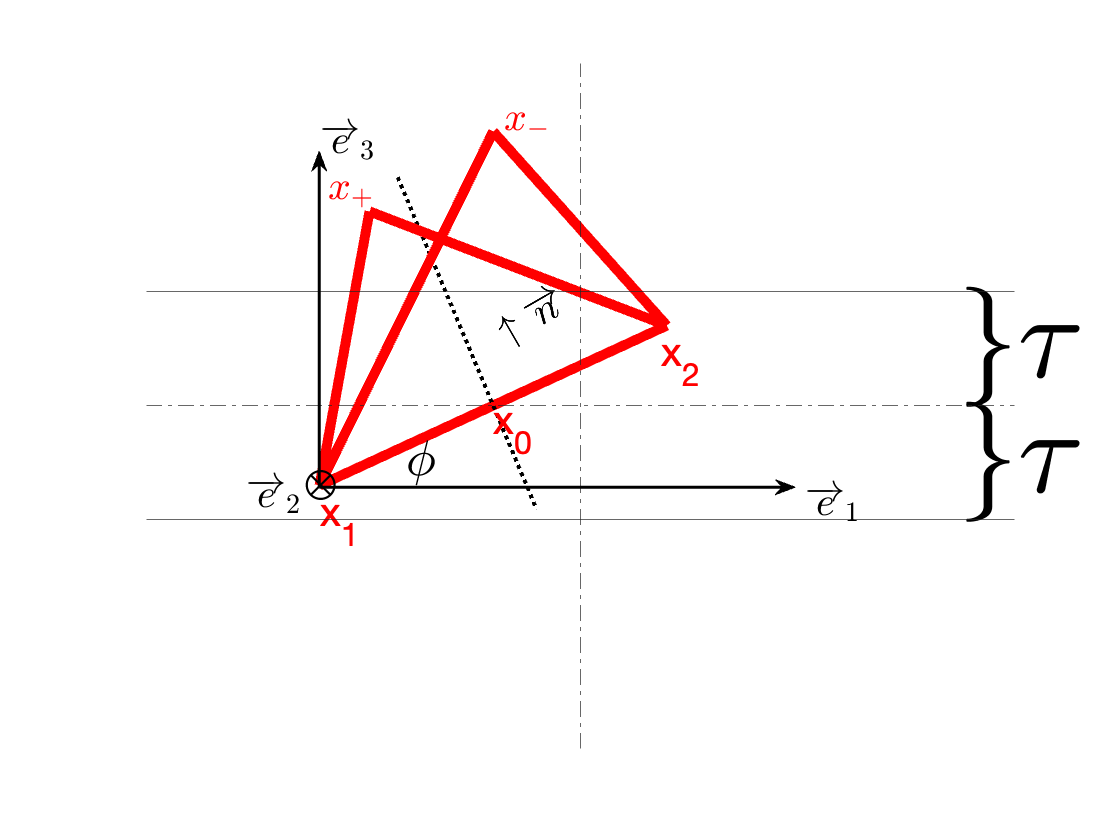}
\caption{Coordinate construction $d=2$} 
\label{fig:construction d=2}
\end{subfigure}
\caption{(a) Optimal configuration for $d=1$ that attains the largest wLAPD with constraints of noise $\tau$ and simplex size $e$. $\theta_1$ corresponds to the (local) optimal configuration for $\theta=|\theta_+-\theta_-| < \frac{\pi}{2}$; $\theta_2$ corresponds to the (global) optimal configuration for $\theta=|\theta_+-\theta_-| > \frac{\pi}{2}$. (b) Construction of coordinates for $d=2$.} \label{fig:largest wLAPD d=1}
\end{figure}

For (ii), the problem reduces to 
\begin{equation*}
    \label{eq:simplified obj (2)}
    \begin{alignedat}{2}
        & \hspace{100pt}\max_{\theta \in [\frac{\pi}{2}, \pi]} \quad \pi-\theta \\
        \text{s.t.} \quad & |x_1^{(2)}| \leq \tau, \quad \theta_- \in \arcsin[0, \frac{\tau-x_1^{(2)}}{e}], \hspace{5pt} \theta_+ \in \pi -\arcsin[0, \frac{\tau-x_1^{(2)}}{e}] \\
        \text{or} \quad & |x_1^{(2)}| \leq \tau, \quad \theta_- \in \arcsin[ \frac{-\tau-x_1^{(2)}}{e}, 0], \hspace{5pt} \theta_+ \in -\pi-\arcsin[\frac{-\tau-x_1^{(2)}}{e}, 0]
    \end{alignedat}
\end{equation*}

The two constraint regions simplify the objective function individually to the following two cases: 
\begin{align*}
    \max_{\theta \in [\frac{\pi}{2}, \pi]} \pi-\theta = 
    \begin{cases}
        2\arcsin\frac{\tau-x_1^{(2)}}{e}, \; & \text{if} \; \theta_- = \arcsin \frac{\tau-x_1^{(2)}}{e}, \theta_+ = \pi-\arcsin \frac{\tau-x_1^{(2)}}{e}.  \\
        2\arcsin\frac{\tau+x_1^{(2)}}{e}  & \text{if} \; \theta_- = \arcsin \frac{-\tau-x_1^{(2)}}{e}, \theta_+ = -\pi-\arcsin \frac{-\tau-x_1^{(2)}}{e}.
    \end{cases}
\end{align*}

It is easy to verify that in the former case, we need $x_1^{(2)}=-\tau$ and that leads to $\theta = \pi-2\arcsin\frac{2\tau}{e}$ and $x_\pm^{(2)}= \tau$. In the later case, we need $x_1^{(2)}=\tau$ and that leads to $\theta = \pi-2\arcsin\frac{2\tau}{e}$ again and $x_\pm^{(2)}= -\tau$. We list all the possible configurations that achieve the optimum
\begin{itemize}
    \item $x_1^{(2)}=\tau, x_\pm^{(2)}=-\tau$.
    \item $x_1^{(2)}=-\tau,x_\pm^{(2)}=\tau$. 
\end{itemize}

In words, the optimal configuration for the $\theta \in [\frac{\pi}{2},\pi]$ case occurs when the shared face lies on one side of the boundary of $T_{\tau}(\M_1)$ and the two apex vertices lie on the other side of the boundary. See $\theta_2$ in Figure \ref{fig:1-simplex in D=2 space}. The maximal value for case (ii) is $2\arcsin\frac{2\tau}{e}$. 

Combining (i) and (ii), we see that (ii) has the configuration that achieves the optimal value for problem \eqref{eq:LAPD within optimization (2)}, which is $2\arcsin\frac{2\tau}{e} \approx \frac{4\tau}{e}$ when $\tau \ll e$ by Assumption \ref{assump:small noise}. The corresponding optimal configuration is 
\begin{align*}
    x_1^{(2)} = -\tau, \;\; \text{and}  \;\; x_{\pm}^{(2)} = \tau, 
\end{align*}
or 
\begin{align*}
    x_1^{(2)} = \tau, \;\; \text{and}  \;\; x_{\pm}^{(2)} = -\tau\, ,  
\end{align*}
and we conclude the proof for $d=1$. 

Now consider $d=2$ for the same optimization problem in \eqref{eq:LAPD within optimization (2)}. Denote by $\Delta_- = [x_1, x_2, x_-]$ and $\Delta_+ = [x_1, x_2, x_+]$ the two adjacent regular triangles. Following the same logic for $d=1$, we shift the coordinate system such that $x_1 = (0,0,x_1^{(3)})^\intercal$ is the origin of span$\{e_1, e_2\}$ and $x_2$ lives in the subspace of span$\{e_1, e_3\}$. $\overrightarrow{x_1x_2}$ satisfies the edge length constraint and so we must have $x_2 = (e\cos\phi,0,x_1^{(3)}+e\sin\phi)^\intercal$ with $\phi \in [0, \pi)$ denoting the angle of the shared face $\overrightarrow{x_1x_2}$ flipped from span$\{ e_1, e_2 \}$. See Figure \ref{fig:construction d=2} for illustration. The case when $\phi \in [\pi, 2\pi)$ can be be analogously analyzed with a symmetry argument. Fixing the position of the shared face $\overrightarrow{x_1x_2}$, the two triangles $\Delta_-, \Delta_+$ can rotate about $\overrightarrow{x_1x_2}$. We introduce $\theta_\pm$ to characterize the angle between $\Delta_\pm$ from span$\{ e_2, \overrightarrow{x_1x_2} \}$. Again we enforce $\theta_- \in [-\frac{\pi}{2}, \frac{\pi}{2}]$ and $\theta = |\theta_--\theta_+| \in [0, \pi]$ w.l.o.g. Due to the edge length constraint, the coordinates of the apex points can be explicitly calculated as 


\begin{align*}
    x_- & = (-\frac{\sqrt{3}}{2}e\sin\theta_-\sin\phi+\frac{e}{2}\cos\phi, \frac{\sqrt{3}}{2}e\cos\theta_-, x_1^{(3)}+\frac{\sqrt{3}}{2}e\sin\theta_-\cos\phi+\frac{e}{2}\sin\phi)^\intercal \\ 
    x_+ & = (-\frac{\sqrt{3}}{2}e\sin\theta_+\sin\phi+\frac{e}{2}\cos\phi, \frac{\sqrt{3}}{2}e\cos\theta_+, x_1^{(3)}+\frac{\sqrt{3}}{2}e\sin\theta_+\cos\phi+\frac{e}{2}\sin\phi)^\intercal
\end{align*}
Note that the middle point $x_0$ of $\overrightarrow{x_1x_2}$ is equal to $\frac{x_1+x_2}{2} = (\frac{e}{2}\cos\phi, 0, x_1^{(3)}+\frac{e}{2}\sin\phi)^\intercal$ and so the two residual vectors $v_- \coloneqq \overrightarrow{x_0x_-}$ and $v_+ \coloneqq\overrightarrow{x_0x_+}$ can be calculated. The optimization problem in \eqref{eq:LAPD within optimization (2)} then boils down to

\begin{align*}
\max_{\theta_-, \theta_+, x_1^{(3)}}  \min &\{\pi- \arccos\frac{v_- \cdot v_+}{\|v_-\| \|v_+\|}, \arccos\frac{v_- \cdot v_+}{\|v_-\| \|v_+\|}\} \\ 
\text{s.t.} \hspace{5pt} |x_1^{(3)}| &\leq \tau, \hspace{5pt} |x_1^{(3)} +e\sin{\phi}| \leq \tau, \\ |x_1^{(3)}+\frac{\sqrt{3}}{2}e\sin\theta_-\cos\phi &+\frac{e}{2}\sin\phi| \leq \tau, \hspace{5pt} |x_1^{(3)}+\frac{\sqrt{3}}{2}e\sin\theta_+\cos\phi+\frac{e}{2}\sin\phi| \leq \tau.    
\end{align*}

Straightforward calculation, using trigonometric properties, reveals that $v_- \cdot v_+ = \frac{3}{4}e^2\cos{\theta}$ with $\|v_-\| = \|v_+\| = \frac{\sqrt3}{2}e$ and so, the above objective can be further reduced to 
\begin{align*}
    \max_{\theta \in [0, \pi]} \min \{\pi-\theta, \theta\}
\end{align*}
Once again we are facing two cases: (i) $\theta \in [0,\frac{\pi}{2}]$ and (ii) $\theta \in [\frac{\pi}{2}, \pi]$. 

For (i), the optimization problem becomes
    \begin{align*}
        \max_{\theta \in [0, \frac{\pi}{2}]} &\theta \\
        \text{s.t.} \hspace{5pt} |x_1^{(3)}| \leq \tau , \hspace{5pt} |x_1^{(3)} & +e\sin\phi|\leq \tau
        \\
        \theta_-, \theta_+ \in \arcsin[\frac{-2\tau-2x_1^{(3)}-e\sin\phi}{\sqrt{3}e\cos\phi}&, \frac{2\tau-2x_1^{(3)}-e\sin\phi}{\sqrt{3}e\cos\phi}]
    \end{align*}
Note that the maximum is obtained when $\theta_-$ and $\theta_+$ take the two endpoints of the domain interval. That is, 
\begin{align*}
    \max \theta &= \arcsin\left(\frac{2\tau-2x_1^{(3)}-e\sin\phi}{\sqrt{3}e\cos\phi}\right)-\arcsin\left(\frac{-2\tau-2x_1^{(3)}-e\sin\phi}{\sqrt{3}e\cos\phi} \right) \\ 
    &=\arcsin\left(\frac{2\tau-(2x_1^{(3)}+e\sin\phi)}{\sqrt{3}e\cos\phi}\right)+\arcsin\left(\frac{2\tau+2x_1^{(3)}+e\sin\phi}{\sqrt{3}e\cos\phi} \right) \\ 
    &= \arcsin(a-b)+\arcsin(a+b) \eqqcolon f(b)
\end{align*}
Note that for a fixed $\phi$, $a = \frac{2\tau}{\sqrt{3}e\cos\phi}$ is a constant whereas $b = \frac{2x_1^{(3)}+e\sin\phi}{\sqrt{3}e\cos\phi} \in [-a, a]$ is the variable. Analogous to $d=1$, the maximum of $f(b)$ is attained at $b=\pm a$. $b=a$ implies $2x_1^{(3)}+e\sin\phi=2\tau$. Due to the constraints $|x_1^{(3)}|\leq\tau$, $|x_1^{(3)}+e\sin\phi|\leq\tau$, and $\phi \in [0,\pi)$, this can only happen when $x_1^{(3)}=\tau$ and $x_1^{(3)}+e\sin\phi=\tau$, which uncovers the configuration $x_1^{(3)}=\tau$ and $\phi=0$. Analogously, $b=-a$ implies $x_1^{(3)}=-\tau, \phi=0$. In the former case, when $x_1^{(3)}=\tau$ and $\phi=0$, one of $\theta_-, \theta_+$ equals 0 and the other equals $\arcsin\left(\frac{-4\tau}{\sqrt3e}\right)$, which suggests that one of $x_-^{(3)}, x_+^{(3)}$ equals $\tau$ and the other equals $-\tau$; In the later case, when $x_1^{(3)}=-\tau$ and $\phi=0$, one of $\theta_-, \theta_+$ equals 0 and the other equals $\arcsin\left(\frac{4\tau}{\sqrt3e}\right)$, which suggests that one of $x_-^{(3)}, x_+^{(3)}$ equals $-\tau$ and the other one equals $\tau$. We list all the possible configurations that achieve the maximum. 

\begin{itemize}
    \item $x_1^{(3)}=x_2^{(3)}=\tau, x_-^{(3)}=\tau, x_+^{(3)}=-\tau$. 
    \item $x_1^{(3)}=x_2^{(3)}=\tau, x_-^{(3)}=-\tau, x_+^{(3)}=\tau$.
    \item $x_1^{(3)}=x_2^{(3)}=-\tau, x_-^{(3)}=\tau, x_+^{(3)}=-\tau$. 
    \item $x_1^{(3)}=x_2^{(3)}=-\tau, x_-^{(3)}=-\tau, x_+^{(3)}=\tau$.
\end{itemize}
In other words, the optimal configuration for $\theta \in [0, \frac{\pi}{2}]$ occurs when the shared face $\overrightarrow{x_1x_2}$ lies on one side of the boundary of $T_{\tau}(\M_1)$, with one of the two apex vertices $\{x_-, x_+\}$ lying on the same side of the boundary and the other apex vertex lying on the opposite side. One can verify that $\theta_{\max}$, in this case, equals  $\arcsin\frac{4\tau}{\sqrt3e}$ and so the optimal value of case (i), for $\theta \in [0, \frac{\pi}{2}]$, is $\arcsin\frac{4\tau}{\sqrt3e}$.

For (ii), without loss of generality, the optimization problem becomes

\begin{align*}
        \max_{\theta \in [\frac{\pi}{2}, \pi]} & \pi- \theta \\
        \text{s.t.} \hspace{5pt} |x_1^{(3)}| \leq \tau & , \hspace{5pt} |x_1^{(3)}+e\sin\phi|\leq \tau \\
        \theta_- \in  [0, \frac{2\tau-2x_1^{(3)}-e\sin\phi}{\sqrt{3}e\cos\phi}], & \;\;\theta_+ = \pi - [0, \frac{2\tau-2x_1^{(3)}-e\sin\phi}{\sqrt{3}e\cos\phi}] \\ 
        \text{or} \; \theta_- \in  [\frac{-2\tau-2x_1^{(3)}-e\sin\phi}{\sqrt{3}e\cos\phi}, 0], & \;\;\theta_+ = -\pi - [ \frac{-2\tau-2x_1^{(3)}-e\sin\phi}{\sqrt{3}e\cos\phi},0]
    \end{align*}
where the two constraints on $\theta_\pm$ depend on the quadrants where the two simplices live. The above can be further simplified to 
\begin{align*}
    \max_{\theta \in [\frac{\pi}{2}, \pi]} \pi- \theta = 
    \begin{cases}
        2\arcsin\frac{2\tau-2x_1^{(3)}-e\sin\phi}{\sqrt3e\cos\phi}, & \theta_- \in [0, \frac{\pi}{2}], \theta_+ \in [\frac{\pi}{2}, \pi]  \\ 
        2\arcsin\frac{2\tau+2x_1^{(3)}+e\sin\phi}{\sqrt3e\cos\phi}, & \theta_- \in [-\frac{\pi}{2}, 0],  \theta_+ \in [-\pi, -\frac{\pi}{2}] 
    \end{cases}
\end{align*}

It is easy to verify that in the former case, we need $2x_1^{(3)}+e\sin\phi=-2\tau$ and in the later case, we need $2x_1^{(3)}+e\sin\phi=2\tau$. Due to the constraints, it has to be that in the former case, $x_1^{(3)}=-\tau, x_1^{(3)}+e\sin\phi=-\tau$, which indicates $\phi=0$ and so $x_-=x_+ = \tau$. Similarly, in the later case, $x_1^{(3)}=\tau$, which indicates $\phi=0$ and so $x_-=x_+= -\tau$. We list all the possible configurations that achieve the maximum.
\begin{itemize}
    \item $x_1^{(3)}=x_2^{(3)}=\tau, x_\pm^{(3)}=-\tau$
    \item $x_1^{(3)}=x_2^{(3)}=-\tau,x_\pm^{(3)}=\tau$ 
\end{itemize}

In words, the optimal configuration for $\theta \in [\frac{\pi}{2},\pi]$ occurs when the shared face lies on one side of the boundary of $T_{\tau}(\M_1)$ and the two apex vertices lie on the other side of the boundary. The maximal value for case (ii) is $2\arcsin\frac{4\tau}{\sqrt3e}$. 

Combining (i) and (ii), we see that (ii) has the configuration that achieves the (global) optimal value for the objective function \eqref{eq:LAPD within optimization (2)} for $d=2$. See Figure \ref{fig:optimal configuration d=2}. In particular, the optimal value equals $2\arcsin\frac{4\tau}{\sqrt3e} \approx \frac{8\tau}{\sqrt3e}$ when $\tau \ll e$ by Assumption \ref{assump:small noise}. The optimal configuration is 
\begin{align*}
    x_1^{(3)} =  x_2^{(3)} = -\tau, \;\; \text{and } x_{\pm}^{(3)} = \tau, 
\end{align*}
or 
\begin{align*}
    x_1^{(3)} =  x_2^{(3)} = \tau, \;\; \text{and } x_{\pm}^{(3)} = -\tau.  
\end{align*}
We therefore finish the proof for $d=2$. 

\begin{figure}[h]
   \centering    \includegraphics[width=0.65\textwidth]{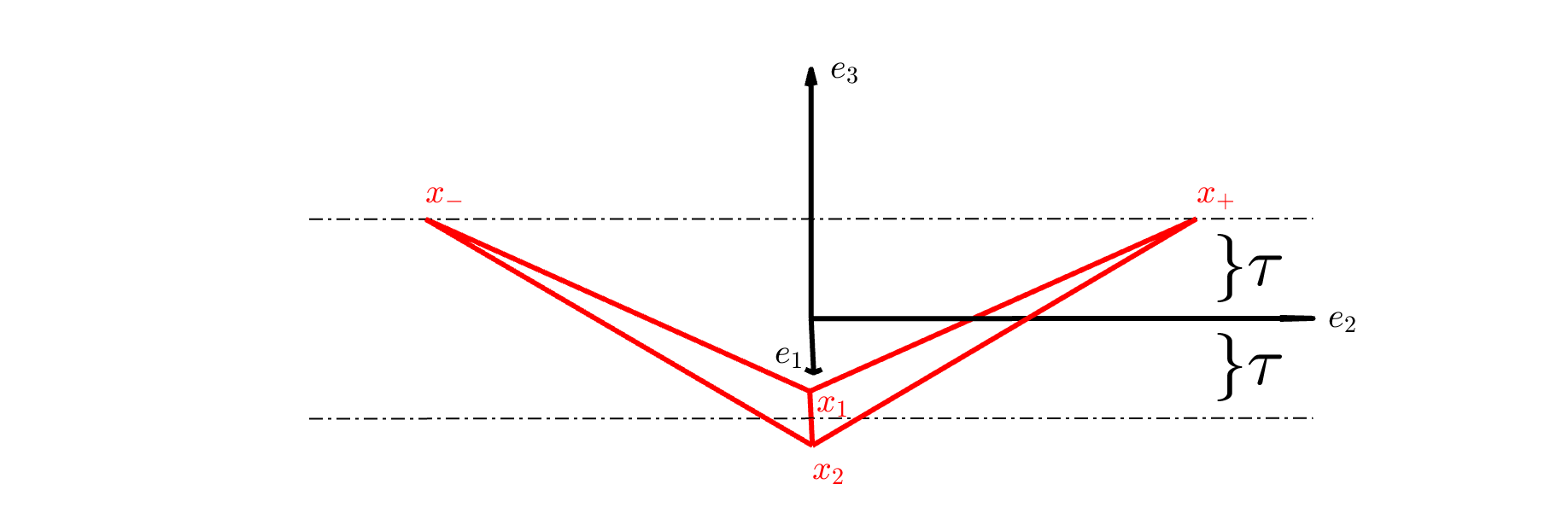}
    \caption{Optimal Configuration for $d=2$.}
   \label{fig:optimal configuration d=2}
\end{figure}

To draw the same conclusion for $d>2$, one needs $x_1^{(d+1)}$ and $d+1$ additional projection parameters (the angles) to capture the constraints, i.e. one must solve an (nonconvex) optimization problem with $d+2$ unknowns. The analytical derivation of the solution is straightforward but cumbersome, and the optimal configuration for any dimension $d$ is always when the apex vertices and the shared face appear on the opposite side of $T_{\tau}(\M_1)$. We skip the proof for the sake of brevity. 

\end{proof}

\subsubsection{wLAPD Upper Bound}
\label{ss: largest wLAPD}
In the last section, we showed that when $q=1$, the global optimum for \eqref{eq:LAPD within optimization (2)} occurs when $x_i^{(d+1)}$ for $i \in [d]$ and $x_{\pm}^{(d+1)}$ have the same scale $\tau$ but opposite sign. Following this optimal configuration, we are going to find a general formula for the upper bound on \eqref{eq:LAPD within optimization (2)} for arbitrary $d$'s. Note that as a byproduct of the calculation in the last section, this upper bound for $d=1$ and $d=2$ were found to be $\frac{4\tau}{e}$ and $\frac{8\tau}{\sqrt{3}e}$, respectively. Our general formula to be found should verify these results. 

Since we can translate the optimal configuration along the manifold, without loss of generality we can make the centroid of the shared face be at zero in the first $d$ coordinates, i.e.
\begin{align*}
    x_0 &= (0, \ldots, 0, -\tau).
\end{align*}
Also, since we can always rotate the optimal configuration in the manifold dimensions, we can assume without loss of generality that the first $d-1$ coordinates of $x_-, x_+$ are zero (they span a 2-dimensional space, but only 1 manifold dimension, which we can rotate to the $d^{th}$ coordinate), and symmetry will give:
\begin{align*}
    x_- &= (0, \ldots, -x, \tau), \;\;\; x_+ = (0, \ldots, x, \tau)
\end{align*}
for some positive constant $x$. However, from the Pythagorean theorem we know $\|\overrightarrow{x_0 x_-}\| =\|\overrightarrow{x_0 x_+}\| = e \sqrt{\frac{d+1}{2d}}$. Solving for $x$ gives $x = \sqrt{e^2\frac{(d+1)}{2d} - (2\tau)^2}$, so that the angle $\theta$ between $\overrightarrow{x_0 x_-},\overrightarrow{x_0 x_+}$ satisfies $\cos{(\theta)} = \frac{16d}{d+1} \frac{\tau^2}{e^2} - 1$. Converting this to $\sin(\theta)$ and by the assumption $\tau \ll e$ we have 
\[
\theta \approx \sin{(\theta)} = \sqrt{\frac{32d}{d+1} \frac{\tau^2}{e^2}- \frac{(16d)^2}{(d+1)^2}\frac{\tau^4}{e^4}} \approx \sqrt{\frac{32d}{d+1}} \frac{\tau}{e}. 
\]
Hence for any $d$, wLAPD has an upper bound $ \sqrt{\frac{32d}{d+1}} \frac{\tau}{e} \sim O\left( \frac{\tau}{e} \right)$. This upper bound verifies our finding from explicit construction and calculation in the last section for $d=1,2$. Note that the empirical result from optimization ($q=1$ curve in Figure \ref{fig:wLAPD bound vs q}) reflects this theoretical bound. 

\subsubsection{Generalization to nonregular simplices}
\label{ss: nonregular simplices}

In the previous section, we found the theoretical upper bound on wLAPD for $q=1$, i.e. when all the simplices are regular. It is hard to navigate the situation when $q$ need be relaxed, which is always the case in practice. We provide a numerical result from optimization in Figure \ref{fig:wLAPD bound vs q}. It is clearly seen that wLAPD remains upper bounded when $q$ is moderately relaxed. 

\begin{figure}[h]
\centering
\begin{subfigure}{0.45\textwidth}
\includegraphics[height=5cm, width=\linewidth]{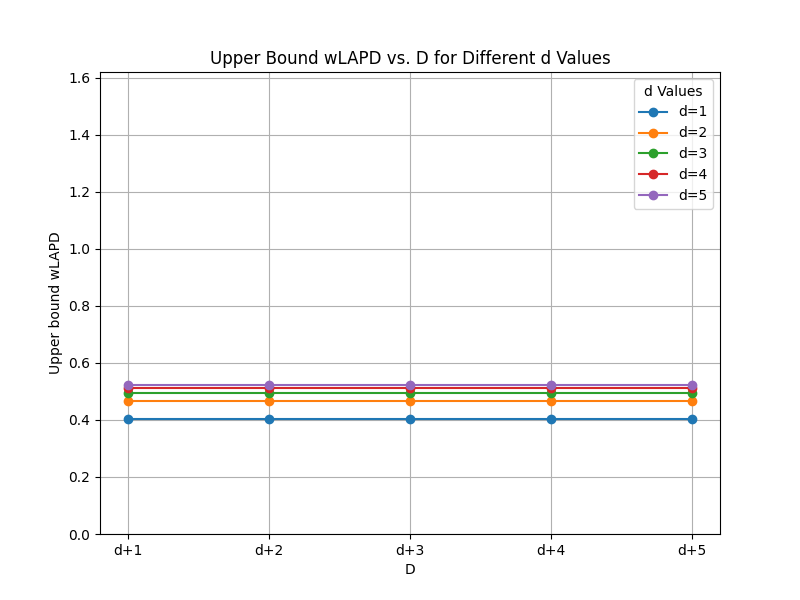}
\caption{wLAPD Upper Bound vs $D$.} 
\label{fig:wLAPD bound vs D}
\end{subfigure}\hspace*{\fill}
\begin{subfigure}{0.45\textwidth}
\includegraphics[height=5cm,width=\linewidth]{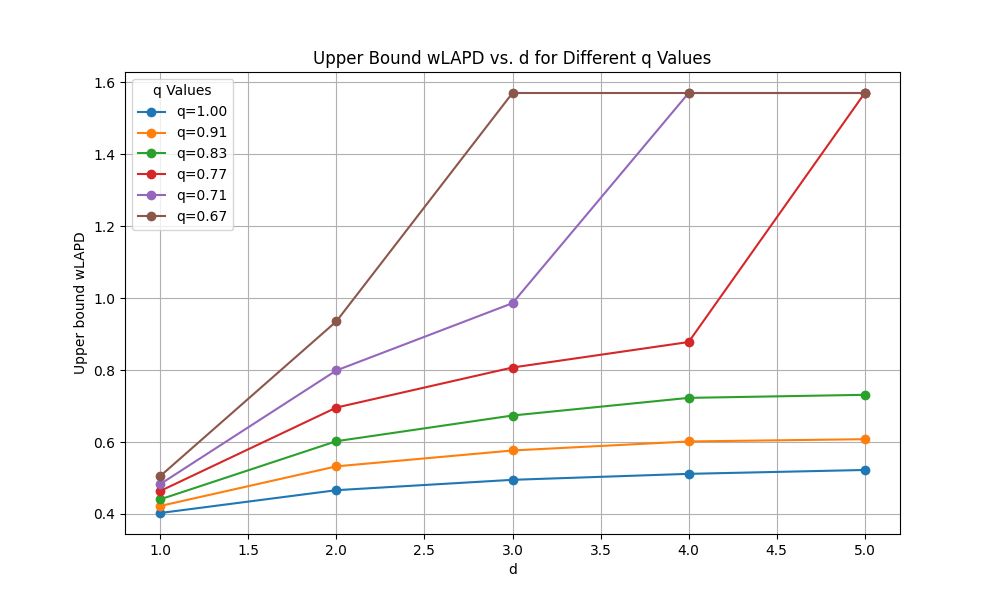}
\caption{wLAPD Upper Bound vs $q$.} 
\label{fig:wLAPD bound vs q}
\end{subfigure}
\caption{wLAPD upper bound simulation results. In (a), $q=1$ is fixed. It verifies the independence of the optimial configuration and wLAPD upper bound from the ambient dimension $D$. In (b), $D=d+1$ is fixed but $q$ is relaxed. It reveals that moderately relaxing $q$ maintains a reasonable upper bound on wLAPD. }
\label{Largest WLAPD (2)}
\end{figure}

\section{Proof of Theorem \ref{thm:between LAPD lower bound}}
\label{pf:between LAPD lower bound}

We now prove the lower bound on bLAPD given in Theorem \ref{thm:between LAPD lower bound}. Due to the relationship between the two weight constructions LAPD$^1\geq$ LAPD$^2$ by \eqref{eq:two LAPD relationship}, it is sufficient to obtain a lower bound for bLAPD under the two-way weight construction:
\begin{align*}    W^2_S(\Delta_i,\Delta_j)&=\min\{\pi -\theta_{ij}, \theta_{ij}\} \, .
\end{align*}

Recall we have $n$ points i.i.d sampled from the uniform volume measure $\text{vol}_D$ on $T_\tau(\M)= T_\tau(\M_1) \cup T_\tau(\M_2)$, where $\text{vol}_D$ refers to the $D$-dimensional Hausdorff measure. For adjacent simplices $\Delta_i, \Delta_j$, we say $\Delta_i$ $\epsilon$-links to $\Delta_j$ (denoted $\Delta_i \rightarrow_{\epsilon} \Delta_j$) if $W^2_S(\Delta_i,\Delta_j) \leq \epsilon$. We also define a simplex to be \textit{straddle} if it forms an angle of at least $\frac{\Theta}{4}$ with both $\M_1$ and $\M_2$, and let $S^{\text{str}}$ denote the set of valid straddle simplices. To lower bound bLAPD by a path enumeration argument, we will need the following supporting results.

\begin{lemma}
\label{lem:tube_size}
Suppose $R \geq \frac{5e}{q \sin\Theta}$. Then any mixed simplex is formed by points contained in $T_R(\M_1 \cap \M_2)$.
\end{lemma}

\begin{proof}
    Let $P$ denote projection onto the intersection. We need to show that if $\Delta$ is a mixed simplex, then $\|x-Px\|\leq \frac{5e}{q\sin\Theta}$ for all $x \in \Delta$. Pick an arbitrary $x_1\in\Delta$; without loss of generality assume $z_1\in\M_1$, i.e. the noiseless sample corresponding to $x_1$ came from $\M_1$; since $\Delta$ is mixed, there exists $x_2\in\Delta$ such that $z_2\in\M_2$. Since $\|x_1-x_2\| \leq \frac{e}{q}$ by construction of valid simplex set $S$ in \eqref{eq:simplex quality constraint}, $\|z_1-z_2\| \leq \frac{e}{q}+2\tau$ by the triangle inequality. Note:
    \begin{align*}
        z_2 - z_1 = Pz_1-z_1 + Pz_2-Pz_1+z_2-Pz_2 := p_1 + p_{\text{int}} + p_2
    \end{align*}
    where $p_1 = Pz_1-z_1$, $p_{\text{int}}=Pz_2-Pz_1$, $p_2 = z_2-Pz_2$. Note that $p_{\text{int}}$ is orthogonal to $p_1+p_2$, so
    \begin{align*}
        \| z_2-z_1\| &= \|p_{\text{int}}\| + \|p_1+p_2\| \\
        \implies \|p_1+p_2\| &\leq \| z_2-z_1\| \leq \frac{e}{q}+2\tau\, .
    \end{align*}
    Furthermore, by the Law of Sines $\|p_1\| \leq \frac{\|p_1+p_2\|}{\sin\Theta} \leq \frac{\frac{e}{q}+2\tau}{\sin\Theta}$ and similarly for $\|p_2\|$. Thus we conclude $\|Pz_1-z_1\| \leq \frac{\frac{e}{q}+2\tau}{\sin\Theta}$. Finally we relate to the noisy distances:
    \begin{align*}
        \| x_1 - Px_1\| &= \| x_1 - z_1 + z_1-Pz_1 + Pz_1-Px_1\| \\
        &\leq \| x_1 - z_1\| + \|z_1-Pz_1 \|+ \|Pz_1-Px_1\| \\
        &\leq \| x_1 - z_1\| + \|z_1-Pz_1 \|+ \|z_1-x_1\| \\
        &\leq \frac{\frac{e}{q}+2\tau}{\sin\Theta} + 2\tau \leq \frac{\frac{e}{q}+4\tau}{\sin\Theta} \leq \frac{5e}{q\sin\Theta}
    \end{align*}
    where the last inequality is due to Assumption \ref{assump:small noise}. 
\end{proof}

\begin{lemma}
\label{lem:bound_intersec_pts}
    Let $n_T$ denote the number of points which land in $\mathcal{T}_R = T_R(\M_1 \cap \M_2) \cap T_{\tau}(\M)$ for some $R \in[\tau, R_{\max}]$ as defined in Assumption \ref{assump:volume_growth}. 
    Then $n_T \leq 2 C_2 n R$ with probability at least $1-\exp\left(-\frac{C_1 nR}{3}\right)$, with $C_1, C_2$ defined in \eqref{equ:elong}.
\end{lemma}
\begin{proof}
    The number of points landing in $\mathcal{T}_R$ is a Binomial random variable with parameters $(n, p)$ where $p$ is the volume ratio between the tubular region of radius $R$ around the $d-1$ dimensional intersection $\text{vol}_D\left(\mathcal{T}_R\right)$ and the tubular region of radius $\tau$ around the collection of $d$ dimensional manifolds $\text{vol}_D \left(T_\tau \left(\M \right) \right)$, i.e. by  \eqref{equ:elong}  
    \begin{align*}        p=\frac{\text{vol}_D\left(\mathcal{T}_R\right)}{\text{vol}_D(T_\tau(\M))} \in
    \frac{[C_1, C_2] R \tau^{D-d}}{\text{vol}_d(\M)\tau^{D-d}} \sim R.
    \end{align*}  
    Hence, the Binomial distribution of $n_T$ has a mean that scales like $\mu \sim nR$. A multiplicative Chernoff bound gives
    \begin{align*}
        \mathbb{P}\left(n_T \geq (1+t) \mu\right) &\leq \exp\left(-\frac{t^2\mu}{2+t}\right)
    \end{align*}
    for $t\geq 0$. For convenience, we choose $t=1$, and using \eqref{equ:elong}, we have $C_1 n R \leq \mu \leq C_2 n R$, so that
    \begin{align*}
        \mathbb{P}\left(n_T \geq 2 C_2 n R \right) \leq \mathbb{P}\left(n_T \geq 2 \mu\right) &\leq \exp\left(-\frac{\mu}{3}\right) \leq \exp\left(-\frac{C_1 nR}{3}\right) \, .
    \end{align*}
\end{proof}

\begin{lemma}
\label{lem:number_links}
 If bLAPD $\leq \epsilon$, then there exists a path of $\frac{\Theta}{2\epsilon}$ consecutive straddle simplices which link with a weight bounded by $\epsilon$.   
\end{lemma}
\begin{proof}
Since bLAPD $\leq \epsilon$, there exists a path of all weights $W^2_S(\Delta_i,\Delta_j)$ less than $\epsilon$ between the manifolds; that is, we have a path connecting the manifolds where the sequence $a_i = W^2_S(\Delta_i, \Delta_{i+1})$ is small. Consider also the dihedral angle made by $\Delta_i$ with the two manifolds $b_i = \angle(\Delta_i, \M_1)$ and $c_i = \angle(\Delta_i, \M_2)$. Note that we have the following triangle inequality
\begin{align*}
    & \angle(\Delta_i, \M_1) \leq W^2_S(\Delta_i, \Delta_{i+1}) + \angle(\Delta_{i+1}, \M_1) \leq \epsilon + \angle(\Delta_{i+1}, \M_1) \\
    & \angle(\Delta_{i+1}, \M_1) \leq W^2_S(\Delta_{i+1}, \Delta_{i}) + \angle(\Delta_{i}, \M_1) \leq \epsilon + \angle(\Delta_{i}, \M_1) \\
    & \implies |b_i - b_{i+1}| = |\angle(\Delta_{i}, \M_1) - \angle(\Delta_{i+1}, \M_1)| \leq \epsilon\, ,
\end{align*}
and similarly $|c_i - c_{i+1}| \leq \epsilon$. Now the sequence $b_i$ increases from $O(\frac{\tau}{e})$ to $\Theta -O(\frac{\tau}{e})$ and $c_i$ decreases from $\Theta -O(\frac{\tau}{e})$ to $O(\frac{\tau}{e})$. At some point $b^*$ it is very close to $\frac{\Theta}{2}$. Note that since the sequence can move by at most $\epsilon$ in one step, if we take $\lfloor\frac{\Theta}{4\epsilon}\rfloor$ steps before and after $b^*$, we will not exit the interval $[\frac{1}{4}\Theta, \frac{3}{4}\Theta]$. That is, there has to be a sequence of $\lfloor\frac{\Theta}{2\epsilon}\rfloor$ consecutive links $\Delta_i \rightarrow \Delta_{i+1}$ where $b_i \in [\frac{1}{4}\Theta, \frac{3}{4}\Theta]$. Application of the triangle inequality ($b_i+c_i \geq \Theta$) gives $c_i \geq \Theta - b_i \geq \frac{\Theta}{4}$. Hence we conclude that there is a sequence of $\lfloor\frac{\Theta}{2\epsilon}\rfloor$ consecutive simplices whose angles with both $\M_1$ and $\M_2$ are at least $\frac{\Theta}{4}$, which is our definition of the straddle simplices.
\end{proof}

\begin{proposition}[Single Link Estimate]\label{prop:single_link_prob}
Given a fixed simplex $[x_1,\ldots,x_{d+1}]=\Delta \in S^\text{str}$, the volume of the region where a new point $y$ can be placed so that $\Gamma = [x_2,\ldots,x_{d+1},y] \rightarrow_{\epsilon} \Delta$ is bounded by $C_{\Theta,q,d,r_0} e^d \epsilon \tau^{D-d}$ for $\epsilon$ small enough.   
\end{proposition}

\begin{proof}
  See Appendix \ref{app:single-link}.
\end{proof}

Utilizing the above, we are ready to prove Theorem \ref{thm:between LAPD lower bound}, which we restate here. 

\second*

\begin{proof}
We first control the number of points landing near the intersection. Defining the event $\Omega := \{ n_T \leq 2C_2 nR\}$, we have $P(\Omega^C)\leq \exp\left(-\frac{C_1 nR}{3}\right)$ by Lemma \ref{lem:bound_intersec_pts}. Thus:
\begin{align*}
    \mathbb{P}(\text{bLAPD} < \epsilon ) &= \mathbb{P}(\text{bLAPD} < \epsilon | \Omega )\mathbb{P}(\Omega)+\mathbb{P}(\text{bLAPD} < \epsilon | \Omega^C )\mathbb{P}(\Omega^C) \\
    &\leq \mathbb{P}( \text{bLAPD} < \epsilon | \Omega )+\mathbb{P}(\Omega^C) \\
    &\leq \mathbb{P}( \text{bLAPD} < \epsilon | n_T = w )+\mathbb{P}(\Omega^C)
\end{align*}
where $w:=\lfloor 2C_2 nR \rfloor$. Note that the conditional probability above is simply the probability that $\text{bLAPD} < \epsilon$ if we sample $w$ i.i.d. points from $\mathcal{T}_R$. If $\text{bLAPD} < \epsilon$, then by Lemma \ref{lem:number_links} there exists a path of $s=\frac{\Theta}{2\epsilon}$ consecutive straddle simplices which link with weights bounded by $\epsilon$, i.e. there exists a path of straddle simplices such that $\Delta_1 \rightarrow_\epsilon \ldots \rightarrow_\epsilon \Delta_s$. We will denote $\Delta_i \rightarrow_{\epsilon, \text{str}} \Delta_j$ if $\Delta_i \rightarrow_{\epsilon} \Delta_j$ and $\Delta_i$ is straddle. By Lemma \ref{lem:tube_size}, if we choose $R=\frac{5e}{q\sin\Theta}$, 
then all of these simplices are in the tube $\mathcal{T}_R$.  
We note there is a one-to-one correspondence between a simplex path $\Delta_1 \rightarrow \ldots \rightarrow \Delta_s$ of $s$ simplices and a point path $x_1 \rightarrow \ldots \rightarrow x_{s+d}$ of $d+s$ points.
Consider now a fixed simplex path $\Delta_1 \rightarrow \ldots \rightarrow \Delta_s$, i.e. the simplex path defined by a fixed $x_1 \rightarrow \ldots \rightarrow x_{s+d}$. Thus
\begin{align*}
    \mathbb{P}(\Delta_1 \rightarrow_{\epsilon,\text{str}} \ldots \rightarrow_{\epsilon,\text{str}} \Delta_s)
    =&\mathbb{P}( \Delta_1 \rightarrow_{\epsilon,\text{str}} \Delta_2)\mathbb{P}(\Delta_2 \rightarrow_{\epsilon,\text{str}} \Delta_3 | \Delta_1 \rightarrow_{\epsilon,\text{str}} \Delta_2) \\ & \ldots
    \mathbb{P}(\Delta_{s-1} \rightarrow_{\epsilon,\text{str}} \Delta_s | \Delta_1 \rightarrow_{\epsilon,\text{str}} \ldots \rightarrow_{\epsilon,\text{str}} \Delta_{s-1}) 
\end{align*}
Now by Proposition \ref{prop:single_link_prob}, each term in the above product is bounded by 
\begin{align*}
  \frac{C_{\Theta,q,d} e^d \epsilon \tau^{D-d}}{\text{vol}_D(\mathcal{T}_R)} &\leq \frac{C_{\Theta,q,d} e^d \tau^{D-d}\epsilon}{C_1 R \tau^{D-d}} \leq \frac{C_{\Theta,q,d} e^d \epsilon q\sin\Theta}{ 5eC_1} = C_3 e^{d-1} \epsilon, \quad C_3 =  \frac{C_{\Theta,q,d}q \sin\Theta}{5C_1} \, ,
\end{align*}
so that for a fixed path,
\begin{align*}
    \mathbb{P}(\Delta_1 \rightarrow_{\epsilon,\text{str}} \ldots \rightarrow_{\epsilon,\text{str}} \Delta_s) &\leq \left(C_3 e^{d-1}\epsilon \right)^{s-1} \, .
\end{align*}
Note here we have used the fact that
\begin{align*}
   \mathbb{P}(\Delta_i \rightarrow_{\epsilon,\text{str}} \Delta_{i+1}) &= \mathbb{P}(\Delta_i \rightarrow_{\epsilon} \Delta_{i+1} \cap \Delta_i \in S^{\text{str}}) \\
   &=\mathbb{P}(\Delta_i \rightarrow_{\epsilon} \Delta_{i+1} | \Delta_i \in  S^{\text{str}})\mathbb{P}(\Delta_i \in  S^{\text{str}})  \\
   &\leq \mathbb{P}(\Delta_i \rightarrow_{\epsilon} \Delta_{i+1} | \Delta_i \in  S^{\text{str}})
\end{align*}
(similarly under conditioning on the previous links). Since we have $w$ possible points from which to form paths, there are a total of ${\frac{w!}{(w-(s+d))!}} \leq w^{s+d}$ ordered sets of points of size $s+d$, so that
 \begin{align*}
    \label{BLAPD bound without noise}
        \mathbb{P}( \text{bLAPD} < \epsilon | n_T = w ) &\leq w^{s+d} (C_3 e^{d-1}\epsilon)^{s-1} \\
    &= w^{d+1} (C_3 w e^{d-1}\epsilon)^{s-1} \\
    &= C_4 (ne)^{d+1} (C_5 n e^{d}\epsilon)^{\frac{\Theta}{2\epsilon}-1},
    \end{align*}
where $C_4 = \left(\frac{10C_2}{q\sin\Theta}\right)^{d+1}$, $C_5 = \frac{10C_2 C_3}{q\sin\Theta}=\frac{2C_{\Theta,q,d}C_2}{C_1}$. We define $\epsilon:=(2C_5ne^d)^{-1}$ and note that
\begin{align*}
n e^d &\geq \frac{2(d+1)}{C_5\Theta\log 2} \log n \quad \implies \quad \frac{\Theta}{4\epsilon}  = \frac{\Theta C_5 ne^d}{2}  \geq \frac{(d+1)\log n}{\log 2}
\end{align*}
One can check that the above inequality implies $(ne)^{d+1}\leq (C_5 n e^{d}\epsilon)^{-\frac{\Theta}{4\epsilon}} = \left( \frac{1}{2} \right)^{-\frac{\Theta}{4\epsilon}}=2^{\frac{\Theta}{4\epsilon}}$. To see this, note that since $e \leq 1$ by assumption, we have $(ne)^{d+1} \leq n^{d+1} = 2^{(d+1)\frac{\log n}{\log 2}} \leq 2^{\frac{\Theta}{4\epsilon}}$ and so 
\begin{align*}
    \mathbb{P}(\text{bLAPD} < \epsilon | n_T = w ) \leq  C_42^{1-\frac{\Theta}{4\epsilon}} \leq 2C_4 2^{-(d+1)\frac{\log n}{\log 2}} = 2C_4 n^{-(d+1)}.
\end{align*}
Note this probability dominates $\mathbb{P}(\Omega^C)$.
In summary, we have shown that provided that
\[ n e^d \geq \frac{2(d+1)}{C_5\Theta\log 2} \log n \,,  \]
one has
\begin{align*}
    \text{bLAPD} &> (2C_5 n e^d)^{-1}
\end{align*}
with probability at least $1-2C_4 n^{-{(d+1)}}-\mathbb{P}(\Omega^C) \geq 1-3C_4 n^{-{(d+1)}}$ where $C_4 = \left(\frac{10C_2}{q\sin\Theta}\right)^{d+1}$ and $C_5 = \frac{2C_{\Theta,q,d}C_2}{C_1}$. 
\end{proof}

\section{Proof of Theorem \ref{thm:between LAPD after denoising}}
\label{pf:between LAPD after denoising}

We now show how a very similar argument can be used to guarantee that if we incorporate a denoising procedure, we can achieve a bLAPD which is $O(\Theta)$. To make the argument precise, we define \text{super straddle} simplices to be simplices which have an angle of at least $\frac{3\Theta}{8}$ with both manifolds $\M_1$ and $\M_2$.

\third*

\begin{proof}
We argue that with high probability all super straddle simplices are removed during denoising, so that the between distance after denoising must be at least $\frac{\Theta}{4}$. Note that if a path starts at a super straddle simplex and takes $\epsilon$-links, the path will contain only straddle simplices as long as the number of links does not exceed $s=\frac{\Theta}{8\epsilon}$. Now suppose there exists a super straddle simplex which survives denoising; then there exists a super straddle simplex that $\epsilon$ links to $s=\frac{\Theta}{8\epsilon}$ neighbors (all of which must be straddle). That is, there exists a path of simplices such that $\Delta_1 \rightarrow_{\epsilon,\text{str}} \ldots \rightarrow_{\epsilon,\text{str}} \Delta_s$. We simply repeat the proof of Theorem \ref{thm:between LAPD lower bound} to show this is unlikely, with $s=\frac{\Theta}{8\epsilon}$ replacing $s=\frac{\Theta}{2\epsilon}$. The only change is the constant. 
\end{proof}



\input{single-link-lemma}


\input{refs.bbl}
\end{document}

%% file: single-link-lemma.tex
\section{Single-link estimate: proof of Proposition \ref{prop:single_link_prob}}\label{app:single-link}
We provide an estimate for linking straddle simplices, summarized in Proposition \ref{prop:single_link_prob}. We first prove the result for the noiseless case, $\tau = 0$. In this case, the manifolds are strictly in a $d$-dimensional space, so that our estimates are with respect to $d$-dimensional volume.


\begin{lemma}[Noiseless single link estimate]\label{lemma:single-link}
  Let $\mathcal{M}_1$, $\mathcal{M}_2$, $S$, $d$, $q$, $e$, $r_0$ be fixed that satisfy \Cref{assump:intersection_dim,assump:size-quality,assump:linear}. Then there are positive constants $\epsilon_0$ and $C$ such that the following is true: Consider any $d$-simplex $\Delta = [\bs{x}_1, \ldots, \bs{x}_{d+1}] \in S^{\text{str}}$. With $\bs{v}$ any point in $\M_2$, let $\Delta_j(\bs{v})$ denote a simplex formed by the $(d+1)$ points $\bs{v}$ and $\{\bs{x}_\ell\}_{\ell\neq j}$. Consider the set of points in $\M_2$ such that a new \textit{valid} straddle simplex $\Delta_j(\bs{v})$ can be formed adjacent to $\Delta$ for any $j$:
  \begin{align}\label{eq:A-def}
    A \coloneqq \left\{ \bs{v} \in \M_2 \;\; \big|\;\; \Delta_j(\bs{v}) \in S \textrm{ and } W^2_S(\Delta, \Delta_j(\bs{v})) \leq \epsilon \textrm{ for any } j \in [d+1] \right\}.
  \end{align}
  Then whenever $\epsilon \leq \epsilon_0$,
  \begin{align*}
    \mathrm{vol}_{d}(A) \leq  C e^d \epsilon
  \end{align*}
\end{lemma}
\begin{proof}
  First note that by Assumption \ref{assump:intersection_dim}, $\M$ is contained in a $(d+1)$-dimensional Euclidean space, and so without loss of generality we assume $D=d+1$ throughout the proof. Within this proof, we use boldface letters, e.g. $\bs{x}$ to denote vectors or points in multi-dimensional space, and normally typeset letters, e.g., $x$ to denote scalars. Our overall strategy will be to use subadditivity of the volume of $A$. With
  \begin{align*}
    A_j = \left\{ \bs{v} \in \M_2 \;\; \big|\;\; \Delta_j(\bs{v}) \in S \textrm{ and } W^2_S(\Delta, \Delta_j(\bs{v})) \leq \epsilon\right\}\, ,
  \end{align*}
  then 
  \begin{align*}
    A = \bigcup_{j \in [d+1]} A_j \hskip 15pt \Longrightarrow \hskip 15pt \mathrm{vol}_d(A) \leq \sum_{j \in [d+1]} \mathrm{vol}_d(A_j),
  \end{align*}
  and we will proceed to compute $\mathrm{vol}_d(A_j)$. Fixing $j \in [d+1]$, and taking some $\bs{v} \in A_j$, we will study the cosine of the angle between $\Delta$ and $\Delta_j(\bs{v})$:
  \begin{align*}
    \alpha_{j}(\bs{v}) \coloneqq \bs{n}(\Delta) \cdot \bs{n}(\Delta_j(\bs{v})) = \cos \angle(\Delta, \Delta_j(\bs{v}))\, .
  \end{align*}
  where $\bs{n}(\cdot)$ denotes a(ny) unit normal vector to the input $d$-simplex.
  Then by definition we have,
  \begin{align*}
    \bs{v} \in A_j \hskip 5pt \Longrightarrow \hskip 5pt W_S^2(\Delta, \Delta_j(\bs{v})) \leq \epsilon \hskip 5pt \stackrel{\eqref{eq:two way weight}}{\Longrightarrow} \hskip 5pt \left|\alpha_j(\bs{v})\right| \geq 1 - \frac{\epsilon^2}{2},
  \end{align*}
  where we have used $\left|\cos x\right| \geq 1 - x^2/2$. Hence we have,
  \begin{align*}
    A_j &\subset B_j, & B_j &\coloneqq \left\{ \bs{v} \in \M_2 \;\; \big|\;\; \Delta_j(\bs{v}) \in S \textrm{ and } \left|\alpha_j(\bs{v})\right| \geq 1 - \epsilon^2/2 \right\}.
  \end{align*}
  We proceed to bound the volume of $B_j$, which we accomplish by deriving necessary conditions on the coordinates of $\bs{v}$ for $B_j$ membership. Without loss we assume $\bs{x}_{d+1} \in \M_2$. (If not, re-index the vertices of $\Delta$. Because $\Delta \in S^{\mathrm{str}}$, at least one of its vertices lies in $\M_2$.) We can assume a certain geometric configuration for the problem: Volumes, angles, simplex qualities, and edge lengths are all invariant under affine unitary transformations, i.e., rigid rotations, reflections, and translations. In addition, we subsequently perform an isotropic scaling to extract dependence on $e$. In $D$-dimensional space, we choose such a transformation corresponding to, in order, a rigid translation, followed by $d = D-1$ rotations, followed by isotropic scaling, so this mapping corresponds to the map $\bs{x} \mapsto \bs{z}_{\ast} \mapsto \bs{z}_{\ast\ast} \mapsto \bs{z}_{\ast\ast\ast} \mapsto \bs{z}$:
  \begin{subequations}\label{eq:geometry}
  \begin{align}\label{eq:geometry-rigid}
    \bs{z}_{\ast}(\bs{x}_{d+1}) &= \bs{0}, \\\label{eq:geometry-rotation}
    \bs{z}_{\ast\ast}(\M_2) &\subset \mathrm{span}\{\bs{e}_1, \ldots, \bs{e}_d\} \\\label{eq:geometry-rotations}    \bs{z}_{\ast\ast\ast}\left(\mathrm{span} \{\mathrm{span}\{\Delta\} \cap \M_2 \}\right) &= \mathrm{span}\{\bs{e}_2, \ldots, \bs{e}_{d}\}\\\label{eq:geometry-scaling}
    \bs{z} &= \frac{\bs{z}_{\ast\ast\ast}}{e}
  \end{align}
  \end{subequations}
  Condition \eqref{eq:geometry-rigid} is a rigid translation. Condition \eqref{eq:geometry-rotation} corresponds to rotation of a normal vector to $\M_2$ passing through $\bs{x}_{d+1} = \bs{0}$. Condition \eqref{eq:geometry-rotations} corresponds to a $(d-1)$-dimensional rotation around a normal vector to $\M_2$ passing through $\bs{x}_{d+1} = \bs{0}$. In summary, the map $\bs{x} \mapsto \bs{z}$ has the form,
  \begin{align*}
    \bs{z} &= \frac{1}{e} \bs{U} \left(\bs{x} - \bs{x}_{d+1}\right), & \bs{U} &\textrm{ is unitary}
  \end{align*}
  We also define $\bs{z}$-images of the given vertices of $\Delta = [\bs{x}_1, \ldots, \bs{x}_{d+1}]$:
  \begin{align*}
    \bs{z}_j &\coloneqq \bs{z}(\bs{x}_j), & j &\in [d+1], & \bs{z}_{d+1} &= \bs{0}
  \end{align*}
  While distances change by a factor of $e$ under the $\bs{x} \mapsto \bs{z}$ mapping, angles do not because the terminal scaling factor is isotropic. The $\bs{z}$-image of the set $B_j$ is a rigid translation, rotation, and isotropic scaling of $B_j$, and hence satisfies,
  \begin{align*}
    C_j &\coloneqq \bs{z}(B_j) , & \mathrm{vol}_d(B_j) = e^d \mathrm{vol}_d(C_j),
  \end{align*}
  We define corresponding quantities and notation in $\bs{z}$-space: 
  \begin{itemize}
    \item With $\bs{y} = \bs{z}(\bs{v})$, then $\bs{v} \in B_j \Longleftrightarrow \bs{y} \in C_j$.
    \item $\widetilde{\Delta} = [\bs{z}_1, \ldots, \bs{z}_{d+1}]$ is the $\bs{z}$-image of $\Delta$
    \item $\widetilde{\Delta}_j(\bs{y}) = [\bs{y}, \{\bs{z}_\ell\}_{\ell \neq j}]$ is the simplex adjacent to $\widetilde{\Delta}$ sharing all vertices except $\bs{z}_j$.
  \end{itemize}
  Two particularly useful consequences of this  $\bs{z} \mapsto \bs{x}$ map are,
  \begin{align*}
    \alpha_j(\bs{v}) = \bs{n}(\Delta) \cdot \bs{n}(\Delta_j(\bs{v})) &= \bs{n}(\widetilde{\Delta}) \cdot \bs{n}(\widetilde{\Delta}_j(\bs{y})) \eqqcolon \widetilde{\alpha}_j(\bs{y}), \\
    \Delta_j(\bs{v}) \in S &\Longleftrightarrow \widetilde{\Delta}_j(\bs{y}) \in \widetilde{S}
  \end{align*}
  where $\widetilde{S}$ is the set of valid simplices (cf. Assumption \ref{assump:size-quality}) whose minimal edge length is unity. I.e., $\Gamma = [\bs{u}_1, \ldots, \bs{u}_{d+1}] \in \widetilde{S}$ means,
  \begin{align*}
    r_0 V_0 \leq \mathrm{vol}_d(\Gamma) &\leq \frac{V_0}{q^d}, &
    1 \leq \left| \bs{u}_i - \bs{u}_j \right| &\leq \frac{1}{q} \;\; \forall\; i, j \in [d+1], \; i < j.
  \end{align*}
  In summary, we have the following equivalent definition for $C_j$:
  \begin{align*}
    C_j = \left\{ \bs{y} \in \mathrm{span}\{\bs{e}_1, \ldots, \bs{e}_d\} \;\; \big|\;\; \widetilde{\Delta}_j(\bs{y}) \in \widetilde{S} \textrm{ and } \left|\widetilde{\alpha}_j(\bs{y})\right| \geq 1 - \epsilon^2/2 \right\}.
  \end{align*}
  The conditions defining $C_j$ allow us to derive constraints on the coordinates of a fixed, arbitrary element $\bs{y} \in C_j$. By the property that $\bs{z}(\M_2)$ lies in $\mathrm{span}\{\bs{e}_1, \ldots, \bs{e}_d\}$, then,
  \begin{align*}
    \bs{y} &= \left(y_1, \ldots, y_d, 0 \right)^T \in \R^{D}
  \end{align*}
  The fact that $\widetilde{\Delta}_j(\bs{y})$ is a valid simplex provides a bound for $|\bs{y}|$:
  \begin{align*}
    |\bs{y}| - |\bs{z}_j| \leq \left| \bs{y} - \bs{z}_j\right| \stackrel{\eqref{eq:simplex quality constraint}}{\leq} \frac{1}{q}
  \end{align*}
  This, along with the fact that $\bs{z}_j = \bs{z}_j - \bs{z}_{d+1}$ is an edge for the valid simplex $\widetilde{\Delta}$ results in the bound,
  \begin{align}\label{eq:r-bound}
    \bs{y} \in C_j \hskip 10pt \Longrightarrow \hskip 10pt |\bs{y}| \leq \frac{1}{q} + |\bs{z}_j| \stackrel{\eqref{eq:simplex quality constraint}}{\leq} \frac{2}{q}.
  \end{align}
  This provides a bound on each coordinate of $\bs{y}$:
  \begin{align}\label{eq:yell-bound}
    |y_\ell| \leq |\bs{y}| &\leq \frac{2}{q}, & \ell \in [d].
  \end{align}
  The remainder of our arguments will craft an additional, more stringent constraint for $y_1$ that is necessary to achieve $|\widetilde{\alpha}_j(\bs{y})| \geq 1 - \epsilon^2/2$. To this end, an important observation is that by definition $|\widetilde{\alpha}_j(\bs{y})| = 1$ iff $\widetilde{\Delta}$ and $\widetilde{\Delta}_j(\bs{y})$ are coplanar, which occurs iff $\bs{y} \in \mathrm{span}\{\widetilde{\Delta}\}$. I.e.,
  \begin{align}\label{eq:ay-is-1}
    \bs{y} \in C_j \textrm{ and } |\widetilde{\alpha}_j(\bs{y})| = 1 \hskip 5pt \Longleftrightarrow \hskip 5pt \bs{y} \perp \bs{e}_{d+1} \textrm{ and } \bs{y} \in \mathrm{span}\{\widetilde{\Delta}\} \hskip 5pt \stackrel{\eqref{eq:geometry-rotations}}{\Longleftrightarrow} \hskip 5pt \bs{y} = (0, y_2, \ldots, y_d, 0)^T
  \end{align}
  An informal continuity argument then suggests that membership in $C_j$, i.e., ``large'' values of $\widetilde{\alpha}_j(\bs{y})$, requires ``small'' values of $y_1$. We make this precise below. To derive a formula for $\widetilde{\alpha}_j(\bs{y})$, we compute expressions for the normal vectors $\bs{n}(\widetilde{\Delta})$ and $\bs{n}(\widetilde{\Delta}_j(\bs{y}))$. We define $\psi$ as the angle between $\widetilde{\Delta}$ and $\bs{z}(\M_2)$, i.e., 
  \begin{align*}
    \psi = W^2(\widetilde{\Delta}, \bs{z}(\M_2)) = \angle(\widetilde{\Delta}, \bs{z}(\M_2)) = \angle(\Delta, \M_2) \in \left[\frac{\Theta}{4}, \frac{\pi}{2}\right],
  \end{align*}
  where $\psi \geq \Theta/4$ because $\Delta \in S^{\mathrm{str}}$. Then $\bs{n}(\widetilde{\Delta})$ can be determined by noting that it's a unit vector, that $\bs{e}_2, \ldots \bs{e}_{d}$ are all vectors in the span of $\widetilde{\Delta}$, and also that $\left|\bs{n}(\widetilde{\Delta}) \cdot \bs{n}(\bs{z}(\M_2))\right| = \left|\bs{n}(\widetilde{\Delta}) \cdot \bs{e}_{d+1}\right| = \cos \psi$:
  \begin{align*}
    \bs{n}(\widetilde{\Delta}) \cdot \bs{e}_\ell = 0, \hskip 5pt \ell \in \{2, \ldots, d\} \hskip 10pt \Longrightarrow
    \hskip 10pt \bs{n}(\widetilde{\Delta}) = \sigma_1 \bs{e}_1 \sin \psi + \sigma_2 \bs{e}_{d+1} \cos \psi,
  \end{align*}
  where $\sigma_1, \sigma_2 \in \{-1,1\}$ are some fixed signs (determined uniquely by any particular choice of $\bs{n}(\widetilde{\Delta})$). Let $\bs{w}_1, \ldots, \bs{w}_d$ be the set of $d$ vertices of $\widetilde{\Delta}_j(\bs{y})$ that are shared by $\widetilde{\Delta}$. I.e., they are any identification of the set of $d$ vectors $\left\{\bs{z}_\ell \right\}_{\ell \in [d+1]\backslash \{j\}}$. The remaining vertex is $\bs{y}$. Then a normal vector to $\widetilde{\Delta}_j(\bs{y})$ is given by,
  \begin{align*}
    \bs{n}(\widetilde{\Delta}_j(\bs{y})) &= \frac{\bs{m}}{|\bs{m}|}, &
    \bs{m} &= \det\left(\begin{array}{ccc}
      \bs{e}_1 & \cdots & \bs{e}_{d+1} \\
      \horzbar & (\bs{y} - \bs{w}_1)^T   & \horzbar \\
      \horzbar & (\bs{w}_2 - \bs{w}_1)^T & \horzbar \\
               & \vdots                  &          \\
      \horzbar & (\bs{w}_d - \bs{w}_1)^T & \horzbar 
    \end{array}\right)
  \end{align*}
  We make the observations that $\bs{m}$ is a multilinear function of the entries in the second row, and hence the entries of $\bs{m}$ are affine functions of $y_1$. Then since $\bs{n}(\widetilde{\Delta})$ is independent of $\bs{y}$, we have,
  \begin{align*}
    \widetilde{\alpha}_j(\bs{y}) = \bs{n}(\widetilde{\Delta}_j(\bs{y})) \cdot \bs{n}(\widetilde{\Delta}) &= \frac{\bs{m}}{|\bs{m}|} \cdot \bs{n}(\widetilde{\Delta}) = \frac{A + B y_1}{\sqrt{A_0^2 + \sum_{\ell \in [d]} (B_\ell y_1 + C_\ell)^2}}, 
  \end{align*}
  where $A, B, A_0, B_\ell, C_\ell$ all depend \textit{smoothly} (are infinitely differentiable) with respect to $\psi, y_2, \ldots, y_d$ because the entries of $\bs{m}$ are smooth functions of these parameters. 
  Note that by the Cauchy-Schwarz inequality,
  \begin{align}\label{eq:cs-alpha}
    |\widetilde{\alpha}_j(\bs{y})| \leq 1,
  \end{align}
  and recall from \eqref{eq:ay-is-1} that $|\widetilde{\alpha}_j(\bs{y})| = 1$ iff $y_1 = 0$. When $y_1 = 0$, let $\widetilde{\alpha}_j(0, y_2, \ldots, y_d, 0) = \sigma \in \{+1,-1\}$. Note that it cannot switch signs as a function of $\psi, y_2, \ldots, y_d$ due to continuity in those variables, and hence $\sigma$ is a constant in those variables. 
  To explicitly separate dependence of $\widetilde{\alpha}_j$ on $y_1$ versus other quantities, we'll write,
  \begin{align*}
    \widetilde{\alpha}_j(\bs{y}) &= \alpha_j(y_1, \bs{h}), & \bs{h} &= (y_2, , \ldots, y_d, \bs{w}_1, \ldots, \bs{w}_{d}) \in H(y_1) \subset \R^{d-1 + d(d+1)},
  \end{align*}
  The set $H(y_1)$ in $\R^{d-1 + d(d+1)}$ corresponds to the set of values that conform to our assumptions; as suggested by the notation, this set depends on $y_1$. We are concerned only with $y_1 = 0$. We claim the set $H(0)$ is compact, for which it suffices to show it's closed and bounded. The bound \eqref{eq:yell-bound} on the entries $y_2, \ldots, y_d$ of $\bs{y}$, implies that the first $d-1$ components of $\bs{h}$ are bounded. The values of each element of $\bs{w}_j$, $j \in [d]$ is bounded because $\bs{w}_j$ is a bounded vector:
  \begin{align*}
    \widetilde{S} \ni \Delta = [\bs{0}, \bs{w}_1, \ldots, \bs{w}_d] \hskip 10pt \Longrightarrow \hskip 10pt \|\bs{w}_j\| \leq 1/q,
  \end{align*}
  Hence, $H(0)$ is a bounded set. An arbitrary $\bs{h}$ has membership in $H(0)$ iff with $y_1 = 0$,
  \begin{align*}
    f_1(\bs{h}) &= \mathrm{vol}_d(\widetilde{\Delta}) \in \left[r_0 V_0, \frac{V_0}{q^d} \right], && \textrm{(Volume distortion of $\widetilde{\Delta}$)} \\
    f_2(\bs{h}) &= \mathrm{vol}_d(\widetilde{\Delta}_j(\bs{y})) \in \left[r_0 V_0, \frac{V_0}{q^d} \right], && \textrm{(Volume distortion of $\widetilde{\Delta}_j(\bs{y})$)} \\
    f_3(\bs{h}) &= \min_{\substack{\bs{a}, \bs{b} \in \mathrm{ex}(\widetilde{\Delta})\\ \bs{a} \neq \bs{b}}} \|\bs{a} - \bs{b}\| \in \left[1, \frac{1}{q}\right], && \textrm{(Edgelength distortion of $\widetilde{\Delta}$)} \\
    f_4(\bs{h}) &= \min_{\substack{\bs{a}, \bs{b} \in \mathrm{ex}(\widetilde{\Delta}_j(\bs{y}))\\ \bs{a}\neq \bs{b}}} \|\bs{a} - \bs{b}\| \in \left[1, \frac{1}{q}\right], && \textrm{(Edgelength distortion of $\widetilde{\Delta}_j(\bs{y})$)} 
  \end{align*}
  Because (i) the volume of any simplex is a continuous function of its extreme points, and (ii) the norm $\|\cdot\|$ is a continuous function of its argument, and (iii) the function $\min(\cdot,\cdot,\ldots,\cdot): \R^{d(d+1)/2} \rightarrow \R$ is a continuous function of its inputs, then the map $\bs{h} \mapsto \bs{f}(\bs{h})$ is a continuous function, and,
  \begin{align*}
    H(0) &= \bs{f}^{-1}(C), & C &= \left[r_0 V_0, \frac{V_0}{q^d} \right]^2 \times \left[1, \frac{1}{q}\right]^2 \subset \R^4
  \end{align*}
  where again we take $y_1 = 0$. Since $C$ is a closed set and $\bs{f}$ is continuous, then $H$ is closed. Therefore, we've shown that $H(0)$ is bounded and continuous, hence compact.
  For a fixed but arbitrary $\bs{h} \in H(0)$, we use the shorthand notation $\widetilde{\alpha}'_j(y_1,\bs{h}) = \frac{\partial}{\partial y_1} \widetilde{\alpha}_j(y_1,\bs{h})$. Then:
  \begin{itemize}
    \item $\sigma \widetilde{\alpha}_j\left(0, \bs{h}\right) = \sigma^2 = 1$, due to \eqref{eq:ay-is-1}.
    \item $\widetilde{\alpha}'_j\left(0, \bs{h}\right) = 0$. This is true because \eqref{eq:cs-alpha} implies that $\widetilde{\alpha}_j(\cdot,\bs{h})$ reaches its global maximum at $y_1 = 0$, and since by \eqref{eq:ay-is-1} $|\widetilde{\alpha}_j(y_1,\bs{h})| < 1$ for $y_1 \neq 0$, then $y_1= 0$ is a strict global extremum of the smooth function $\widetilde{\alpha}_j(y_1,\bs{h})$, so necessarily $y_1 = 0$ is a stationary point.
    \item $\sigma \widetilde{\alpha}''_j\left(0, \bs{h}\right) < 0$. If this were false, then the Mean Value Theorem would imply that $\sigma \widetilde{\alpha}_j(\cdot,\bs{h})$ would be locally non-decreasing around $0$, which violates the global extremum property of $y_1 = 0$.
  \end{itemize}
  The observations above imply that $\widetilde{\alpha}_j$ is $y_1$-stationary and strictly concave at $y_1 = 0$ for fixed $\bs{h}$. Our ultimate goal is to show $\bs{h}$-uniform and $\Delta$-uniform \textit{strong} concavity. I.e., we seek to show that $\sigma \widetilde{\alpha}''_j \leq m < 0$, where $m$ is independent both of $\bs{h} \in H(0)$ and also of the particular choice of simplex $\Delta \in S$.

  We define, 
  \begin{align*}
    K = \min_{\bs{h} \in H(0)} - \sigma \widetilde{\alpha}''_j\left(0, \bs{h}\right) > 0,
  \end{align*}
  where the inequality holds because $\widetilde{\alpha}''(\cdot,\bs{h})$ is a smooth (in particular continuous) function that is strictly negative on the compact set $H(0)$. Note that because we minimize over all elements $\bs{h} \in H(0)$, i.e., over all valid configurations of points in $\M_1$ and $\M_2$, then $K$ is \textit{independent} of the particular simplex $\Delta$ we start with, and depends only on manifolds $\M_1$ and $\M_2$, their dimension and intersection angle $d, \Theta$, and the geometric parameters $e, q, r_0$ defining our simplices. Given this value of $K$, we now consider a fixed $\bs{h}$ and define a quantity $\eta(\bs{h})$ as,
  \begin{align*}
    \eta(\bs{h}) = \max \left\{ \eta > 0 \;\;\big|\;\; \sigma \widetilde{\alpha}''_j(y_1,\bs{h}) \leq -\frac{K}{2} \textrm{ for all } |y_1| \leq \eta \right\} > 0.   
  \end{align*}
  We note the continuity of $\widetilde{\alpha}''(y_1,\bs{h})$ in $y_1$ and $\bs{h}$ imply strict positivity $\eta(\bs{h}) > 0$ and continuity of $\eta(\bs{h})$ in $\bs{h}$, respectively. We will next need the smallest value of $\eta$:
  \begin{align*}
    \eta_0 \coloneqq \min_{\bs{h} \in H(0)} \eta(\bs{h}) > 0,
  \end{align*}
  where the strict inequality is again appeals to the fact that $\eta(\cdot)$ is strictly negative over the compact set $H(0)$. In addition, by minimizing over all $\bs{h} \in H(0)$, then $\eta_0$ is also \textit{independent} of the particular simplex $\Delta$, and depends only on $\M_1, \M_2, d, \Theta, e, q, r_0$. By this definition, $\left(- \eta_0, \eta_0\right)$ is an interval where $\sigma \widetilde{\alpha}''(\cdot,\bs{h})$ is no larger than $-K/2$ for all $\bs{h} \in H$. Hence, applying Taylor's Theorem (with the Mean Value remainder) to the function $y_1 \mapsto \widetilde{\alpha}_j(y_1,\bs{h})$ yields the $\bs{h}$-uniform bound,
  \begin{align*}
    \sigma \widetilde{\alpha}_j(y_1, \bs{h}) &\leq \sigma \widetilde{\alpha}_j\left(0, \bs{h}\right) + \sigma \widetilde{\alpha}_j'\left(0, \bs{h}\right) y_1 + \frac{-K/2}{2} y_1^2 \\
    &= 1 - \frac{K}{4} y_1^2,
  \end{align*}
  uniformly in $\bs{h}$ whenever $|y_1| \leq \eta_0$. Finally, note that if $\bs{y} \in C_j$, then 
  \begin{align*}
    \frac{\epsilon^2}{2} \geq 1 - |\widetilde{\alpha}_j(y_1;\bs{h})| = 1 - \sigma \widetilde{\alpha}_j(y_1;\bs{h}) \geq \frac{K}{4} y_1^2,
  \end{align*}
  whenever $y_1 \leq \eta_0$. Hence, 
  \begin{align}\label{eq:y1-bound}
    \bs{y} \in A_j \textrm{ and } \epsilon \leq \eta_0 \sqrt{\frac{K}{2}} \hskip 10pt \Longrightarrow \hskip 10pt |y_1| \leq \epsilon \sqrt{\frac{2}{K}}.
  \end{align}
  Then combining \eqref{eq:yell-bound} with \eqref{eq:y1-bound}, we conclude that if $\epsilon < \eta_0 \sqrt{\frac{K}{2}}$, then
  \begin{align*}
    \mathrm{vol}_d(B_j) = e^d \mathrm{vol}_d(C_j) \leq e^d 2^d \epsilon \sqrt{\frac{2}{K}} \left(\frac{2}{q} \right)^{d-1} \leq \epsilon e^d \frac{2^{2d-1}}{q^{d-1}\sqrt{K}},
  \end{align*}
  and finally exercising subadditivity of volume,
  \begin{align*}
    \mathrm{vol}_d(A) \leq \sum_{j=0}^d \mathrm{vol}_d(A_j) \leq \sum_{j=0}^d \mathrm{vol}_d(B_j) \leq \epsilon e^d \left[ \frac{(d+1) 2^{2d-1}}{q^{d-1} \sqrt{K}}\right] \eqqcolon C \epsilon e^d.
  \end{align*}
  This bound requires $\epsilon \leq \epsilon_0 \coloneqq \eta_0 \sqrt{\frac{K}{2}}$, but both $\eta_0$ and $K$ are independent of $\Delta$, and depend only on $\M_1, \M_2, d, \Theta, e, q, r_0$.
\end{proof}

The argument presented in the proof of Lemma \ref{lemma:single-link} contains the essential idea for the desired estimate in Proposition \ref{prop:single_link_prob}. We expect the $\tau > 0$ to have essentially the same character but to be considerably more technical because noise in the nuisance $D-d$ dimensions can lift points $\bs{y}$ off the manifold $\M_2$. Therefore, in lieu of formally generalizing this to the noisy $\tau > 0$ case, we provide an informal argument for the noisy setup: when $\tau > 0$, we would be interested in bounding the volume of a tubular version of the set $A$ in \eqref{eq:A-def}:
\begin{align}\label{eq:Atau-def}
  A_\tau \coloneqq \left\{ \bs{y} \in T_\tau(\M_2) \;\; \big|\;\; \Delta_j(\bs{y}) \in S \textrm{ and } W^2_S(\Delta, \Delta_j(\bs{y})) \leq \epsilon \textrm{ for any } j=0, \ldots, d\right\}.
\end{align}
Since this region is in $D$-dimensional space, we would bound $\mathrm{vol}_D(A_\tau)$. As in the proof of Lemma \ref{lemma:single-link}, we would use membership in the set $A_\tau$ to derive bounds on the coordinates of a candidate point $\bs{y} \in \R^D$. The treatement of the first $d$ coordinates would be similar, corresponding to points on $\M_2$. The treatment of the remaining $D-d$ coordinates would be subject to perturbations bounded by $\tau$. The bounds for these remaining coordinates would be dictated by ensuring $|\alpha_j(\bs{y},q)| \geq 1 - \epsilon^2/2$. Because $\alpha_j(\bs{y},q)$ is continuous in the coordinates of $\bs{y}$, a rough Taylor series argument suggests that,
\begin{align*}
  |y_j| \leq \tau \textrm{ for $j \in [d+1,D]$ } \hskip 5pt \sim \hskip 5pt |\alpha_j(\bs{y},q)| \gtrsim 1 - \epsilon^2,
\end{align*}
for $\tau \lesssim \epsilon$. Then the remaining $D-d$ dimensions correspond to the restriction $|y_j| \leq \tau$, and so there is a constant $C$ such that,
\begin{align*}
  \mathrm{vol}_D(A_\tau) \sim C \epsilon e^d \tau^{D-d},
\end{align*}
which is the bound stated in Proposition \ref{prop:single_link_prob}.